\documentclass[11pt]{article}
\usepackage[margin=1in]{geometry}
\usepackage{times}
\usepackage{tikz}
\usetikzlibrary{
    arrows.meta,
    positioning,
    calc
}
\usepackage{wrapfig}

\usepackage[utf8]{inputenc} 
\usepackage[T1]{fontenc}    
\usepackage{url}            
\usepackage{booktabs}       
\usepackage{nicefrac}       
\usepackage{microtype}      
\usepackage{xcolor}         
\usepackage{adjustbox}
\usepackage{amsmath, amssymb, amsthm, amsfonts}       

\usepackage{cite} 
\usepackage{bm}
\usepackage{etoolbox}
\usepackage{enumitem}
\setlist[itemize]{leftmargin=1em}
\usepackage{graphicx}
\usepackage{caption}
\usepackage{mathtools}
\usepackage{parskip}
\usepackage{subcaption}
\usepackage{algorithm}
\usepackage{algpseudocode}
\usepackage{xparse}
\usepackage{xfrac}
\usepackage[colorlinks=true,citecolor=blue,linkcolor=blue,urlcolor=blue]{hyperref}
\algnotext{EndFor}
\makeatletter
\renewcommand{\theALG@line}{\thealgorithm.\arabic{ALG@line}}
\providecommand{\theHALG@line}{}
\renewcommand{\theHALG@line}{\thealgorithm.\arabic{ALG@line}}
\makeatother

\title{Deep Learning as Neural Low-Degree Filtering:\\
A Spectral Theory of Hierarchical Feature Learning}

\author{%
Yatin Dandi$^{1,2}$,
Matteo Vilucchio$^1$,
Luca Arnaboldi$^1$,
Hugo Tabanelli$^1$,
Florent Krzakala$^1$\\[0.75em]
$^1$Information Learning and Physics Laboratory, \'{E}cole Polytechnique F\'{e}d\'{e}rale de Lausanne (EPFL)\\
$^2$Statistical Physics of Computation Laboratory, \'{E}cole Polytechnique F\'{e}d\'{e}rale de Lausanne (EPFL)
}
\date{September 4, 2026}

\newtheorem{theorem}{Theorem}
\newtheorem{lemma}{Lemma}
\newtheorem{proposition}{Proposition}
\newtheorem{definition}{Definition}

\newtheorem{assumption}{Assumption}
\newtheorem{remark}{Remark}

\usepackage[capitalize,noabbrev]{cleveref}
\usepackage{commath}

\newcommand{\R}{\mathbb{R}}

\newcommand{\cH}{\mathcal{H}}

\newcommand{\E}{\mathbb{E}}

\newcommand{\bilingualcommand}[3]{%
	\newcommand{#1}[1][\ ]{%
		##1%
		\iflanguage{english}{\text{#2}}{%
			\iflanguage{french}{\text{#3}}{}%
		}%
		##1%
	}%
}

\bilingualcommand{\where}{where}{où}
\bilingualcommand{\textif}{if}{si}
\bilingualcommand{\textand}{and}{et}
\bilingualcommand{\textiff}{if and only if}{si et seulement si}
\bilingualcommand{\otherwise}{otherwise}{sinon}

\newcommand{\bg}{{\boldsymbol{g}}}

\newcommand{\bw}{{\boldsymbol{w}}}

\newcommand{\bv}{{\boldsymbol{v}}}
\newcommand{\ba}{{\boldsymbol{a}}}

\newcommand{\balpha}{{\boldsymbol{\alpha}}}

\newcommand{\bz}{{\boldsymbol{z}}}
\newcommand{\bx}{{\boldsymbol{x}}}

\newcommand{\by}{{\boldsymbol{y}}}

\newcommand{\be}{{\boldsymbol{e}}}

\newcommand{\bu}{{\boldsymbol{u}}}

\NewDocumentCommand{\bb}{o}{\mathbf{b}\IfValueT{#1}{^{#1}}}
\NewDocumentCommand{\hh}{o}{\mathbf{h}\IfValueT{#1}{^{#1}}}
\NewDocumentCommand{\zz}{o}{\mathbf{z}\IfValueT{#1}{^{#1}}}
\NewDocumentCommand{\hhtilde}{o}{\tilde{\mathbf{h}}\IfValueT{#1}{^{#1}}}
\NewDocumentCommand{\W}{o}{\mathrm{W}\IfValueT{#1}{^{#1}}}
\NewDocumentCommand{\Wtilde}{o}{\tilde{\mathrm{W}}\IfValueT{#1}{^{#1}}}

\NewDocumentCommand{\pp}{o}{\mathbf{p}\IfValueT{#1}{^{#1}}}
\NewDocumentCommand{\pptilde}{o}{\tilde{\mathbf{p}}\IfValueT{#1}{^{#1}}}
\NewDocumentCommand{\Patch}{o}{\mathsf{P}\IfValueT{#1}{^{#1}}}
\NewDocumentCommand{\Kpatch}{o}{K_{\mathrm{patch}}\IfValueT{#1}{^{#1}}}
\NewDocumentCommand{\Chat}{o}{\widehat{C}\IfValueT{#1}{^{#1}}}
\NewDocumentCommand{\ww}{o}{\mathbf{w}\IfValueT{#1}{^{#1}}}
\NewDocumentCommand{\wtilde}{o}{\tilde{\mathbf{w}}\IfValueT{#1}{^{#1}}}
\NewDocumentCommand{\alphau}{o}{\ensuremath{\alpha\IfValueT{#1}{^{#1}}}}

\DeclareMathOperator*{\argmax}{arg\,max}
\DeclareMathOperator*{\argmin}{arg\,min}

\providecommand{\norm}[1]{\ensuremath{\left\lVert #1\right\rVert}}

\newcommand{\op}{\mathrm{op}}

\newcommand{\RF}{\mathrm{RF}}

\begin{document}

\maketitle

\begin{abstract}
Understanding how deep neural networks learn useful internal representations from data remains a central open problem in the theory of deep learning. We introduce \emph{Neural Low-Degree Filtering} (Neural LoFi), a stylized limit of gradient-based training in which hierarchical feature learning becomes an explicit iterative spectral procedure. In this limit, the dynamics at each layer decouple: given the current representation, the next layer selects directions with maximal accessible low-degree correlation to the label. This yields a tractable surrogate mechanism for deep learning, together with a natural kernel-space interpretation. Neural LoFi provides a mathematically explicit framework for studying multi-layer feature learning beyond the lazy regime. It predicts how representations are selected layer by layer, explains how emergence of concepts arises with given sample complexity,
and gives a concrete mechanism by which depth progressively constructs new features from old ones through low-degree compositionality. We complement the theory with mechanistic experiments on fully connected and convolutional architectures, showing that Neural LoFi improves over lazy random-feature baselines, recovers meaningful structured filters, and predicts representations aligned with early gradient-descent feature discovery with real datasets.
\end{abstract}

\section{Introduction}

One of the most striking features of modern deep learning \cite{lecun2015deep} is the ability of neural networks to learn complex high-dimensional functions from data with extraordinary effectiveness \cite{sejnowski2020unreasonable}. Yet this empirical success still outpaces our theoretical understanding. 
A large body of empirical work suggests that deep networks do not merely fit input-output relations, but progressively build structured representations. In convnets, feature-visualization studies have shown that successive layers extract patterns of increasing complexity, from local edge-like detectors to more semantic motifs \cite{zeiler2014visualizing}. Transfer-learning experiments indicate that earlier layers tend to contain more general features, whereas deeper layers become increasingly specialized to the task \cite{yosinseki2014transferable}. The {\it platonic representation} viewpoint has emphasized that models often learn surprisingly aligned representations across architectures and training procedures \cite{huh2024platonic}. These observations point to a central challenge: 

\emph{Can we understand which features are discovered during training, in which order, and at what sample complexity? And can we understand why some functions are learned more efficiently with multi-layer models?}

Several theoretical perspectives have clarified important parts of this picture. In the lazy regime \cite{chizat2019lazy}, or equivalently in the Neural Tangent Kernel (NTK) limit \cite{jacot2018neural,lee2019wide}, training is described as kernel regression in an essentially fixed feature space, giving a powerful theory of optimization and generalization but largely bypassing feature learning. Mean-field analyses of two-layer networks capture genuine parameter evolution and representation learning \cite{mei2018mean,chizat2018global,rotskoff2018neural,sirignano2020mean}, while tensor-program and dynamical mean-field theory approaches provide a general framework for infinite-width limits, including feature-learning regimes under suitable parametrizations \cite{yang2021tuning,hajjar2024training,bordelon2022self,pacelli2023statistical}. In parallel, stylized high-dimensional models, such as single-index and multi-index settings, have made it possible to analyze when neural networks recover latent structure from data \cite{ghorbani2020neural,bietti2022learning,damian2024computational,troianifundamental}, and recent work has begun to clarify the computational role of depth in synthetic toy models \cite{garnierbrun2025transformerslearnstructureddata,cagnetta2024deep,favero2021locality,favero2025compositional
,dandicomputational,ren2026provable}. Still, we lack a simple predictive mechanism for how deep networks select, organize, and refine features across layers.

In this work, we propose such a mechanism, which we call {\it Neural Low-Degree Filtering} (Neural LoFi). Neural LoFi is a stylized, mathematically tractable iterative spectral surrogate for feature learning: given a current representation, each layer forms a label-weighted moment operator on the current features, selects its leading eigendirections, and lifts the resulting projected features through a nonlinear random feature map. This procedure is motivated by a stylized small-initialization limit of gradient-based training, in which the feature-learning component of the layerwise dynamics is governed by a Hessian-like, label-weighted second-order operator.

The analysis of Neural LoFi leads to two main lessons. First, at a fixed representation, feature learning is governed by a \emph{relevance--complexity} trade-off: the next layer selects features that have large low-degree correlation with the label while remaining simple in the geometry induced by the current representation. Neural LoFi makes this trade-off explicit through a variational problem, where relevance is measured by supervised low-degree correlation and complexity by the RKHS norm. This variational view also yields an {\it explicit, data-driven criterion for feature emergence}: a new direction becomes learnable when its population correlation rises above the empirical noise floor, whose scale is controlled by the residual effective dimension of the current kernel.

Second, depth makes this process powerful because the selected features are lifted into a new representation, changing the geometry in which the next layer searches for signal. Thus a deep network can turn structure that is high-degree in the input into structure that is low-degree in an intermediate representation. This leads to a principle of {\it low-degree compositionality}: deep learning is efficient when each stage of the target is not only compositional, but visible through low-degree correlations in the current representation. The same viewpoint also gives a kernel interpretation of feature learning: Neural LoFi constructs a sequence of task-adaptive kernels, one per layer, rather than a single fixed kernel chosen before seeing the labels. In this sense, depth allows for an adaptive multilayer kernel construction driven by supervised low-degree feature selection. 

Although we use Neural LoFi primarily as a theoretical lens, the resulting procedure is also of independent interest: it is layerwise, backpropagation-free, and based only on label-weighted low-degree moments of the current representation. In this sense, it provides a simple spectral primitive for feature discovery, sitting between fixed random-feature models, which do not adapt their representation, and full end-to-end backpropagation, whose dynamics are much harder to analyze.

The rest of the paper develops this picture. We first introduce Neural LoFi and motivate its label-weighted moment operator from the early feature-learning dynamics of gradient descent. We then derive an RKHS variational characterization, which makes explicit the relevance--complexity trade-off solved at each layer. This leads to a criterion for feature emergence, expressed through the residual effective dimension of the current kernel, and to the principle of low-degree compositionality. We conclude with a study of a solvable mathematical model that illustrates our main point (emergence of features and low-degree compositionality), and with real-data experiments, including fully connected and convolutional architectures, illustrating the predicted mechanisms arising in practice.

A public implementation, together with code to reproduce the numerical illustrations, is available at
\url{https://github.com/IdePHICS/Neural-LoFi-Theory}.

\section{Neural Low-Degree Filtering}
\label{sec:neural-lofi}

\subsection{Setup and guiding approximation}
We now define Neural LoFi and present its feature-space and kernel interpretations. Its derivation as a surrogate for early gradient-based feature learning is given in App.~\ref{app:GD-neural}. Consider supervised data
\[
    \mathcal D_n=\{(\bm x_\mu,y_\mu)\}_{\mu=1}^n,
    \qquad
    \bm x_\mu\in\mathbb R^d,\quad y_\mu\in\mathbb R,
\]
with scalar labels for simplicity (The extension to vector-valued outputs is discussed in App.~\ref{app:vector-labels}).  We assume throughout this section that the labels are centered, and write
\(\widehat{\mathbb E}_n\) for the empirical average over the training set.  A depth-\(L\) neural network builds a sequence of representations
\begin{equation}
    \bm z_0(\bm x)=\bm x,
    \qquad
    \bm h_\ell(\bm x)
    =
    \bm W_\ell \bm z_{\ell-1}(\bm x),
    \qquad
    \bm z_\ell(\bm x) \in \mathbb{R}^{p_\ell}
    =
    \sigma(\bm h_\ell(\bm x)),
    \qquad
    \ell=1,\dots,L,
    \label{eq:standard-network}
\end{equation}
for sequence of widths $p_1,\cdots, p_L$. For structured architectures, such as convolutional networks, the same notation should be read locally: \({\bm z}_{\ell-1}(\bm x)\) is replaced by the local patch or feature vector seen by a filter, and the corresponding estimator is averaged over spatial locations.
The model's output is then defined by a readout applied to the final layer 
\begin{equation}
    \hat f(x)= \langle \ba_L,\bz_L(x)\rangle.
\end{equation}
Fix a layer \(\ell\), and consider the preactivation
$
    h_{\ell,i}(\bm x)
    =
    \langle \bm w_{\ell,i},\bz_{\ell-1}(\bm x)\rangle 
$. We show  in Appendix \ref{app:GD-neural}, that under a layer-wise vanishing initialization scheme $a_{L} \ll W_{L-1} \ll \cdots \ll W_1 \ll 1$, the training of different neurons in a layer, can be described up to leading approximation through a linear dynamics consisting of two terms: (i) a constant drift along a mean direction given by $\hat{\bm u}^\ell \in \mathbb{R}^{p_{\ell}}$
and (ii), a linear projection along the matrix $\widehat{\bm C}^{(\ell)}  \in \mathbb{R}^{p_{\ell} \times p_{\ell}}$, with $\hat{\bm u}^\ell,  \widehat{\bm C}^{(\ell)}$ defined by:
\begin{equation}\label{eq:lofi-main-operator}
    \hat{\bm u}^\ell \coloneqq \frac{1}{n}\sum_{\mu=1}^n y_{\mu} \bz_{\ell-1}(\bx_\mu) , \quad 
    \widehat{\bm C}^{(\ell)}
    :=
    \frac1n
    \sum_{\mu=1}^n
    y_\mu\,
    \bz_{\ell-1}(\bm x_\mu)
    {\bz}_{\ell-1}(\bm x_\mu)^\top .
\end{equation}

This result is formalized below, which is a consequence of Taylor's theorem and the fact that under a layer-wise vanishing scaling of initializations, interactions between neurons within the same layer and between the present layer and the subsequent, untrained layers, are suppressed: 
\begin{proposition}[Informal]\label{thm:layerwise-GD}
Suppose that  $\hat f(x)$ is trained through layer-wise gradient descent on squared loss $\mathcal{L}=\frac{1}{2n}\sum_{\mu=1}^n (y_\mu-\hat{f}(\bx_\mu))^2$ with initialization scale  $a_{L} \ll W_{L-1} \ll \cdots \ll W_1 \ll 1$. Then, for any $\eta>0$, and layer $\ell$, any fixed neuron $w_{\ell,i}$ for $i \in p_{\ell}$, satisfies for small enough time:

\begin{equation}\label{eq:thm1approxmain}
\bw_{\ell,i}(t+1) - \bw_{\ell,i}(t)
\approx \eta\,c_{0,i} \hat{\bu}_\ell
+ \eta\,c_{1,i}\widehat{C}^{(\ell)}\,\bw_{\ell,i}(t)+O(\norm{\bw_{\ell,i}(0)}^2),
\end{equation}
where \(c_{0,i},c_{1,i}\in\mathbb{R}\) are initialization-dependent random coefficients, independent across neurons \(i\).
\end{proposition}

Appendix~\ref{app:higher-order-averaging} gives a complementary explanation for this suppression.  In~\Cref{thm:layerwise-GD}, higher-order terms are ordered by the layer-wise small-initialization scales, while 
Appendix~\ref{app:higher-order-averaging} instead isolates the effect of averaging over data: after testing the averaged gradient against a direction, its \(k\)-th contribution is a correlation between two degree-\(k\) polynomial components, both orthogonal to lower-degree terms.  
If the current representation has \(m_\ell\) degrees of freedom, this degree-\(k\) space has dimension of order \(m_\ell^k\), so a random initialization direction has typical squared overlap \(m_\ell^{-k}\) with a fixed label-weighted component. 
Thus higher-order interactions may be present on individual samples but are suppressed in the averaged gradient.

The first term in Eq. \ref{eq:thm1approxmain} is the linear correlation between the current representation and the label. It can be removed by centering, by fitting the best linear predictor in the current representation, or simply interpreted as the degree-one component of the dynamics. The second term, however, is the first genuinely multiple feature-learning component. Let $\widehat{C}^{(\ell)}=\hat{\bm V}\Lambda\hat{\bm V}^\top$ denote the eigendecomposition. For small step-size, the linear term ${\color{blue}c_{1,i}}\widehat{C}^{(\ell)}\,\bw_{\ell,i}(t)$ yields the per-neuron approximation $\bw_{\ell,i}(t)\approx \exp({\color{blue}c_{1,i}} \widehat{\bm C}^{(\ell)} t)\,\bw_{\ell,i}(0)$. Stacking the $p_\ell$ neuron weights into the row matrix $W_\ell(t)\in\mathbb{R}^{p_\ell\times p_{\ell-1}}$ with $i$-th row $\bw_{\ell,i}(t)^\top$, the dynamics splits into two interpretable pieces:
\begin{equation}\label{eq:lofi-decomp}
    W_\ell(t)
    \;\approx\;
    W_\ell(0)\,\exp\bigl(c_1\,\widehat{\bm C}^{(\ell)}\,t\bigr)
    \;=\;
    \underbrace{W_\ell(0)\,\hat{\bm V}}_{\text{random transformation}} \times \quad \underbrace{\exp(c_1\,\Lambda\,t)\,\hat{\bm V}^{\!\top}}_{\text{spectral filter}\,\to\,\text{top-eigenvector projection}}\,.
\end{equation}
The second part depends only on $\widehat{\bm C}^{(\ell)}$: it projects onto the eigen-basis of $\widehat{\bm C}^{(\ell)}$ and reweighs each direction by $\exp(c_1\lambda_r^{(\ell)} t)$. 
Since the coefficients \(c_{1,i}\) are random and may have either sign, the amplified directions are selected by large \emph{absolute} eigenvalues; this is why Neural LoFi orders the spectrum by \(|\lambda_r^{(\ell)}|\).
The first part $W_\ell(0)\,\hat{\bm V}$ is $t$-independent: it is a fixed random linear transformation of those eigenvectors, inherited from the i.i.d.\ initialization of the neuron weights. 
The post-training feature map $\sigma\bigl(W_\ell(t)\,\bz_{\ell-1}(\bm x)\bigr)$ at each layer is therefore, to leading order, the composition of (i)~a spectral filter that selects the leading eigen-directions of $\widehat{\bm C}^{(\ell)}$ and (ii)~a random per-neuron mixing of those directions: small-initialization training acts, to leading order, as a {\it spectral filtering} of the  representation. \looseness=-1

While the exponential weighting in Equation \ref{eq:lofi-decomp} holds under vanishing initialization and small time-scales, we expect, in general, the network to maintain a preference towards features corresponding to larger eigenvalue magnitudes for other training regimes. We discuss how weighting regimes arise under different training settings for two-layer networks in App. \ref{app:two-layer}. The simplest choice of such weights is top-$k$ selection by \(|\lambda_r^{(\ell)}|\): $w(\lambda_r^{(\ell)})=1$ for \(r\leq k\) after magnitude ordering, and $0$ otherwise, which we use in~\Cref{alg:neural-lofi}.

Unlike the general gradient descent dynamics involving complex interactions across layers and neurons, the regime identified by Proposition \ref{thm:layerwise-GD} is {\it sequential} across layers and {\it decoupled} across neurons.  We call this the  {\it Neural LoFi regime} and put forward the following hypothesis:
\begin{center}{\it
    The \textbf{Neural LoFi regime} described by Proposition \ref{thm:layerwise-GD} provides a tractable surrogate towards understanding hierarchical learning in deep neural networks.}
\end{center}

\begin{algorithm}[t]
\caption{Neural Low-Degree Filtering}
\label{alg:neural-lofi}
\begin{algorithmic}[1]
\Require Dataset $\{(\bm x_\mu,y_\mu)\}_{\mu=1}^n$, depth $L$, ranks $\{k_\ell\}_{\ell=1}^L$, widths $\{p_\ell\}_{\ell=1}^L$
\State Initialize $\bm z_0(\bm x)\gets \bm x$ and $p_0 = d$.
\For{$\ell=1,\dots,L$}
    \State $\hat{\bm u}^\ell  \leftarrow \sfrac{1}{n}\sum_{\mu=1}^n y_{\mu} \bz_{\ell-1}(\bx_\mu)$, $\hat{\bm v}^\ell_0 \leftarrow \hat{\bm u}^\ell / \|\hat{\bm u}^\ell\|$
    \Comment{Estimate linear component, normalize}
    \State Form the label-weighted moment operator:
\vspace{-0.6em}
\[
        \widehat{\bm C}^{(\ell)}
        \gets
        \frac1n
        \sum_{\mu=1}^n
        y_\mu\,
      \bz_{\ell-1}(\bm x_\mu)
        \bz_{\ell-1}(\bm x_\mu)^\top \in \mathbb{R}^{p_{\ell-1} \times p_{\ell - 1}}.
\]
\vspace{-0.6em}
    \State{$
        \widehat{\bm V}_\ell
        =
        [\hat{\bm v}^\ell_0 , \hat{\bm v}^{(\ell)}_1,\dots,\hat{\bm v}^{(\ell)}_{k_\ell}] 
        \in \mathbb{R}^{p_{\ell-1} \times k_\ell}
    $, with ordered \(k_\ell\) eigenvectors by decreasing \(|\hat\lambda|\).}
    \State{$
        \bm g_\ell(\bm x)
        \gets
        \widehat{\bm V}_\ell^\top
        \bz_{\ell-1}(\bm x)
        \in\mathbb R^{k_\ell}.
    $} 
    \Comment{Project onto the learned feature subspace}
    \State{$
        \bm z_\ell(\bm x)
        \gets
        \sfrac{1}{\sqrt{p_\ell}} \sigma(\bm R_\ell \bm g_\ell(\bm x)),
        \,
        \bm R_\ell\in\mathbb R^{p_\ell\times k_\ell}.
    $}
    \Comment{Lift by nonlinear random feature map}
\EndFor
\State{Fit a final linear or logistic readout $\bm a$ on $\bm z_L(\bm x)$.}
\State \Return Final predictor $\hat f(\bm x)=\langle \bm a,\bm z_L(\bm x)\rangle$
\end{algorithmic}
\end{algorithm}

\subsection{The Neural Low-Degree Filtering algorithm}
The decomposition in Eq.~\eqref{eq:lofi-decomp} directly motivates the two operations of Neural LoFi (Algorithm~\ref{alg:neural-lofi}), which replace the implicit dynamics of the Neural LoFi regime (Proposition~\ref{thm:layerwise-GD}) with two concrete, decoupled steps per layer. First, a \emph{filter} step projects the current representation onto the leading eigen-directions of the label-weighted moment operator $\widehat{\bm C}^{(\ell)}$, making the spectral filter explicit as $\widehat{\bm V}_\ell$. Second, a \emph{lift} step applies a nonlinear random feature map $\bm R_\ell$ that emulates the per-neuron randomness of the left factor $W_\ell(0)\hat{\bm V}$ in Eq.~\eqref{eq:lofi-decomp}, producing the next representation.

The matrix \(\widehat{\bm V}_\ell\) defines the \emph{learned feature projection} at layer \(\ell\). The projections
\begin{equation}\label{eq:learned-feature-projection}
    \bm g_\ell(\bm x)
    =
    \widehat{\bm V}_\ell^\top
    \bz_{\ell-1}(\bm x) \,,
\end{equation}
are the features selected by low-degree filtering. The random nonlinear lift \(\sigma(\bm R_\ell \bm g_\ell)\) then re-expands these selected coordinates into a richer representation, so that the next layer can search for new low-degree correlations in the transformed feature space.
The projection step can be interpreted as a {\it weighted principal component analysis}: the linear component $\hat{\bm u}^\ell$ gives the direction in the feature space $\bz_{\ell-1}(\bm x)$ maximizing the linear correlation with $y$, and the top eigenvectors of $\widehat{\bm C}^{(\ell)}$ extract the most relevant directions\footnote{In some experiments, we did not even kept the linear component, as the second-order components are numerous and seem to suffice for extracting the most relevant features.} in feature space $\bz_{\ell-1}(\bm x)$ weighted by their 2\textsuperscript{nd}-order correlation with $y$. 

\begin{figure}[t]
    \centering
    \includegraphics[width=\linewidth]{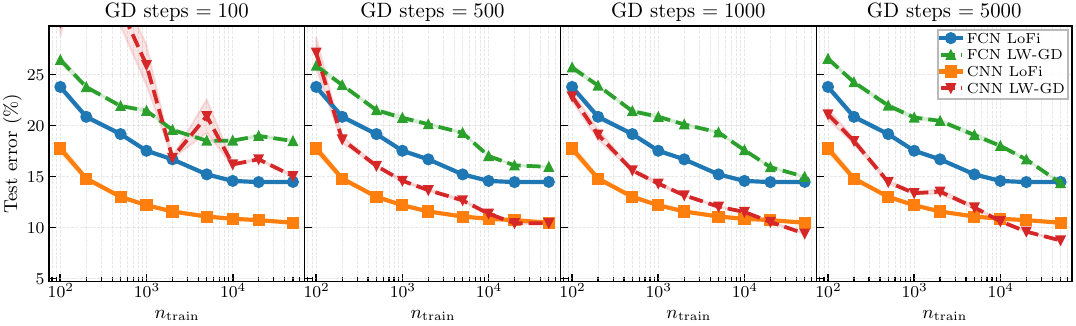}
    \caption{
        \textbf{Neural LoFi versus layerwise gradient descent (LW-GD).}
        Test error on binary CIFAR-10 \cite{krizhevsky2009learning} (animals vs. vehicles) for fully connected networks (FCN) and convolutional networks (CNN). We compare Neural LoFi with networks trained by layerwise gradient descent/backpropagation, shown for different numbers of training steps. In the low-data and low-time regime, Neural LoFi  exceeds GD, illustrating that the spectral surrogate captures an efficient supervised feature-extraction mechanism. See also Fig.\ref{fig:gradient-alignment} for the gradient alignment.}
\label{fig:gd-lofi-comparison-LW}
\end{figure}

\Cref{fig:gd-lofi-comparison-LW} gives a first sanity check for Neural LoFi as a surrogate for gradient-based feature learning. We compare it with layerwise GD (LW-GD) on the binary CIFAR-10 animal-vs.-vehicle task, for both fully connected and convolutional architectures. The results show that Neural LoFi can match early GD performance, especially in the low-data regime, while remaining a one-pass spectral procedure. We return to these experiments in detail in Section~\ref{sec:real-data}, where we also compare with the standard backpropagation gradient descent.

\subsection{Function-space and kernel interpretation}
The projection components $[\hat{\bm v}^\ell_0 , \hat{\bm v}^{(\ell)}_1,\dots,\hat{\bm v}^{(\ell)}_{k_\ell}]$ reside in the dual space w.r.t. the features $\bz_{\ell-1}$ and are themselves not directly interpretable and depend on the choice of random projection weights $R_{\ell-1},\cdots, R_1$. However, we show below that the resulting features $\langle\hat{\bm u}^{(\ell)},\bz_{\ell-1}(\bm x)\rangle$ and $\langle \hat{\bm v}^{(\ell)}_j,\bz_{\ell-1}(\bm x)\rangle$, for $j=1,\dots,k_\ell$, admit a clear description in terms of low-degree projections of $y$ onto the RKHS defined by the previous layer. Furthermore, the features converge to deterministic infinite-width limits. Let
\begin{equation}
    K_{\ell-1}(\bm x,\bm x')
    =
    \langle
    \bz_{\ell-1}(\bm x),
    \bz_{\ell-1}(\bm x')
    \rangle
    \label{eq:lofi-current-kernel}
\end{equation}
be the kernel induced by the current representation, and let \(\mathcal H_{\ell-1}\) be its Reproducing Kernel Hilbert Space (RKHS \cite{steinwart2008support,scholkopf2002learning}). Through an analogue of the classical representer theorem (Lemma \ref{lem:representer-lofi}), we show that for any $j \in k_{\ell}$, the features
$
    \varphi_{j}(\bm x)
    \coloneqq
    \langle \hat{\bm v}^{(\ell)}_j, \bz_{\ell-1}(\bm x)\rangle$, lie in the subspace spanned by $\{K_{\ell-1}(\bm x_{\mu},\cdot)\}_{\mu=1}^n$, and satisfy the equivalence of norms given by \(\|\varphi_{j}(\bm x)\|_{\mathcal H_{\ell-1}}=\|\hat{\bm v}^{(\ell)}_j\|_2\). 
Since $\hat{\bm v}^{(\ell)}_1$ maximizes $\bv^\top\widehat{\bm C}^{(\ell)}\bv$ subject to the $\|\bm v\|_2=1$, we obtain the equivalent criterion:
\begin{equation}
\hat{\bm v}^{(\ell)}_1
    \in
    \argmax_{\|\bm v\|_2=1}
    \left|
    \frac1n
    \sum_{\mu=1}^n
    y_\mu
    \langle \bm v,\bz_{\ell-1}(\bm x_\mu)\rangle^2
    \right|
\rightarrow
    \varphi_1^{(\ell)}
    \in
    \argmax_{\|\varphi\|_{\mathcal H_{\ell-1}}=1}
    \left|
    \frac1n
    \sum_{\mu=1}^n
    y_\mu\,\varphi(\bm x_\mu)^2
    \right|.
    \label{eq:lofi-rkhs-variational}
\end{equation}
Maximizing this correlation with the label, together with the norm constraint in feature space, leads to the fundamental conclusion:
\begin{center}
\emph{
    Neural LoFi searches for features that are {\bf simple} in the geometry induced by the previous layer, but {\bf predictive} through their low-degree correlation with the task.}
\end{center}

This is the sense in which the method is a low-degree filter.  We show in App.\ref{app:kernel-nlofi}, how this can  be turned into a practical kernel version of our approach. This is formalized through the following result, which also shows that the selected features converge to the corresponding variational maximizers for the limiting infinite-width kernel as the layer width grows:
\begin{theorem}[Variational characterization and infinite-width convergence]
\label{thm:variational-rkhs}
Let $\hat{\bm v}^\ell_0$ denote the linear component and $[\hat{\bm v}^{(\ell)}_1,\dots,\hat{\bm v}^{(\ell)}_{k_\ell}]$ the ordered eigenvectors from Algorithm~\ref{alg:neural-lofi}. 
Define the linear feature
\[
    \psi^\ell(\bm x) \coloneqq \langle \hat{\bm v}^\ell_0,\, \bz_{\ell-1}(\bm x)\rangle,
\]
and the second-order features $\varphi_j(\bm x) \coloneqq \langle \hat{\bm v}^{(\ell)}_j,\bz_{\ell-1}(\bm x)\rangle$ for $j=1,\dots,k_\ell$. Then:
\begin{enumerate}[label=(\roman*),noitemsep,leftmargin=1em,wide=0pt]
    \item \textbf{Linear feature.} $\psi^\ell$ is the unit-norm maximizer of the empirical first-order correlation with $y$:
    \begin{equation}
        \psi^\ell \;=\; \argmax_{\psi:\,\|\psi\|_{\mathcal H_{\ell-1}}=1}
        \left|\,\widehat{\mathbb E}_n\bigl[y\,\psi(\bm x)\bigr]\,\right|.
        \label{eq:prop-linear-var}
    \end{equation}
    \item \textbf{Second-order features.} For each $k=1,\dots,k_\ell$, $\varphi_k$ is the unit-norm maximizer of the empirical second-order correlation, successively orthogonalized to the previously selected features:
    \begin{equation}
        \varphi_k \;=\; \argmax_{\substack{\varphi:\,\|\varphi\|_{\mathcal H_{\ell-1}}=1\\ \varphi\perp\varphi_1,\dots,\varphi_{k-1}}}
        \left|\,\widehat{\mathbb E}_n\bigl[y\,\varphi(\bm x)^2\bigr]\,\right|.
        \label{eq:prop-quadratic-var}
    \end{equation}
    \item \textbf{Infinite-width convergence.} Let $K^{\infty}_{\ell-1}$ denote the infinite-width limiting kernel obtained as the layer width $p_{\ell-1}\to\infty$, and let $\{\phi^{\infty}_k\}_{k\geq 1}$ be the successive maximizers of the objective in~\eqref{eq:prop-quadratic-var}, with the RKHS norm and orthogonality induced by $K^{\infty}_{\ell-1}$. Assume that the eigenvalues associated with the selected maximizers in this limiting label-weighted variational problem are separated from the rest of its spectrum. Then, for any pseudo-lipschitz (of finite order) $\sigma(\cdot)$, and any fixed $k$, as $p_{\ell-1}\to\infty$,
    \[
        \varphi_k \;\xrightarrow{\;\;\;}\; \phi^{\infty}_k,
    \]
    where convergence is in $L^2$ under the data distribution.
\end{enumerate}
\end{theorem}

The above characterization in feature space translates to an explicit recursion over the kernels defined by each layer, once the selected features
$
    \varphi_1^{(\ell)},\dots,\varphi_{k_\ell}^{(\ell)}
$ are obtained, define
$
    \bm g_\ell(\bm x)
    =
    \big(
    \varphi_1^{(\ell)}(\bm x),\dots,\varphi_{k_\ell}^{(\ell)}(\bm x)
    \big).
$
If the next layer is a random nonlinear lift, then the induced kernel is
\begin{equation}
    K_\ell(\bm x,\bm x')
    =
    \mathbb E_{\bm r}
    \left[
    \sigma(\bm r^\top \bm g_\ell(\bm x))
    \sigma(\bm r^\top \bm g_\ell(\bm x'))
    \right].
    \label{eq:lofi-kernel-recursion}
\end{equation}

This recursion is the kernel-level form of feature learning. Unlike a {\it fixed} kernel method, with a single geometry before seeing the labels, Neural LoFi constructs a sequence of task-adaptive kernels
$K_0 \longrightarrow K_1 \longrightarrow \cdots \longrightarrow K_L$,
where each transition is supervised: the next kernel is built from the low-complexity features in \(\mathcal H_{\ell-1}\) whose squared activations are most correlated with the label.  Neural LoFi thus turns deep learning into an explicit procedure for adaptive kernel construction. For vector-valued labels, Appendix~\ref{app:vector-labels} describes the analogous construction: the scalar first-order criterion in \eqref{eq:prop-linear-var} is replaced by the vector objective \(\|\E[\by\,\phi(\bx)]\|_2^2\).  After orthogonalizing the selected scalar features, this already gives up to \(c\) feature directions when \(\by\in\mathbb R^c\); for linear candidates \(\phi(\bx)=\langle \bv,\bz_{\ell-1}(\bx)\rangle\), the criterion is equivalent to SVD on \(\E[\by\,\bz_{\ell-1}^\top]\). Appendix~\ref{app:covariance-correction} describes additional corrections that account for the covariance of the current feature map, through partial whitening of both the first- and second-order LoFi estimators and a covariance drift in the second-order power iteration.

\section{Neural LoFi lessons: emergence and low-degree compositionality}
\label{sec:lofi-lessons}
Neural LoFi is not proposed here as a replacement for backpropagation, but as a tractable, simpler, surrogate for understanding learning. This section now attempts to draw lessons from the surrogate. 

\subsection{Emergence of concepts, effective dimension, and a criterion}
\label{sec:emergence-snr}

A growing body of empirical evidence suggests that learning is often not a smooth process: training dynamics can display long plateaus followed by abrupt transitions, with new directions in representation space \emph{emerging} sequentially rather than all at once
\cite{wei2022emergent,raventos2023pretraining,arora2023theory,schaeffer2023emergent}.
Neural LoFi provides a simple mechanism for this phenomenon. By Theorem~\ref{thm:variational-rkhs}, the $k$th feature extracted by layer $\ell$ maximizes the following second-order {\it empirical} correlation:

\begin{equation}
\widehat \rho_\ell^{(k)}
:=
\sup_{\|\varphi\|_{\mathcal H_{\ell-1}}=1, \varphi \perp \varphi_1,\cdots, \varphi_{k-1}}
\left|
\widehat{\mathbb E}_n
\left[
y\,\varphi(\bm x)^2
\right]
\right|,
\label{eq:lofi-accessibility}
\end{equation}
where $\mathcal H_{\ell-1}$ is the RKHS induced by the representation after $\ell-1$ layers, and $\widehat{\mathbb E}_n$ is the empirical average over data. This variational form gives a direct criterion for feature emergence. Suppose that the features $\varphi_1,\cdots, \varphi_{k-1}$ have been extracted and define the set of candidate features as:

\begin{equation}
  \mathcal{S}^\ell_{k}
  \coloneqq
  \left\{
  \varphi\in\mathcal H_{\ell-1}:
  \norm{\varphi}_{\mathcal H_{\ell-1}}=1,\,
  \varphi \perp \varphi_1,\cdots, \varphi_{k-1}
  \right\}.
  \label{eq:lofi-candidate-class}
\end{equation}
For a fixed unit-norm feature
$\varphi\in \mathcal{S}^\ell_{k}$, define
\begin{equation}
    c_\ell(\varphi)
    :=
    \mathbb E\left[y\,\varphi(\bm x)^2\right],
    \qquad
    \widehat c_{\ell,n}(\varphi)
    :=
    \widehat{\mathbb E}_n\left[y\,\varphi(\bm x)^2\right].
\end{equation}
The empirical correlation $\widehat c_{\ell,n}(\varphi)$ fluctuates around its population value
$c_\ell(\varphi)$. Since Neural LoFi optimizes over a class of candidate features, the relevant measure of variance is the uniform fluctuation over $\mathcal{S}^\ell_{k}$,
\begin{equation}
\tau^k_\ell(n)
\coloneqq
\sup_{\varphi \in  \mathcal{S}^\ell_{k}}
\left|
\widehat c_{\ell,n}(\varphi)-c_\ell(\varphi)
\right|.
\end{equation}

To recover the population minimizer feature, the variance must be small w.r.t. the population correlation
\begin{equation}
\rho_\ell^{(k)}
    :=
    \sup_{\varphi \in  \mathcal{S}^\ell_{k}}
    \left|
    \mathbb E\left[y\,\varphi(\bm x)^2\right]
    \right|.
\end{equation}

Hence, the criterion for the emergence of a feature at layer $\ell$ becomes:
\begin{equation}
    \rho_\ell^{(k)}
    \gg
    \tau^k_\ell(n) .
    \label{eq:lofi-emergence-threshold}
\end{equation}
Below this threshold, the leading empirical direction is dominated by noise. Above it, a task-dependent direction separates from the noise and becomes learnable: Feature emergence is thus akin to a phase transition {\it \`a la} Baik-Ben Arous-P{\'e}ch{\'e}/Edwards-Jones (BBP/EJ) \cite{baik2005phase,edwards1976eigenvalue}.  
We show in Appendix~\ref{app:effective-dimension}, using local Rademacher-complexity bounds over $\mathcal{S}^\ell_{k}$ \cite{mendelson2003performance,bartlett2005local}, that up to poly-logarithmic factors:
\begin{equation}\label{eq:emergence_feature}
     \tau_\ell^k(n)
    \lesssim
    r^k_{\ell}\sqrt{\frac{D_{\ell,k}^{\mathrm{eff}}(r^k_{\ell})}{n}},
\end{equation}
where \(D_{\ell,k}^{\mathrm{eff}}(r)\) is the {\it residual effective dimension} of features with scale $r$ (see Def.\ref{def:lofi-effective-dimension}) within the current RKHS after removing the previously extracted features \cite{caponnetto2007optimal,rudi2017generalization}, and $r^k_{\ell}\coloneqq \argmax_{r} (r \sqrt{D_{\ell,k}^{\mathrm{eff}}(r)})$ is the dominant scale of fluctuations among candidate features. This can be estimated from the empirical spectrum of the kernel, making Eq.\eqref{eq:lofi-emergence-threshold} a data-driven diagnostic for emergence of concepts. 

We develop this formalization of emergence further in Appendix~\ref{app:effective-dimension}. The criterion recovers known sample-complexity scales in solvable hierarchical models
\cite{defilippis2026optimal,tabanelli2026deep,nichani2023provable,wang2023learning,fu2025learning}. More importantly, on real data it predicts the sample scale at which learned concepts appear in different layers: in the CIFAR-10 experiment of Section~\ref{sec:real-data}, individual eigenvector overlaps rise near the predicted thresholds; see Fig.~\ref{fig:sample-complexity-pred}. This makes Neural LoFi a quantitative theory of layerwise concept emergence.

\subsection{Why depth helps: low-degree compositionality}
\label{sec:low-degree-compositionality}
While approximation-theoretic works show that deep networks can represent certain compositional functions much more efficiently than shallow ones, especially in tree-like or hierarchical settings
\cite{mhaskar2017deep,telgarsky2016benefits}, efficient representation does not by itself imply efficient learning. The previous discussion motivates the notion of \emph{low-degree compositionality}: depth is useful to learn functions when each layer exposes a new low-degree signal that is visible in the representation constructed by the previous layers. This means that the target admits a sequence of intermediate representations
\[
\bm x
\;\longrightarrow\;
\bm z_1(\bm x)
\;\longrightarrow\;
\bm z_2(\bm x)
\;\longrightarrow\;
\cdots
\;\longrightarrow\;
\bm z_L(\bm x),
\]
such that, at each stage, the next useful representation is visible through a low-degree statistic of the current one. In the second-order Neural LoFi mechanism, this is precisely the condition that the population counterpart of Equation~\eqref{eq:lofi-accessibility} is larger than the statistical noise floor. Thus, at layer $\ell$, Neural LoFi asks whether the current representation contains a simple feature whose second-order statistics are predictive of the task. This gives a learnability criterion for compositional structure: A deep model does not benefit from {\it arbitrary compositions}; it benefits from {\it hierarchies} in which each intermediate step exposes a {\it statistically detectable low-degree signal}. 
Equivalently, the target should become progressively simpler when expressed in the learned coordinates. 

\begin{center}
    \emph{A function may be high-degree, or otherwise complex, as a function of the original input, while still being learnable through a {\bf sequence of low-degree problems}: Depth then converts one hard estimation problem in the input space into several easier estimation problems in adapted representation spaces.} 
\end{center}

This perspective is consistent with a growing body of theoretical work on idealized models, where compositional structure leads to sequential feature recovery and sample-complexity gains from depth, see:
\cite{cagnetta2024deep,favero2021locality,favero2025compositional
,garnierbrun2025transformerslearnstructureddata,dandicomputational,tabanelli2026deep,ren2026provable,cagnetta2026deriving}. Neural LoFi abstracts a common principle from these settings: a useful hierarchy is one whose next layer is statistically visible through low-degree correlations in the representation already learned. This viewpoint is in particular related to the {\it compositional information exponent} introduced for hierarchical Gaussian targets
\cite{dandicomputational}. In that setting, one assumes access to a planted hierarchy of intermediate features $\bm h_\ell^\star(\bm x)$ and asks for the smallest degree $q$ such that
\begin{equation}
\left\|
\mathbb E
\left[
(\bm h_\ell^\star(\bm x))^{\otimes q}
f^\star(\bm x)
\right]
\right\|_F
=
\Theta(1).
\label{eq:cie-planted}
\end{equation}
A low compositional exponent means that the target has a strong low-degree dependence on the hidden intermediate representation. Instead of assuming access to $\bm h_\ell^\star$, Neural LoFi searches over the current learned feature space through Eq.~\eqref{eq:lofi-accessibility}, and acts as a data-adaptive compositional information test by asking whether the current representation contains simple, statistically visible features that are predictive of the task.
In this sense, Neural LoFi implements a supervised coarse-graining procedure: Each layer keeps a small number of task-relevant directions and discards much of the irrelevant high-dimensional variation, in a way reminiscent of  physics renormalization-inspired views of learning across scales
\cite{wilson1983renormalization}. 

\subsection{Further related works}

\paragraph{Hessian and spectral learning ----}
Hessian-based spectral information has long played a central role in statistical physics and high-dimensional inference, for instance in community detection \cite{saade2014spectral}, spiked matrix--tensor models \cite{ros2019complex,sarao2019afraid,sarao2020marvels}, and BBP/EJ-type  selection mechanisms \cite{baik2005phase,edwards1976eigenvalue}. Related label-weighted second-order operators also appear in single-index and multi-index estimation, where they yield spectral initializations and recovery procedures \cite{lu2020phase,mondelli2018fundamental,maillard2020phase,vilucchio2025asymptotics}. More recently, similar operators have been connected to the early dynamics of shallow neural networks through Hessian or Hessian-like interpretations \cite{bonnaire2025role,zhang2025neural,montanari2026phase,defilippis2026optimal}. Neural LoFi extends this viewpoint beyond the first initial layer and to iterative multi-layer methods.

\paragraph{Beyond second order ---} While the second-order operator \eqref{eq:lofi-main-operator} is  the first nontrivial term in the expansion, higher-order terms correspond to {\it tensor} correlations between the current representation and the label. In Gaussian single-index and multi-index models, such higher-degree components govern harder directions and later learning time scales, as captured by information, leap, and generative exponents \cite{arous2021online,abbe2023sgd,damian2024computational,troianifundamental,damiangenerative}. Higher-order LoFi variants could replace \(y\varphi^2\) by higher-degree statistics, but would lead to (difficult) tensor spectral problems, as tensor PCA can be notably harder than matrix PCA \cite{hopkins2015tensor,lesieur2017statistical,wein2019kikuchi,arous2020algorithmic}. This limitation also points to a natural future direction: Multi-pass gradient descent can break the constraints imposed by information and leap exponents by repeatedly transforming the effective label and representation \cite{dandi2024benefits} as well as through staircase mechanisms \cite{sarao2019afraid,abbe2021staircase,benarous2024stochastic,bardone2024sliding,tabanelli2025computational}.  An interesting direction for future work is a multi-pass Neural LoFi, alternating low-degree spectral filtering with target/residual transformations and feature correction, that would connect the present mechanism to the more powerful generative-exponent/SQ-type learnability picture. We further discuss  connections with this theoretical literature in App.\ref{app:two-layer}.\looseness=-1

Recent alternating-gradient-flow analyses give a complementary dynamical account of feature learning from small initialization, in which training proceeds through plateaus and sharp drops as features activate sequentially \cite{kunin2025alternating}. Related work on sequential group composition studies how networks learn structured compositional tasks one representation-theoretic component at a time, and how depth can exploit associativity to improve scaling \cite{marchetti2026sequential}. These works are closely aligned with the goal of understanding the order in which features emerge under gradient descent; Neural LoFi instead isolates a layerwise spectral primitive that predicts which low-degree directions become available at each representation.

\paragraph{Quadratic networks and spectral feature learning ---}
A closely related line of work studies quadratic and polynomial networks, where the feature-learning component of the dynamics is especially transparent. Early works analyzed the optimization landscape and recovery properties of one-hidden-layer or polynomial networks through tensor and low-rank structure
\cite{ge2016matrix,venturi2019spurious,sarao2020optimization,arjevani2026geometry}.
More recent high-dimensional analyses of overparameterized quadratic networks show that ERM can be mapped to matrix sensing with nuclear-norm regularization, so that capacity control emerges through low-rank learned feature maps
\cite{erbanuclear}. Even closer to our perspective, \cite{defilippis2025scaling} characterize the spectra of trained quadratic and diagonal networks in the feature-learning regime: informative directions appear as spectral outliers or spikes, while the bulk corresponds to learned noise or overfitting and should be pruned. A related bulk-plus-spikes picture was later observed in attention models
\cite{boncoraglio2025inductive}. This mirrors the Neural LoFi mechanism: the Hessian-like label-weighted operator extracts task-correlated spectral directions and discards the residual bulk.  This viewpoint is also consistent with empirical observations that trained deep networks often exhibit structured, non-random spectra in their learned weights with eigenvalues popping out of the bulk \cite{martin2021implicit}.

\paragraph{Scattering, coarse-graining, and multiscale representations ---}
Our construction is  related in spirit to scattering transforms \cite{andreux2020kymatio}. Both scattering and Neural LoFi build representations through a multilevel cascade of filtering and nonlinear lifting. The key difference is that scattering uses fixed analytic filters, typically wavelets, designed for invariance and stability, whereas Neural LoFi is supervised and task-adaptive: each layer selects directions through a label-weighted spectral operator on the current representation. This also resonates with coarse-graining and renormalization-inspired views of deep representations \cite{marchand2022wavelet}.

\paragraph{Rainbow networks ---}
The rainbow analysis  of \cite{guth2024rainbow} gives a post-training description of deep nets: trained layers behave like random feature maps with learned, often low-rank, weight covariances, yielding deterministic hierarchical kernels in the infinite-width limit. In this view, feature learning amounts to identifying low-dimensional covariance structure between random high-dimensional embeddings. Neural LoFi provides a possible mechanism for the emergence of this structure: its label-weighted spectral operator selects the task-correlated directions to keep, while the subsequent random lift re-expands them into a new feature space. This is precisly the picture advocated in \cite{guth2024rainbow}.

\paragraph{Mechanistic interpretability ---}
Neural LoFi is complementary to mechanistic interpretability, which aims to reverse-engineer trained networks into interpretable features and circuits \cite{olah2020zoom,elhage2021mathematical}. Rather than starting from a trained model and asking which circuits implement a behavior, Neural LoFi asks why certain features are selected during training. Its basic objects are spectral directions in representation space, not necessarily individual neurons, aligning with the modern ``features as directions'' viewpoint underlying superposition \cite{elhage2022toy}.

\paragraph{Recursive Feature Machines ---}
A recent line of work proposed the \emph{average gradient outer product} (AGOP) as a mechanism for feature learning and uses it to define Recursive Feature Machines \cite{radhakrishnan2024mechanism,radhakrishnan2025linear}. Closely related convolutional variants were shown to recover edge-detector-like filters and deep visual features through AGOP-based updates \cite{beaglehole2023mechanism}, and recent works further connect neural feature matrices, AGOP-type objects, and weight-Gram dynamics to gradient descent and convergence \cite{boixadsera2026fact,cha2026weightgram}. This line of work is close in philosophy to Neural LoFi: both seek a simple iterative surrogate for representation learning. The constructions differ, however, in where the supervised spectral selection occurs. AGOP/RFM methods first fit a predictor and then update the feature metric using the gradient covariance of that predictor, whereas Neural LoFi applies a label-weighted spectral filter directly to the current representation before the nonlinear lift. 

Appendix~\ref{app:agop} explains the corresponding limitation of AGOP/RFM. For the hierarchical target, fixed-kernel KRR requires the degree-\(2q\) scale \(\Theta(d^{2q})\) before its prediction error improves. For RFM, feature recovery is measured by the fraction of AGOP-transformed feature energy assigned to the planted subspace. In the quadratic multi-index model this fraction vanishes below the degree-two dimension \(D_2=\Theta(d^2)\). The same mechanism suggests a \(d^{2q}\) bottleneck for Deep RFM on the hierarchical model. Neural LoFi, meanwhile, provably recovers the hidden representation at the scale \(D_qd_1=\Theta(d^{q+\epsilon})\) established in Appendix~\ref{app:th_rf}.

\paragraph{Local and backpropagation-free learning ---} In spirit,
Neural LoFi is also related to local alternatives to backpropagation, including feedback alignment and direct feedback alignment
\cite{lillicrap2016random,nokland2016direct,launay2020direct}, local-loss and layerwise training methods
\cite{nokland2019training,belilovsky2020decoupled}, target propagation
\cite{lee2015difference}, and biologically inspired plasticity rules
\cite{illing2019biologically,illing2021local}. Although Neural LoFi is not proposed as a biologically faithful model, it is backpropagation-free and layerwise: its spectral operator can be approximated by label-modulated Hebbian/Oja-style updates\cite{oja1982simplified,sanger1989optimal,gerstner2018eligibility,illing2021local}.

\paragraph{Pruning ---}  Neural LoFi also gives a feature-space perspective on why overparameterization and pruning are not contradictory. Classical pruning methods show that many weights or connections can be removed after training with little loss in performance
\cite{lecun1989optimal,han2015learning}, while the lottery-ticket hypothesis suggests that large dense networks may contain smaller trainable subnetworks
\cite{frankle2019lottery}. In Neural LoFi, width and rank play complementary roles: large width creates a rich feature space in which task-correlated directions can be discovered, while  pruning retains only the directions that finite data can reliably support. Thus large networks are useful for discovery, but low-rank feature selection controls the effective dimension used for prediction.

\section{Mechanistic illustrations: from a solvable model to real data}
\label{sec:real-data}
For this section, codes are available at
\url{https://github.com/IdePHICS/Neural-LoFi-Theory}. 

\begin{figure}[t]
    \centering
    \includegraphics[width=\linewidth]{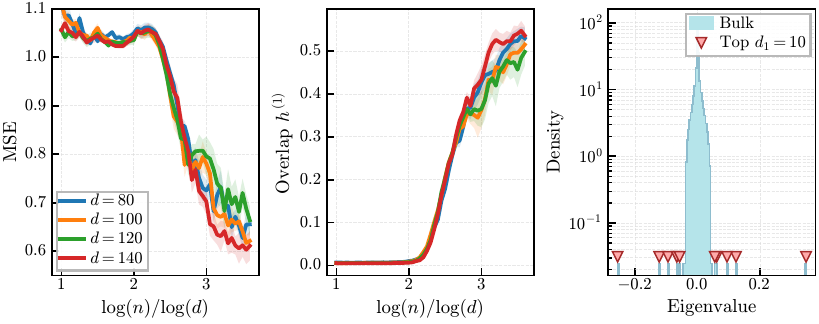}
    \caption{
    \textbf{Neural LoFi in a mathematically solvable model}: 
    We used data generated by the two-level target Eq.\eqref{model}, with ($k=2$), latent dimension $d_1=\lfloor d^\epsilon\rfloor$, $\epsilon=\tfrac12$, and final readout $g^\star(t)=\tanh(t)$, learned by a Neural LoFi approach. For $d\in\{80,100,120,140\}$, we use first-layer random-feature widths $p_1\in\{20000,30000,40000,50000\}$ and second-layer widths $p_2\in\{512,768,1024,1280\}$. The final one-dimensional non-linearity is fitted with a polynomial kernel of maximal degree $5$, using ridge regularization $10^{-6}$ and kernel regularization $10^{-4}$. \textbf{Left:} Mean Squared Error (MSE) of the final predictor as a function of $\alpha$, for input dimensions $d\in\{80,100,120,140\}$. Notice the drop at the predicted sample complexity for feature learning $n \propto d^{2.5}$.
    \textbf{Center:} Overlap between the recovered first-layer representation $\widehat h^{(1)}$ and the ground-truth latent variables $h^{(1)}$.   
    \textbf{Right:} Spectrum of the first random-feature spectral operator $\widehat C_1$ for $d=100$ at $\alpha=3$, where $d_1=10$, showing the emergent informative features.
    }
    \label{fig:rf_theoretical}
\end{figure}

\subsection{A solvable model}
\label{sec:the model}
We first consider a toy model in which the full Neural LoFi mechanism can be inspected mathematically, that illustrates the salient feature discussed in the former section: {\it emergence} and {\it low-degree compositionality}. Following the hierarchical polynomial constructions of
\cite{nichani2023provable,wang2023learning,fu2025learning,tabanelli2026deep}, let $\bm x\sim \mathcal N(0,I_d)$ and generate labels through a planted compositional hierarchy
$
\bm x
\;\longrightarrow\;
\bm h^{(1)}(\bm x)
\;\longrightarrow\;
h^{(2)}(\bm x)
\;\longrightarrow\;
y$. Let $H_k(\bm x)$ denote the degree-$k$ Hermite feature vector of the input. We plant $d_1=d^\varepsilon$ random directions
$\bm A_1^{(1)},\ldots,\bm A_{d_1}^{(1)}$ in this degree-$k$ feature space, together with a random symmetric matrix
$\bm A^{(2)}\in\mathbb R^{d_1\times d_1}$ acting on the first hidden layer. The hidden variables are
\begin{equation}
    h_i^{(1)}(\bm x)
    =
    \big\langle \bm A_i^{(1)}, H_k(\bm x)\big\rangle,
    \qquad i=1,\ldots,d_1,
\end{equation}
and the second representation is a quadratic function of these variables,
\begin{equation}
    h^{(2)}(\bm x)
    =
    \big\langle \bm A^{(2)}, H_2(\bm h^{(1)}(\bm x))\big\rangle,
    \qquad
    y = g^\star(h^{(2)}(\bm x)).\label{model}
\end{equation}
While the target is high-degree as a function of the input, it factors into low-degree transitions: $\bm h^{(1)}$ is degree-$k$ in $\bm x$, and $h^{(2)}$ is quadratic in $\bm h^{(1)}$. This is why the model is useful for illustrating both the advantage of feature learning and the advantage of depth: A fixed kernel or random-feature method \cite{mei2022generalization} sees only the composed map $\bm x\mapsto y$; if $g^\star$ contains a degree-$r$ component, then for $k=2$ this includes degree-$4r$ structure in the input. No learning whatsoever would arise before at least $n=O(d^4)$ data just to beat a random guess!

A two-layer feature learner can already improve on such a one-shot approach by recovering the first hidden representation $\bm h^{(1)}$, but the remaining target is still a nonlinear function of the $d_1=d^\varepsilon$ latent variables, and thus we still face the curse of dimension (over $d_1$ instead of $d^\varepsilon$).

This is where the multi-layer approach solve the problem: A second feature-learning step provides the depth advantage: once $\bm h^{(1)}$ is recovered, the next hidden variable $h^{(2)}$ is only a quadratic spectral problem in dimension $d_1$, followed by a one-dimensional readout. The model realizes the separation identified in \cite{tabanelli2026deep}, but now with Neural LoFi itself: Each layer lowers the effective degree of the remaining task by making the next hidden variable visible through a low-degree statistic. The final target is therefore learned by Neural LoFi progressively, and efficiently, building the correct intermediate representation, rather than by fitting the full high-degree map $\bm x\mapsto y$ in one shot. 

Appendix~\ref{app:th_rf} provides the corresponding mathematical theorems and synthetic experiments, demonstrating feature emergence at the predicted sample scales (Eq.\ref{eq:lofi-emergence-threshold}), spectral outlier formation, alignment with the planted features, and recovery of the final compositional target. 

We thus refer to Appendix~\ref{app:th_rf} for more details  and just briefly illustrate here the performance of Neural LoFi in Fig.~\ref{fig:rf_theoretical}: The function becomes learnable at \(n\gg D_k d_1=O(d^{k+\varepsilon})\), which, in the quadratic case used in the experiments, is \(n\gg d^{2+\varepsilon}\). Around this scale, the Neural LoFi estimator simultaneously shows a drop in prediction error, a sharp increase in overlap with the planted representation, and the separation of the leading eigenvalues from the spectral bulk (see right panel in Fig.~\ref{fig:rf_theoretical} as well as Fig.~\ref{fig:rf_spectrum} showing the emergence of learned features in the first layer in Appendix).

\begin{figure}[t]
    \centering
    \includegraphics[width=0.6\linewidth]{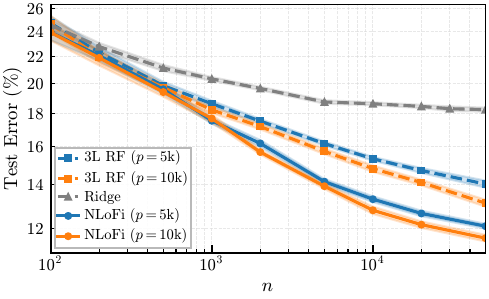}
    \hfill
    \includegraphics[width=0.37082\linewidth,height=0.37082\linewidth]{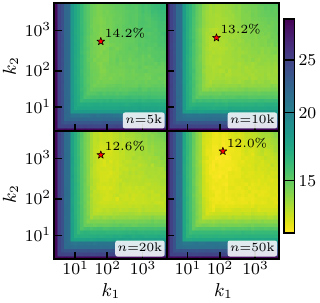}
    \caption{
    \emph{Fully connected Neural LoFi on the CIFAR-10 animal-vs.-vehicle task.}
    \textit{Left:} Test error vs number of training samples for ridge regression, three-layer random features, and Neural LoFi, with projection dimensions $p=5\mathrm{k}$ and $p=10\mathrm{k}$.
    \textit{Right:} Test error over the number of retained features $(k_1,k_2)$ in the first two LoFi layers, for different training-set sizes and fixed projection dimension $p=5\mathrm{k}$. Stars indicate the best point in each grid.
    }
    \label{fig:heatmaps}
\end{figure}

\begin{figure}[t]
    \centering
   \includegraphics[width=\linewidth]{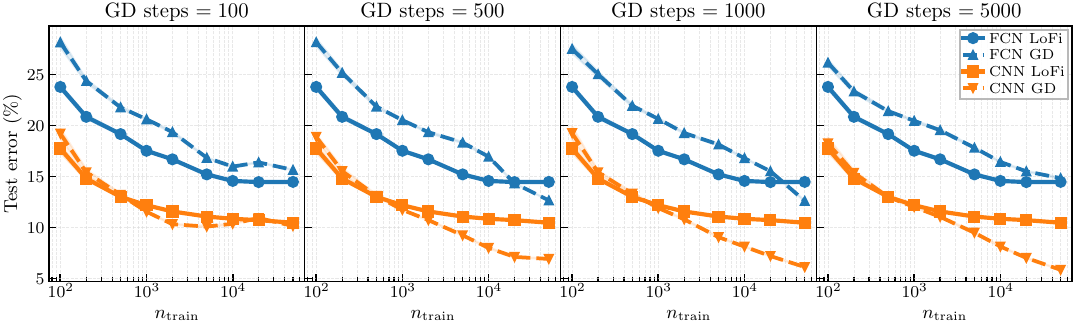}
    \caption{
        \textbf{Neural LoFi versus gradient descent/backpropagation (GD).
        }
        Same as Fig.\ref{fig:gd-lofi-comparison} but with full-backpropagation GD instead of Layer-wise GD: Test error on binary CIFAR-10 \cite{krizhevsky2009learning} (animals vs. vehicles) for fully connected networks (FCN) and convolutional networks (CNN). We compare Neural LoFi with networks trained by gradient descent/backpropagation, shown for different numbers of training steps. In the low-data regime, and at early training times even with more data, Neural LoFi matches or exceeds GD, illustrating that the spectral surrogate captures an efficient supervised feature-extraction mechanism.
    }
\label{fig:gd-lofi-comparison}
\end{figure}
\subsection{Fully connected networks (FCN)} 
We now focus on experiments on real data, intended as mechanistic illustrations of the theory, not as a claim that Neural LoFi is a competitive training algorithm. Fig.\ref{fig:heatmaps} (FCN on CIFAR-10 \cite{krizhevsky2009learning}) illustrates two qualitative predictions: i) the label-weighted LoFi operator extracts useful directions beyond those available to fixed random features, leading to a consistent improvement over ridge and random-feature baselines. ii) feature selection is non-trivial: the best test error is not always achieved by keeping the largest number of features. As the sample size grows, the optimal number of retained features increases or remains stable, matching the signal-to-noise picture in which weaker directions become accessible with more data.
\begin{figure*}[t]
    \centering
    \begin{subfigure}[t]{0.49\textwidth}
        \centering
        \includegraphics[width=\linewidth]{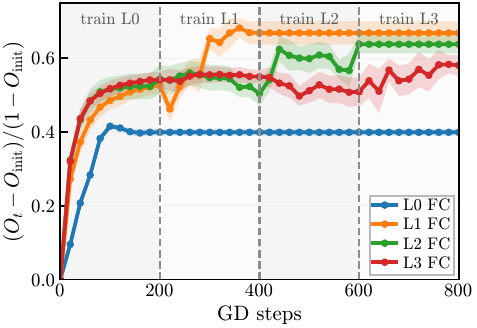}
        \caption{Layerwise GD.}
    \end{subfigure}\hfill
    \begin{subfigure}[t]{0.49\textwidth}
        \centering
        \includegraphics[width=\linewidth]{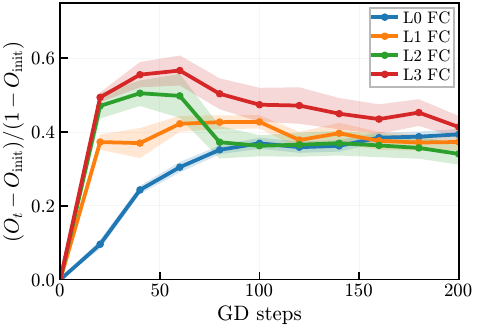}
        \caption{Joint GD.}
    \end{subfigure}
    \caption{
    \emph{CIFAR-10 animal-vs.-vehicle: four-layer GD subspace recovery.}
    GD-SVD subspace recovery by Neural LoFi (see App.~\ref{app:gd-feature-alignment}).
  Both panels plot the normalized subspace-overlap gain \((O_t-O_{\mathrm{init}})/(1-O_{\mathrm{init}})\) between the top \(r=8\) GD-SVD feature span and the fixed LoFi filtered-feature span before the random lift, using \(k=(20,40,60,80)\) retained LoFi directions and averaging over eight random initializations on \(N=10{,}000\) test samples. \textit{Left:} GD is trained layer by layer; during each shaded interval only the indicated hidden layer is updated. \textit{Right:} all hidden layers are trained jointly for the first 200 GD steps.
    }
    \label{fig:gradient-alignment}
\end{figure*}

We also compare Neural LoFi with GD training protocols:
 with layer-wise training as well as with standard backpropagation gradient descent. 
Fig. \ref{fig:gd-lofi-comparison}, for instance, shows that backpropagation gradient descent, while better than Layer-wise training (Fig.\ref{fig:gd-lofi-comparison-LW}) in that it requires less data and less time to beat Neural-Lofi, is still very similar. 

Perhaps more interesting are the plots in Fig.~\ref{fig:gradient-alignment} where, for each hidden layer, the normalized overlap gain \((O_t-O_{\mathrm{init}})/(1-O_{\mathrm{init}})\), where \(O_t\) is the empirical subspace overlap between the features spanned by the top-$8$ singular vectors of a network trained by GD and the features extracted by Neural-Lofi (more details in Appendix \ref{app:gd-feature-alignment}, evaluated on \(N=10{,}000\) held-out samples and averaged over eight random initializations. The left panel trains the hidden layers sequentially, with earlier layers fixed during later phases, while the right panel trains all layers jointly. The layerwise protocol exhibits staircase-like growth: the layer currently being optimized develops new features aligned with the LoFi features for that layer, while already-trained layers remain approximately stable. The joint-training panel lacks these clean layerwise transitions, supporting the interpretation of Neural LoFi as a layerwise spectral approximation to the feature-discovery phase of GD. This is further supported by the aforementioned Fig.~\ref{fig:gd-lofi-comparison}, that shows direct test-error comparisons between GD and Neural LoFi.

Most directly, Fig.~\ref{fig:sample-complexity-pred} tests the feature-emergence criterion of Eqs.~\eqref{eq:lofi-emergence-threshold} and~\eqref{eq:emergence_feature}. We compare each eigenvector learned from \(n\) samples with a large-sample reference eigenvector, and mark the predicted emergence thresholds. The overlap curves rise close to these thresholds, showing that the effective-dimension criterion predicts the sample scale at which individual task-relevant directions ("concepts") become learnable. Additional details on this experiment are given in App.~\ref{app:feature-recovery}; further numerical evidence, including spectra of the $\widehat C$ operator and experiments with the infinite-width Neural LoFi kernel, is reported in App.~\ref{app:numerical-extra}.

\begin{figure}[t]
    \centering
    \includegraphics[width=\linewidth]{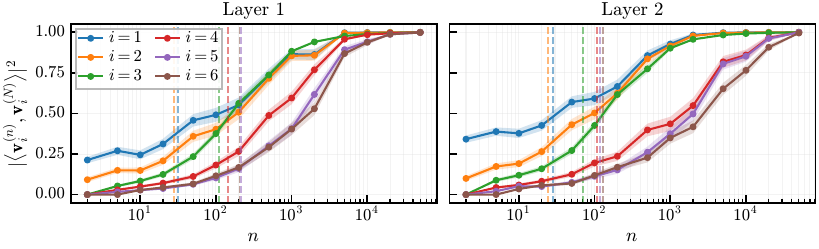}
    \caption{
        \textbf{Predicting when individual features emerge on CIFAR-10.}
        For a three-layer fully connected Neural LoFi model on the CIFAR-10 animal-vs.-vehicle task, we track the squared overlap
        $|\langle v_i^{(n)},v_i^{(N)}\rangle|^2$
        between eigenvectors estimated from \(n\) samples and large-sample reference eigenvectors computed with \(n=60\mathrm{k}\) samples. Curves show mean \(\pm\) SEM over 100 random subsamples at fixed random features. Dashed vertical lines indicate the predicted emergence thresholds \(n_\ell^i\) from \eqref{eq:lofi-emergence-threshold},\eqref{eq:emergence_feature}
        (for further details, see Eq.~\eqref{eq:real-data-emergence-estimate} in Appendix). The sharp rise of each overlap near its predicted threshold shows that the effective-dimension criterion predicts when individual task-relevant directions become learnable.
    }
    \label{fig:sample-complexity-pred}
\end{figure}

\subsection{Convolutional networks (CNN)}
Neural LoFi is not restricted to fully connected architectures: the moment operator can be built in the feature space exposed by the architecture. This makes it natural to incorporate structural inductive biases such as locality, weight sharing, equivariance, or graph structure, in the spirit of geometric deep learning~\cite{bronstein2021geometric}. Here, we consider convolutional networks.
Let $\bm z_{\ell-1}(\bm x_\mu)\in\mathbb R^{s_\ell\times p_{\ell-1}}$ denote the input features to layer $\ell$, where $s_\ell$ is the number of spatial locations and $p_{\ell-1}$ is the number of channels. We define the operator in channel space as:
\begin{equation}
    \widehat{\bm C}^{(\ell)}
    \coloneqq
    \frac{1}{n s_\ell}
    \sum_{\mu=1}^n
    \sum_{j=1}^{s_\ell}
    y_\mu\,
    \bm z_{\ell-1}(\bm x_\mu)_j
    \bm z_{\ell-1}(\bm x_\mu)_j^\top .
\end{equation}
The top eigenvectors of $\widehat{\bm C}^{(\ell)}$ define the task-relevant channel directions retained at layer $\ell$; the resulting features are then passed through the same fixed random lifting used by Neural LoFi.

We remind that Fig.~\ref{fig:gd-lofi-comparison} showed how similar direct GD optimization is to Neural LoFi at early training time. The convolutional experiment also gives the most direct visual evidence for low-degree compositionality. Applying the LoFi operator directly to pixels already recovers familiar low-level visual features: a leading color-sensitive direction, followed by localized contrast and edge-like filters. Several later filters resemble finite-difference or Laplacian-like stencils, echoing classical edge-filter emergence in natural image models~\cite{olshausen1996emergence} and the early filters learned by CNNs trained with gradient descent~\cite{hinton2012imagenet,zeiler2014visualizing,radhakrishnan2024mechanism}. 
In appendix, we further show in
Fig.~\ref{fig:sample-complexity-pred-conv}  the feature-emergence criterion of Eqs.~\eqref{eq:lofi-emergence-threshold} and~\eqref{eq:emergence_feature} for GNN, as we did for FCN in Fig.\ref{fig:sample-complexity-pred}.

\begin{figure}[t]
    \centering
\includegraphics[width=2cm,height=6.25cm]{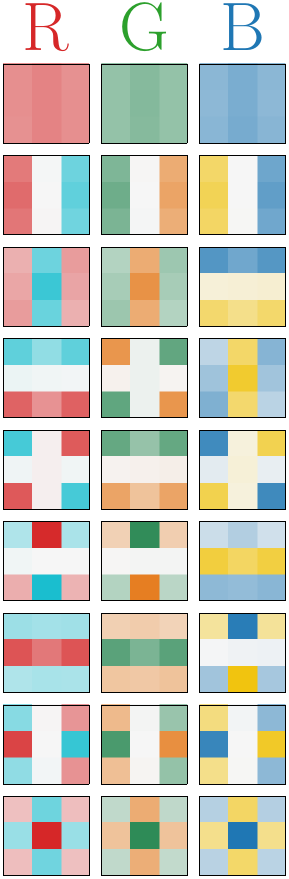}
    \hspace{.5cm}
    \includegraphics[width=6.9cm,height=6.25cm]{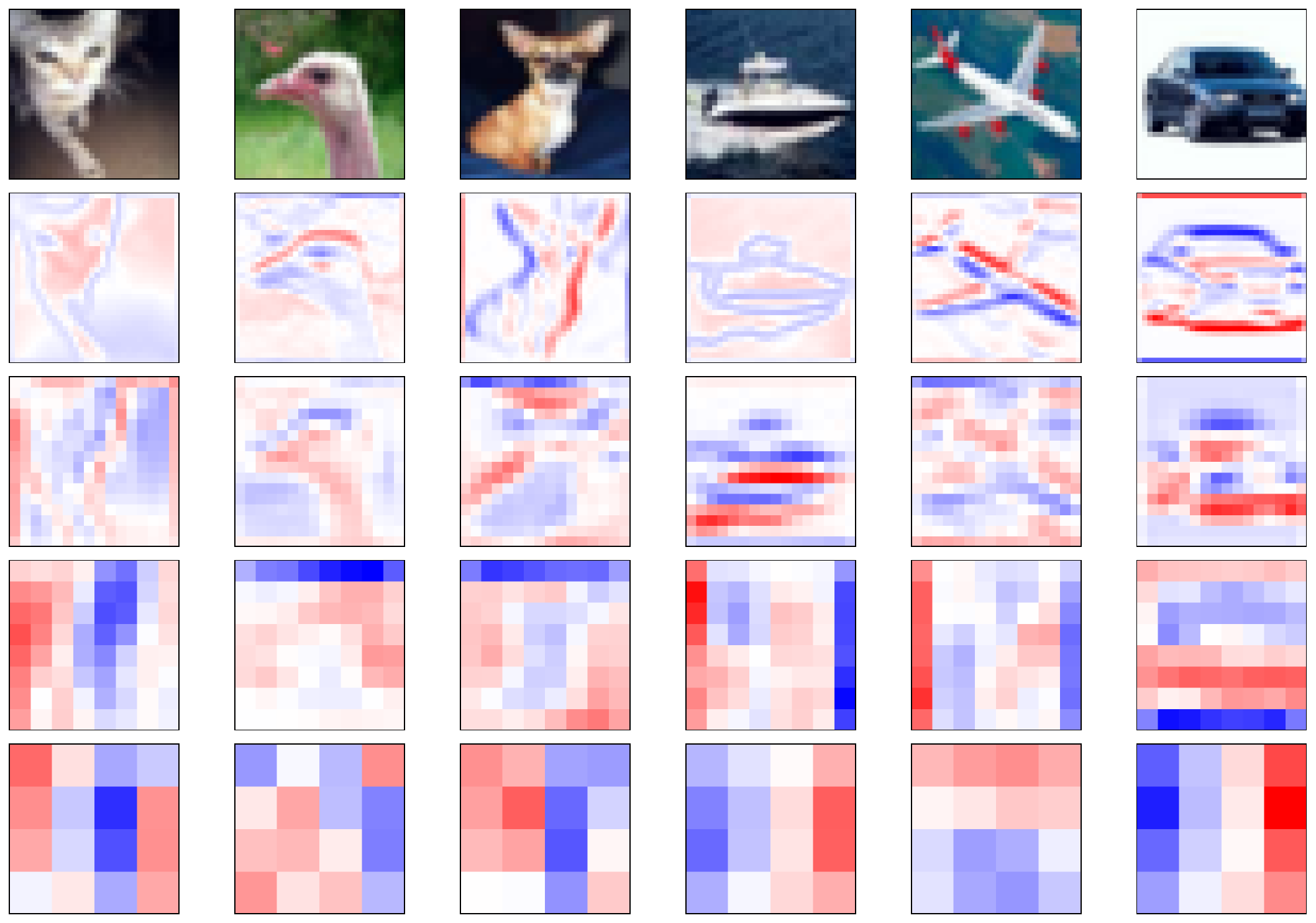}
  \caption{
  Neural LoFi with convolutional layers on the CIFAR-10 animal-vs.-vehicle task. Neural LoFi recovers the usual structures and edge detectors associated to GNN {\it without backpropagation or end-to-end optimization}
  \textit{Left:} Top first-layer filters obtained from the three RGB input channels.
  \textit{Right:} Activations of the sixth LoFi feature on test images at successive depths; top row shows the input images, lower rows  the corresponding activation maps.
  }
  \label{fig:cnn-filter-activations}
\end{figure}

Strikingly, these salient visual structures emerge {\it without backpropagation or end-to-end optimization}: they are obtained simply by diagonalizing the Neural LoFi label-weighted moment operator. This shows that part of early visual feature learning is already captured by the low-degree supervised correlations emphasized by our theory. Depth then turns these simple local filters into more selective features. In Figure~\ref{fig:cnn-filter-activations}, the same LoFi feature behaves like a generic edge or contrast detector at shallow depth, but becomes increasingly selective after successive LoFi layers, responding only to more structured spatial patterns. This is, again, the qualitative mechanism suggested by the theory: each layer extracts a low-degree task-correlated component from the current representation, and repeated extraction builds progressively more structured features. 
Additional  experiments (including the spectrum of the operators, and different datasets) are presented in App.~\ref{app:numerical-extra}.

\section{Conclusion}\label{sec:conclusions}
We introduced \emph{Neural LoFi}, a tractable spectral theory of multilayer feature learning beyond the lazy regime. The framework turns deep representation learning into an explicit layerwise procedure: given a current representation, each layer selects low-degree features that are both task-correlated and simple in the geometry induced by previous layers. This yields a concrete relevance--complexity principle for feature selection, a quantitative criterion for when new features emerge, and a principle of \emph{low-degree compositionality}: depth helps when each layer makes the next useful signal statistically detectable.

This picture is visible beyond idealized models. Neural LoFi improves over random-feature baselines, predicts the sample scale at which layerwise features emerge, aligns with early gradient-descent representations, and recovers salient convolutional filters without backpropagation. Together, these results suggest that part of deep feature learning is already captured by supervised low-degree spectral structure.

The deliberate simplicity of Neural LoFi also delineates a clear research program. Theoretically, it remains to understand the long-time dynamics of the task-adaptive kernels generated by LoFi, their behavior beyond second-order filtering, and their extension to structured architectures such as attention. 

Algorithmically, the main challenge is to scale LoFi from a mechanistic surrogate into a useful training primitive---for initialization, feature pretraining, pruning, or diagnostics. Incorporating data reuse, multi-pass dynamics, and backward feature correction would bring the framework closer to trained networks.

\section*{Acknowledgments}
We would like to thank Bruno Loureiro and Lenka Zdeborov\'{a} for discussions and encouragement. We acknowledge funding from the Swiss National Science Foundation grants OperaGOST (grant number $200021\_200390$) and DSGIANGO (grant number $225837$). This work was supported by the Simons Collaboration on the Physics of Learning and Neural Computation via the Simons Foundation grant ($\#1257412$).  This work was also supported by the EPFL AI Center PhD Fellowship Program 2026.

\bibliographystyle{unsrt}
\bibliography{refs}

\newpage
\appendix

\section{Neural LoFi from Gradient Descent}
\label{app:GD-neural}
In this section, our goal is to prove Proposition \ref{thm:layerwise-GD}, whose formal version is stated below as Theorem \ref{thm:neural-lofi}, i.e. show that under layerwise training with vanishing initialization, the early-time dynamics of each layer produces spikes along the linear direction $\hat{\bu}_\ell$ and aligns with the top eigenvectors of $\widehat{C}^{(\ell)}$, for a standard fully-connected architecture without skip connections.

Consider the standard $L$-layer network
\begin{equation}
\hat f(\bx)=\langle a_L, \bz_L(\bx)\rangle,
\qquad
\bz_0(\bx)=\bx,
\qquad
z_\ell(x)=\sigma\!\bigl(W_\ell \bz_{\ell-1}(\bx)\bigr),\quad \ell=1,\dots,L,
\label{eq:sec2-arch}
\end{equation}
with square loss $\mathcal{L}=\tfrac{1}{2n}\sum_\mu(y_\mu-\hat f(\bx_\mu))^2$. The final readout $a_L\in\mathbb{R}^{p_L}$ is fixed throughout training; the weight matrices $W_\ell$ are updated layer-wise.

We write $w_{\ell,i}$ for the $i$-th row of $W_\ell$ and set
\begin{equation}\label{eq:init-scales}
\alpha_m := \max_i \|\bw_{m,i}(0)\|, \qquad m=1,\dots,L,
\end{equation}
with  $\alpha_{L+1}:=\|a_L\|$. 

\begin{assumption}[Layerwise scale separation]\label{ass:scale-sep}
The dataset, architecture, and activation are fixed. The scales, including the
final readout scale \(\alpha_{L+1}=\|a_L\|\), form a strict hierarchy: for
every $\ell$ and every $m\in\{\ell+1,\dots,L+1\}$,
\begin{equation}\label{eq:scale-hierarchy}
\alpha_m=o(\alpha_\ell).
\end{equation}
\end{assumption}

Thus later hidden layers, and the final readout, vanish on a strictly smaller
asymptotic scale than the layer currently being trained.  The hierarchy is used
below in a scale-counting sense: every replacement of a later-layer quantity by
its leading initialization-scale term carries at least one additional factor
from some \(\alpha_m\) with \(m>\ell\).

\begin{assumption}\label{ass:nonlin}
$\sigma:\mathbb{R}\to\mathbb{R}$ is three times continuously differentiable,
$\sigma(0)=0$, and $\|\sigma'\|_\infty,\|\sigma''\|_\infty,\|\sigma'''\|_\infty<\infty$.
\end{assumption}

Define the \emph{linear spike direction} and the \emph{label-weighted covariance} at layer $\ell$:
\begin{equation}\label{eq:sec2-linear-component}
\hat{\bu}_\ell
\coloneqq \frac{1}{n}\sum_{\mu=1}^n y_\mu\,\bz_{\ell-1}(\bx_\mu),
\qquad
\widehat{C}^{(\ell)}
\coloneqq \frac{1}{n}\sum_{\mu=1}^n
y_\mu\,\bz_{\ell-1}(\bx_\mu)\bz_{\ell-1}(\bx_\mu)^\top.
\end{equation}

\begin{theorem}[Layerwise GD approximation]\label{thm:neural-lofi}
Under Assumptions~\ref{ass:scale-sep}--\ref{ass:nonlin}, there exist constants
$\tau,C>0$ such that the layerwise gradient descent run admits
horizons
\begin{equation}\label{eq:layerwise-horizons}
T_\ell:=\left\lfloor \frac{\tau\alpha_\ell}{\eta}\right\rfloor,
\qquad \ell=1,\dots,L,
\end{equation}
with the following property: for every layer $\ell$, every neuron $i$, and every
$0\le t\le T_\ell$,
\[
\|w_{\ell,i}(t)\|\le C\alpha_\ell.
\]
Thus the layer receives an $O(\alpha_\ell)$ update while staying on its
initialization scale.  Define the leading effective readout
\begin{equation}\label{eq:effective-readout-fc}
\bar{a}_{\ell,i}
:= c_0^{L-\ell}\Bigl[\Bigl(\prod_{m=\ell+1}^{L}W_m(0)\Bigr)^\top a_L\Bigr]_i
\end{equation}
where \(c_0=\sigma'(0)\). Suppose that the effective readout for neuron $i$ is non-degenerate i.e.:
\begin{equation}\label{eq:readout-nondegenerate}
    |\bar a_{\ell,i}|\ge c_{\rm rd} \alpha_{L+1}\prod_{m=\ell+1}^{L}\alpha_m
\end{equation}
for a fixed \(c_{\rm rd}>0\). For such neurons and \(0\le t<T_\ell\),
\begin{equation}\label{eq:sec2-main-dynamics}
w_{\ell,i}(t+1) - w_{\ell,i}(t)
= \eta\,c_0\,\bar{a}_{\ell,i}\,\hat{\bu}_\ell
+ \eta\,\bar{a}_{\ell,i}\,c_1\,\widehat{C}^{(\ell)}\,w_{\ell,i}(t)
+ R_{\ell,i}(t),
\end{equation}
where \(c_1=\sigma''(0)\). Equivalently, the two leading coefficients are
\(c_{0,\ell i}:=c_0\bar a_{\ell,i}\) and
\(c_{1,\ell i}:=c_1\bar a_{\ell,i}\).  These initialization-dependent
coefficients are independent across neurons \(i\). For every \(0\le t\le T_\ell\) the accumulated residual obeys
\begin{equation}\label{eq:remainder-bound}
    \left\|
        \sum_{s=0}^{t-1}R_{\ell,i}(s)
    \right\|
    \le
    C\alpha_\ell^3 .
\end{equation}
Equivalently, the spike contributes \(O(\alpha_\ell)\) over the horizon, the
covariance-amplification term contributes \(O(\alpha_\ell^2)\), and the
discarded part is \(O(\alpha_\ell^3)\). Thus the displayed approximation keeps
the first two nontrivial orders in the layer-\(\ell\) initialization scale.
\end{theorem}

\begin{remark}
The three components in Equation~\eqref{eq:sec2-main-dynamics} have the following roles:
\begin{itemize}[noitemsep]
\item \textbf{Linear spike.} $\eta c_0 \bar{a}_{\ell,i}\hat{\bu}_\ell$ is independent of the current weight; over the horizon $T_\ell$ it moves the neuron by $O(\alpha_\ell)$ toward the empirical label-feature correlation.
\item \textbf{Covariance amplification.} $\eta \bar{a}_{\ell,i} c_1\widehat{C}^{(\ell)}w_{\ell,i}(t)$ is linear in the current weight and amplifies components aligned with the leading eigenvectors of $\widehat{C}^{(\ell)}$.
\item \textbf{Remainder.} The per-step remainder is second order in the current
layer scale, hence it accumulates to \(O(\alpha_\ell^3)\) over
\(T_\ell\asymp \alpha_\ell/\eta\) steps. Terms coming from later layers are
smaller by the scale hierarchy in Assumption~\ref{ass:scale-sep}.
\end{itemize}
The proof below shows that \(\bar a_{\ell,i}\) is the leading, sample-independent
backpropagation coefficient along the path from layer \(\ell\) to the readout.
The non-degeneracy condition only excludes neurons whose leading
backpropagation coefficient is accidentally cancelled; such neurons are inactive
at the displayed order.
\end{remark}

\subsection{Proof of Theorem~\ref{thm:neural-lofi}}

\begin{proof}
All constants below may depend on the fixed data and architecture bounds,
$L$ and the derivative bounds of $\sigma$.
We write \(K\) for a finite constant of this kind, changing from line to line,
and do not track separate data, width, or label constants.

For the standard architecture the gradient of the loss w.r.t.\
\(w_{\ell,i}\) involves a single backpropagation path through all layers
\(\ell+1,\dots,L\). Define the layer-to-layer Jacobian
\begin{equation}\label{eq:jacobian}
\bm{J}_{L\ell}(x)
:= \frac{\partial z_L(x)}{\partial z_\ell(x)}
= \prod_{m=\ell+1}^{L}\mathrm{diag}\!\bigl(\sigma'(\bm{h}_m(x))\bigr)W_m,
\quad \bm{h}_m(x)=W_m z_{m-1}(x),
\end{equation}
and the sample-dependent effective readout
\[
    \bar a_{\ell,i}(x):=[\bm J_{L\ell}(x)^\top a_L]_i .
\]
By the chain rule,
\begin{equation}\label{eq:full-grad}
\nabla_{w_{\ell,i}}\mathcal{L}
= -\frac{1}{n}\sum_{\mu=1}^n
  r_\mu\,\bar{a}_{\ell,i}(\bx_\mu)\,
  \sigma'\!\bigl(u_{\ell,i}(\bx_\mu)\bigr)\,\bz_{\ell-1}(\bx_\mu),
\quad r_\mu := y_\mu - \hat f(\bx_\mu),
\end{equation}
where $u_{\ell,i}(x)=\langle w_{\ell,i},\bz_{\ell-1}(x)\rangle$.

The key observation is that, under Assumption~\ref{ass:scale-sep}, the activation
derivatives along this path are well approximated by their value at zero, so
\(\bar a_{\ell,i}(x)\) is approximated by the initialization-dependent
coefficient in~\eqref{eq:effective-readout-fc}. On the layerwise horizons
considered below the representations are uniformly bounded and
\(\|W_m\|_{\mathrm{op}}\lesssim \alpha_m\). Hence
\(\|\bm h_m(x)\|\lesssim\alpha_m\), and Taylor expanding \(\sigma'\) around
zero gives
\begin{equation}\label{eq:diag-expand}
\mathrm{diag}(\sigma'(\bm{h}_m(x))) = c_0 I + c_1\,\mathrm{diag}(\bm{h}_m(x)) + O(\alpha_m^2),
\quad c_0=\sigma'(0),\;c_1=\sigma''(0).
\end{equation}
Let
\[
    \bm J_{L\ell}^{(0)}:=c_0^{L-\ell}\prod_{m=\ell+1}^L W_m(0)
\]
with the convention that the empty product is the identity. Expanding the
product for \(\bm J_{L\ell}(x)\) around this leading term gives the explicit
scale bound
\[
    \|\bm{J}_{L\ell}(x)-\bm J_{L\ell}^{(0)}\|_{\mathrm{op}}
    \le
    K
    \left(\prod_{m=\ell+1}^{L}\alpha_m\right)
    \left(\sum_{m=\ell+1}^{L}\alpha_m\right).
\]
Since each \(\alpha_m=o(\alpha_\ell)\) for \(m>\ell\), the right-hand side is
\(o(\alpha_\ell)\). Therefore
\begin{equation}\label{eq:Jac-correction}
\|\bm{J}_{L\ell}(x) - \bm{J}_{L\ell}^{(0)}\|_{\mathrm{op}}
= o(\alpha_\ell),
\end{equation}
and when \(\ell=L\) the left-hand side is zero. Multiplying by the final
readout gives, with \(\gamma_\ell=\alpha_{L+1}\prod_{m=\ell+1}^L\alpha_m\),
\[
    |\bar a_{\ell,i}(x)-\bar a_{\ell,i}|
    \le
    K\gamma_\ell
    \left(\sum_{m=\ell+1}^{L}\alpha_m\right),
\]
while \eqref{eq:readout-nondegenerate} gives the relative estimate
\begin{equation}\label{eq:readout-const-approx}
    \frac{|\bar{a}_{\ell,i}(x)-\bar{a}_{\ell,i}|}
    {|\bar a_{\ell,i}|}
    \le
    C
    \sum_{m=\ell+1}^{L}\alpha_m
    =
    o(\alpha_\ell).
\end{equation}
In particular,
\(\bar a_{\ell,i}(x)=\bar a_{\ell,i}(1+\varepsilon_{\ell,i}(x))\) with
\(\sup_x|\varepsilon_{\ell,i}(x)|=o(\alpha_\ell)\). The error is zero when \(\ell=L\).

We next translate the above bound to the required bound on the dynamics of $W_\ell$. We start by stating uniform bounds required throughout.  Since the sample,
widths, and activation are fixed, and since we may restrict to scales
\(\max_m\alpha_m\le1\), the row-norm condition
\(\|w_{j,i}\|\le2\alpha_j\) implies recursively that
\[
    \max_{\mu,j}\|z_j(\bx_\mu)\|\le K .
\]
Hence, at the start of
layer \(\ell\),
\[
\|\hat{\bu}_\ell\|
\le \frac1n\sum_\mu |y_\mu|\,\|\bz_{\ell-1}(\bx_\mu)\|
\le K,
\]
and
\[
\|\widehat C^{(\ell)}\|_{\mathrm{op}}
\le \frac1n\sum_\mu |y_\mu|\,\|\bz_{\ell-1}(\bx_\mu)\|^2
\le K .
\]

The same bounds give \(\|\bm J_{L\ell}^{(0)}\|_{\mathrm{op}}\le K\) and hence
uniformly bounded effective readouts.  The residuals \(r_\mu\) are uniformly
bounded as well.  Therefore
\(\|\nabla_{w_{\ell,i}}\mathcal L\|\le K\).  Choose \(\tau\) small enough that
\(\tau K\le 1\).  Then, for every
$t$ with $\eta t\le \tau\alpha_\ell$,
\[
\|w_{\ell,i}(t)\|
\le \|w_{\ell,i}(0)\|+\eta tK
\le 2\alpha_\ell .
\]

On this horizon,
\(|u_{\ell,i}(\bx_\mu)|\le K\alpha_\ell\).  Taylor's theorem and
Assumption~\ref{ass:nonlin} give
\begin{equation}\label{eq:sec2-sigma-linearization}
\sigma'(u) = c_0 + c_1\,u + \omega_\sigma(u),\quad
|\omega_\sigma(u)|\leq \tfrac{1}{2}\|\sigma'''\|_\infty u^2 .
\end{equation}
Moreover, by~\eqref{eq:readout-const-approx},
\(\bar a_{\ell,i}(\bx_\mu)=\bar a_{\ell,i}(1+\varepsilon_{\ell,i}(\bx_\mu))\)
with \(\sup_\mu|\varepsilon_{\ell,i}(\bx_\mu)|=o(\alpha_\ell)\).  Thus the
readout replacement error is smaller than the leading readout contribution
itself; quantitatively it is
\(|\bar a_{\ell,i}|\,o(\alpha_\ell)=o(\alpha_\ell^2)\), because
\(|\bar a_{\ell,i}|=O(\gamma_\ell)=o(\alpha_\ell)\).  By the hierarchy of scales, we further have
$|\hat f(\bx_\mu)|=o(\alpha_\ell^2)$, and
$r_\mu=y_\mu+o(\alpha_\ell^2)$.
Substituting these three estimates in~\eqref{eq:full-grad} yields
\[
\frac1n\sum_{\mu=1}^n
r_\mu\,\bar a_{\ell,i}(\bx_\mu)\,
\sigma'\!\bigl(u_{\ell,i}(\bx_\mu)\bigr)\,\bz_{\ell-1}(\bx_\mu)
=
\bar a_{\ell,i}\bigl[c_0\hat{\bu}_\ell
+c_1\widehat C^{(\ell)}w_{\ell,i}(t)\bigr]
+\mathcal E_{\ell,i}(t),
\]
so the two leading scalar coefficients in the main-text update
\eqref{eq:thm1approxmain} of Proposition~\ref{thm:layerwise-GD} are
\(c_{0,\ell i}=c_0\bar a_{\ell,i}\) and
\(c_{1,\ell i}=c_1\bar a_{\ell,i}\).  To see the independence, condition on
\(W_{\ell+2}(0),\ldots,W_L(0)\) and \(a_L\).  Then
\(\bar a_{\ell,i}\) is the inner product between the \(i\)-th column of
\(W_{\ell+1}(0)\) and a fixed downstream vector, up to the common scalar
\(c_0^{L-\ell}\).  The columns of \(W_{\ell+1}(0)\) are independent at
initialization, so the coefficients are independent across \(i\) at leading
order
with the per-step error bound
\[
\|\mathcal E_{\ell,i}(t)\|
\le
K\Bigl(
    |\bar a_{\ell,i}|\,\sup_\mu|\varepsilon_{\ell,i}(\bx_\mu)|
    + |\bar a_{\ell,i}|\alpha_\ell^2
    + o(\alpha_\ell^2)
\Bigr)
\le K\alpha_\ell^2 .
\]
Finally,
$w_{\ell,i}(t+1)-w_{\ell,i}(t)=-\eta\nabla_{w_{\ell,i}}\mathcal L$, so setting
$R_{\ell,i}(t):=\eta\mathcal E_{\ell,i}(t)$ proves
\eqref{eq:sec2-main-dynamics}.  Summing over \(0\le s<t\le T_\ell\) gives
\[
    \left\|
        \sum_{s=0}^{t-1}R_{\ell,i}(s)
    \right\|
    \le
    \eta t\,K\alpha_\ell^2
    \le
    K\alpha_\ell^3,
\]
which is \eqref{eq:remainder-bound}.
\end{proof}
\newpage
\section{Vector labels}
\label{app:vector-labels}
\providecommand{\bc}{\mathbf{c}}

In the main text we wrote Neural LoFi for scalar labels.  For vector labels
\(\by\in\mathbb R^c\), the natural object is not a signed scalar correlation
but the label-space vector of correlations.  Given the layer representation
\(\bz_{\ell-1}(\bx)\in\mathbb R^{p_{\ell-1}}\) and a scalar feature
\(\psi:\mathcal X\to\mathbb R\), define
\begin{equation}\label{eq:vec-corr}
    \bc_{\ell,n}[\psi]
    :=
    \widehat{\mathbb E}_n[\by\,\psi(\bx)]
    =
    \bigl(\widehat{\mathbb E}_n[y_a\psi(\bx)]\bigr)_{a=1}^c
    \in\mathbb R^c .
\end{equation}
For a fixed unit label direction \(\ba\), the corresponding scalar LoFi
correlation is
\[
    \widehat{\mathbb E}_n[
        \langle\ba,\by\rangle\,\psi(\bx)]
    =
    \langle\ba,\bc_{\ell,n}[\psi]\rangle .
\]
Optimizing this projected scalar correlation over \(\ba\) gives
\[
    \left\|\bc_{\ell,n}[\psi]\right\|_2^2
    =
    \max_{\|\ba\|_2=1}
    \left|
        \widehat{\mathbb E}_n[
            \langle\ba,\by\rangle\,\psi(\bx)]
    \right|^2
    .
\]
Thus the vector-label first-order LoFi criterion is
\begin{equation}\label{eq:vec-lin}
    \psi_{\ell,j}^{(1)}
    \in
    \arg\max_{\substack{\|\psi\|_{\mathcal H_{\ell-1}}=1\\
    \psi\perp \psi_{\ell,1}^{(1)},\dots,\psi_{\ell,j-1}^{(1)}}}
    \left\|
        \widehat{\mathbb E}_n[\by\,\psi(\bx)]
    \right\|_2^2.
\end{equation}
The vector-label second-order LoFi criterion is
\begin{equation}\label{eq:vec-quad}
    \varphi_{\ell,j}^{(2)}
    \in
    \arg\max_{\substack{\|\varphi\|_{\mathcal H_{\ell-1}}=1\\
    \varphi\perp \varphi_{\ell,1}^{(2)},\dots,\varphi_{\ell,j-1}^{(2)}}}
    \left\|
        \widehat{\mathbb E}_n[\by\,\varphi(\bx)^2]
    \right\|_2^2.
\end{equation}

\subsection{Parametric implementation}

We now explain how the feature-selection criteria
\eqref{eq:vec-lin}--\eqref{eq:vec-quad} translate into concrete spectral
computations when the candidate features are linear functions of the current
representation. The first-order criterion becomes an SVD of a label-feature
matrix, while the second-order criterion becomes an alternating Rayleigh
maximization over a label direction and a feature direction.

\paragraph{First-order term: SVD.}
For a linear feature
\(\psi_{\bv}(\bx)=\langle\bv,\bz_{\ell-1}(\bx)\rangle\), define
\[
    \widehat M_\ell
    :=
    \widehat{\mathbb E}_n[\by\,\bz_{\ell-1}^\top]
    \in\mathbb R^{c\times p_{\ell-1}} .
\]
Then \(\bc_{\ell,n}[\psi_{\bv}]=\widehat M_\ell\bv\), so
\eqref{eq:vec-lin} becomes
\[
    \max_{\|\bv\|_2=1}
    \|\widehat M_\ell\bv\|_2^2 .
\]
The maximizing feature directions are the right singular vectors of
\(\widehat M_\ell\), while the associated label directions are the left
singular vectors.  In the first layer this is exactly the SVD of
\(\widehat{\mathbb E}_n[\by\,\bx^\top]\).

\paragraph{Second-order term: alternating Rayleigh maximization.}
For a linear second-order feature
\(\varphi_{\bv}(\bx)=\langle\bv,\bz_{\ell-1}(\bx)\rangle\), define
\[
    \widehat{\bm r}_\ell(\bv)
    :=
    \widehat{\mathbb E}_n[
        \langle\bv,\bz_{\ell-1}\rangle^2\by]
    \in\mathbb R^c .
\]
Then \eqref{eq:vec-quad} is
\[
    \max_{\|\bv\|_2=1}
    \|\widehat{\bm r}_\ell(\bv)\|_2^2
    =
    \max_{\|\bv\|_2=1,\ \|\ba\|_2=1}
    \left|
        \bv^\top \widehat C_\ell(\ba)\bv
    \right|^2,
\]
where
\[
    \widehat C_\ell(\ba)
    :=
    \widehat{\mathbb E}_n[
        \langle\ba,\by\rangle\,
        \bz_{\ell-1}\bz_{\ell-1}^\top]
    \in\mathbb R^{p_{\ell-1}\times p_{\ell-1}} .
\]
For fixed \(\ba\), the best \(\bv\) is the signed leading eigenvector of
\(\widehat C_\ell(\ba)\), ordered by \(|\lambda|\).  For fixed \(\bv\), the
best label direction is
\[
    \ba
    =
    \widehat{\bm r}_\ell(\bv)/
    \|\widehat{\bm r}_\ell(\bv)\|_2 .
\]
This gives the alternating maximization used in the vector-label algorithm.

Combining the two spectral steps gives the vector-label Neural LoFi procedure:
at each layer we keep the leading right singular directions from the
first-order matrix, add the directions found by the second-order alternating
Rayleigh step, orthogonalize the resulting feature subspace, and then apply the
same random feature lift as in scalar Neural LoFi.

\begin{algorithm}[H]
\caption{\textcolor{blue}{Neural LoFi for vector labels}}
\label{alg:vector-neural-lofi}
\small
\begin{algorithmic}[1]
\Require Dataset $\{(\bx_\mu,\by_\mu)\}_{\mu=1}^n$ with
$\by_\mu\in\mathbb R^c$, depth $L$, first-order ranks
$\{q_\ell\}_{\ell=1}^L$, second-order ranks $\{k_\ell\}_{\ell=1}^L$,
widths $\{p_\ell\}_{\ell=1}^L$, inner iterations $T$, tolerance $\varepsilon$.
\State Initialize $\bz_0(\bx)\gets \bx$ and $p_0=d$.
\For{$\ell=1,\dots,L$}
    \State Form data matrices
    $Y=[\by_1,\dots,\by_n]^\top$ and
    $Z_{\ell-1}=[\bz_{\ell-1}(\bx_1),\dots,\bz_{\ell-1}(\bx_n)]^\top$.
    \State Compute the first-order SVD
    \[
        \widehat M_\ell
        =
        \frac1nY^\top Z_{\ell-1}
        =
        \widehat U_\ell\widehat\Sigma_\ell
        (\widehat V_\ell^{(1)})^\top ,
    \]
    and retain the top \(q_\ell\) right singular vectors in
    \(\widehat V_{\ell,q}^{(1)}\).
    \State Initialize \(\ba^{(0)}\in\mathbb S^{c-1}\), e.g. the leading left
    singular vector in \(\widehat U_\ell\).
    \For{$t=0,\dots,T-1$}
        \State Form the projected-label covariance
        \[
            \widehat C_{\ba^{(t)}}^{(\ell)}
            \gets
            \frac1n
            \sum_{\mu=1}^n
            \langle\ba^{(t)},\by_\mu\rangle\,
            \bz_{\ell-1}(\bx_\mu)\bz_{\ell-1}(\bx_\mu)^\top .
        \]
        \State Compute the signed leading eigenpair
        \[
            (\widehat\lambda^{(t)},\widehat\bv^{(t)})
            \in
            \arg\max_{\lambda,\ \|\bv\|_2=1}
            |\lambda|
            \quad\text{s.t.}\quad
            \widehat C_{\ba^{(t)}}^{(\ell)}\bv=\lambda\bv .
        \]
        \State Compute the label-correlation vector
        \[
            \widehat{\bm r}^{(t)}
            \gets
            \frac1n
            \sum_{\mu=1}^n
            \langle\widehat\bv^{(t)},\bz_{\ell-1}(\bx_\mu)\rangle^2
            \by_\mu .
        \]
        \State Update
        \(\ba^{(t+1)}
        \gets
        \widehat{\bm r}^{(t)}/\|\widehat{\bm r}^{(t)}\|_2\).
        \State \textbf{if}
        \(t>0\) and
        \(|\|\widehat{\bm r}^{(t)}\|_2-\|\widehat{\bm r}^{(t-1)}\|_2|
        <\varepsilon\)
        \textbf{then break}.
    \EndFor
    \State Set \(\ba_\ell^\star\gets \ba^{(t)}\).  Form
    \(\widehat C_{\ba_\ell^\star}^{(\ell)}\) and extract
    \(k_\ell\) eigenvectors by decreasing \(|\widehat\lambda|\):
    \[
        \widehat V_\ell^{(2)}
        =
        [\widehat\bv_1^{(\ell)},\dots,\widehat\bv_{k_\ell}^{(\ell)}] .
    \]
    \State Combine first- and second-order directions
    \[
        \widehat V_\ell
        \gets
        \operatorname{orth}
        \bigl([\widehat V_{\ell,q}^{(1)},\widehat V_\ell^{(2)}]\bigr).
    \]
    \State Project:
    \[
        \bg_\ell(\bx)
        \gets
        \widehat V_\ell^\top\bz_{\ell-1}(\bx).
    \]
    \State Lift:
    \[
        \bz_\ell(\bx)
        \gets
        p_\ell^{-1/2}\sigma(R_\ell\bg_\ell(\bx)),
        \qquad
        R_\ell\in\mathbb R^{p_\ell\times(q_\ell+k_\ell)} .
    \]
\EndFor
\State Fit a final linear or logistic readout
\(A\in\mathbb R^{c\times p_L}\) on \(\bz_L(\bx)\).
\State \Return Final predictor \(\widehat{\by}(\bx)=A\bz_L(\bx)\).
\end{algorithmic}
\end{algorithm}

\subsection{Connection with gradient descent}

The preceding subsection gives a static algorithmic rule. We now explain why
the same spectral objects arise from the early gradient-descent dynamics of a
vector-output network. We start with the two-layer setting, where the updates
of the readout and hidden weights form a coupled power iteration on the
label-feature matrix, and then indicate why the same mechanism is local in
depth.

Consider a two-layer vector-output network
\[
    \widehat{\by}(\bx)
    =
    A\sigma(W\bz(\bx))
    =
    \sum_{i=1}^p \ba_i\sigma(\langle\bw_i,\bz(\bx)\rangle),
    \qquad
    A=[\ba_1,\dots,\ba_p]\in\mathbb R^{c\times p},
    \quad
    W=\begin{bmatrix}\bw_1^\top\\ \vdots\\ \bw_p^\top\end{bmatrix},
\]
trained with squared loss and residual
\(\be_\mu=\by_\mu-\widehat{\by}(\bx_\mu)\).  Let
\[
    \widehat M
    :=
    \widehat{\mathbb E}_n[\by\,\bz^\top]
    \in\mathbb R^{c\times p_{\ell-1}} .
\]
At small initialization scale, \(\be=\by+o(1)\), while
\(\sigma(W\bz)=c_0W\bz+\text{higher-order terms}\) and
\(\sigma'(\langle\bw_i,\bz\rangle)=c_0+\text{higher-order terms}\).  Hence
the readout update is
\[
    \frac{dA}{dt}
    =
    \widehat{\mathbb E}_n[\be\,\sigma(W\bz)^\top]
    =
    c_0\,\widehat M W^\top
    +
    \text{higher-order terms}.
\]
Thus \(A\) is driven by the right projection of the label-feature matrix
\(\widehat M\) through the current hidden directions \(W\): column \(i\) obeys
\(\frac{d\ba_i}{dt}\simeq c_0\widehat M\bw_i\).  Conversely, the hidden weights obey
\begin{equation}\label{eq:vector-two-layer-hidden-gradient}
    \frac{d\bw_i}{dt}
    =
    \widehat{\mathbb E}_n[
        \langle\be,\ba_i\rangle\,
        \sigma'(\langle\bw_i,\bz\rangle)\bz].
\end{equation}
The same expansion gives
\[
    \frac{d\bw_i}{dt}
    =
    c_0\,\widehat M^\top\ba_i
    +
    \text{higher-order terms},
    \qquad
    \text{or equivalently}
    \qquad
    \frac{dW}{dt}
    =
    c_0\,A^\top\widehat M
    +
    \text{higher-order terms}.
\]
Thus \(W\) is driven by the left projection of \(\widehat M\) through the
current readout directions \(A\).  Ignoring the higher-order terms, the
leading dynamics are the coupled power iteration
\[
    \frac{dA}{dt}\simeq c_0\widehat M W^\top,
    \qquad
    \frac{dW^\top}{dt}\simeq c_0\widehat M^\top A.
\]
After normalization and deflation, the coupled iteration aligns the readout
directions with the left singular vectors of \(\widehat M\) and the hidden
directions with the corresponding right singular vectors.  This is the
gradient-descent mechanism behind the first-order vector LoFi SVD of
\(\widehat{\mathbb E}_n[\by\,\bz^\top]\).

The same argument is local in depth.  For a general layer \(\ell\), view
\(\bz_{\ell-1}\) as the input to the block
\(\bz_\ell=\sigma(W_\ell\bz_{\ell-1})\), and view the next-layer weights
\(W_{\ell+1}\), together with the backpropagated residual at layer
\(\ell+1\), as the vector readout.  Then the update of \(W_{\ell+1}\) plays the
role of the update of \(A\) above, while the update of \(W_\ell\) plays the
role of the update of \(W\).  Thus the same left/right projection mechanism
produces the layerwise SVD picture.

Keeping also the term linear in \(\bw_i\) in the hidden-weight expansion gives
\begin{equation}\label{eq:vector-two-layer-projected-lofi}
    \frac{d\bw_i}{dt}
    =
    c_0\,\widehat{\mathbb E}_n[
        \langle\by,\ba_i\rangle\bz]
    +
    c_1\,\widehat C(\ba_i)\bw_i
    +
    \text{higher-order terms}.
\end{equation}
Thus each readout direction \(\ba_i\) induces the scalar projected label
\(\langle\by,\ba_i\rangle\), and hence the projected-label covariance
\(\widehat C(\ba_i)\) appearing in Algorithm~\ref{alg:vector-neural-lofi}.

This already marks an important difference from the scalar-label case.  When
\(\by\in\mathbb R^c\), the linear label-feature matrix
\(\widehat M=\widehat{\mathbb E}_n[\by\bz^\top]\) has up to \(c\) nonzero left
singular directions, and its right singular vectors can therefore expose
multiple feature directions from the first-order signal alone.  Thus, for
vector labels, the SVD step on \(\widehat M\) may already be sufficient for
feature learning before one invokes the projected second-order covariance
\(\widehat C(\ba)\).

\newpage

\section{Partial whitening and a covariance correction}
\label{app:covariance-correction}

Theorem~\ref{thm:neural-lofi} replaces the residual $y-\hat f$ by $y$, so a
candidate feature $\phi$ is selected by maximizing its correlation with the
label. To retain the leading effect of $\hat f$, one can approximate it by
$\phi$ and instead minimize the empirical squared loss
\[
    \frac12\widehat{\mathbb E}_n[(y-\phi)^2]
    =
    \frac12\widehat{\mathbb E}_n[y^2]
    -\widehat{\mathbb E}_n[y\phi]
    +\frac12\widehat{\mathbb E}_n[\phi^2].
\]
Relative to correlation maximization, this adds the variance penalty
$\widehat{\mathbb E}_n[\phi^2]$. For
$\phi_{\bw}=\langle\bw,\bz_{\ell-1}\rangle$,
$\widehat{\mathbb E}_n[\phi_{\bw}^2]
=\bw^\top\widehat\Sigma^{(\ell)}\bw$, with
$
    \widehat\Sigma^{(\ell)}
    :=
    \widehat{\mathbb E}_n[
        \bz_{\ell-1}(\bx)\bz_{\ell-1}(\bx)^\top]
$. A tractable way to incorporate this covariance is regularized whitening
on its leading empirical eigenspace. Let
$\widehat\Sigma^{(\ell)}=\sum_{j=1}^{p_{\ell-1}}s_j\bv_j\bv_j^\top$
with $s_1\geq\cdots\geq s_{p_{\ell-1}}$, and define the regularized
top-$m$ whitening operator
\[
    P_{\lambda,m}^{(\ell)}
    :=
    \sum_{j=1}^{m}(s_j+\lambda)^{-1/2}\bv_j\bv_j^\top,
    \qquad
    \widetilde\bz_{\ell-1}=P_{\lambda,m}^{(\ell)}\bz_{\ell-1}.
\]
The empirical first-order correlation vector and second-order label-weighted
operator computed from $\widetilde\bz_{\ell-1}$ are, respectively,
\begin{equation}\label{eq:whitened-lofi-objects}
    \widetilde\bu_\ell
    =P_{\lambda,m}^{(\ell)}\hat\bu_\ell,
    \qquad
    \widetilde C^{(\ell)}
    =P_{\lambda,m}^{(\ell)}\widehat C^{(\ell)}P_{\lambda,m}^{(\ell)}.
\end{equation}

For the first-order feature, mapping the correlation direction back to the
original coordinates gives
\[
    \bw_{\lambda,m}^{(1)}
    =\bigl(P_{\lambda,m}^{(\ell)}\bigr)^2\hat\bu_\ell
    =\sum_{j=1}^{m}
      \frac{\langle\bv_j,\hat\bu_\ell\rangle}{s_j+\lambda}\bv_j.
\]
This is the ridge coefficient restricted to the leading empirical covariance
eigenspace. If $m=p_{\ell-1}$, it is exactly
\[
    \argmin_{\bw}
    \left\{
      \widehat{\mathbb E}_n[
        (y-\langle\bw,\bz_{\ell-1}\rangle)^2]
      +\lambda\|\bw\|_2^2
    \right\}
    =
    (\widehat\Sigma^{(\ell)}+\lambda I)^{-1}\hat\bu_\ell.
\]
For the second-order feature, the ordinary eigendecomposition of
$\widetilde C^{(\ell)}$ in~\eqref{eq:whitened-lofi-objects} replaces the
eigendecomposition of $\widehat C^{(\ell)}$. Under full whitening, a mapped-back
direction $\bw$ therefore solves the generalized eigenvalue problem
\[
    \widehat C^{(\ell)}\bw
    =\gamma(\widehat\Sigma^{(\ell)}+\lambda I)\bw;
\]
top-$m$ whitening gives the corresponding problem restricted to
$\operatorname{span}\{\bv_1,\ldots,\bv_m\}$.

The covariance correction also appears directly by retaining the next term in
the squared-loss dynamics. Write $t=\langle\bw,\bz_{\ell-1}\rangle$ and use
the local quadratic approximation
\[
    \phi_{\bw}(\bx)=c_0t+\frac{c_1}{2}t^2,
    \qquad
    \mathcal L_{\rm loc}(\bw)
    =\frac12\widehat{\mathbb E}_n[(y-\phi_{\bw}(\bx))^2].
\]
Then
\begin{align}
    -\nabla_{\bw}\mathcal L_{\rm loc}(\bw)
    &=
    \widehat{\mathbb E}_n[
        (y-\phi_{\bw})(c_0+c_1t)\bz_{\ell-1}] \notag\\
    &=
    c_0\hat\bu_\ell
    +\bigl(c_1\widehat C^{(\ell)}
      -c_0^2\widehat\Sigma^{(\ell)}\bigr)\bw
    -\frac{3c_0c_1}{2}\widehat{\mathbb E}_n[t^2\bz_{\ell-1}]
    -\frac{c_1^2}{2}\widehat{\mathbb E}_n[t^3\bz_{\ell-1}] \notag\\
    &=
    c_0\hat\bu_\ell
    +\bigl(c_1\widehat C^{(\ell)}
      -c_0^2\widehat\Sigma^{(\ell)}\bigr)\bw
    +O(\|\bw\|_2^2).
    \label{eq:squared-loss-covariance-correction}
\end{align}
The term $-c_0^2\widehat\Sigma^{(\ell)}\bw$ is therefore the first correction
to the correlation-based dynamics in Theorem~\ref{thm:neural-lofi}. After
treating the first-order forcing separately, the corresponding second-order
iteration is
\[
    \bw^{(t+1)}
    \leftarrow
    \operatorname{normalize}\!\left(
      \bw^{(t)}+\eta
      (c_1\widehat C^{(\ell)}-c_0^2\widehat\Sigma^{(\ell)})\bw^{(t)}
    \right).
\]

\newpage

\begingroup

\section{Higher-order Decay Ansatz}
\label{app:higher-order-averaging}

\subsection{A Motivation: averaging in the gradient}
In the main text, \Cref{thm:layerwise-GD} motivates Neural LoFi through the
label-weighted moment operator, and Appendix~\ref{app:GD-neural} derives this
operator from a layer-wise small-initialization expansion.  That argument
orders terms by their initialization scale.  The present section gives a different, and
complementary, justification for suppressing higher-order terms that does not
rely on layer-wise scale separation, and is therefore more natural from the point of view of machine learning practices. The point is instead that the gradient
is an average over data: even if a higher-order interaction is present and can
matter on individual samples, its contribution to the averaged gradient is
controlled by its correlation with the corresponding label-weighted degree
component after lower-order correlations have been removed.

Fix a layer \(\ell\) and write \(Z=\bz_{\ell-1}(X)\).  For a neuron initialized
at direction \(w_0\), the exact layerwise GD update direction from
\eqref{eq:full-grad} is, up to the global sign convention,
\[
    \widehat G_{\ell,i}^{\mathrm{GD}}(w_0)
    =
    \widehat{\E}_n
    \left[
        r\,\bar a_{\ell,i}(X)\,
        \sigma'(\langle w_0,Z\rangle)Z
    \right],
    \qquad
    r:=y-\hat f(X),
\]
where \(\bar a_{\ell,i}(X)\) is the sample-dependent effective readout coming
from the later layers.  For this averaging argument, we treat this readout as
constant in expectation and assume its fluctuations vanish after averaging over
data.  We also ignore the self-interaction from \(\hat f(X)\) in
\(r=y-\hat f(X)\), replacing \(r\) by \(y\) at leading order.  After absorbing
the constant readout into the neuron-dependent scalar coefficient, the gradient
direction relevant for the averaging argument is therefore
\[
    \widehat G_\ell(w_0)
    =
    \widehat{\E}_n
    \left[
        y\,\sigma'(\langle w_0,Z\rangle)Z
    \right].
\]
To see which feature directions this gradient can reveal, test it against a
reference direction \(v\):
\[
    \langle \widehat G_\ell(w_0),v\rangle
    =
    \widehat{\E}_n
    \left[
        y\,\sigma'(\langle w_0,Z\rangle)\langle Z,v\rangle
    \right].
\]
At population level, let \(P_\ell\) be the law of \(Z\), and choose an
orthonormal polynomial basis
\(\{\psi_{\ell,k,a}\}_{k,a}\) for the polynomial span in \(L^2(P_\ell)\),
chosen so that for each \(k\), \(\{\psi_{\ell,k,a}\}_a\) is an orthonormal
basis of the degree-\(k\) subspace obtained after removing all correlations
with degrees \(<k\).  This basis lets us decompose the tested projection
\(\langle G_\ell(w_0),v\rangle\) into contributions from different degree
interactions.  The first two contributions are exactly the Neural LoFi terms:
the degree-zero interaction gives the linear part \(\langle u_\ell,v\rangle\),
with \(u_\ell:=\E[yZ]\), while the degree-one interaction gives the
second-order part \(\langle v,C_\ell w_0\rangle\), with
\(C_\ell:=\E[yZZ^\top]\); empirically this is
\(\langle v,\widehat C^{(\ell)}w_0\rangle\).  The higher-degree interactions
are the terms controlled by the ansatz below.  If
\[
    a_{\ell,v}(Z):=\E[y\langle Z,v\rangle\mid Z],
    \qquad
    s_{\ell,w_0}(Z):=\sigma'(\langle w_0,Z\rangle),
\]
then both functions can be expanded in this basis.  Writing the coefficients as
\[
    \alpha_{\ell,k,a}(v)
    :=
    \langle a_{\ell,v},\psi_{\ell,k,a}\rangle_{L^2(P_\ell)},
    \qquad
    \beta_{\ell,k,a}(w_0)
    :=
    \langle s_{\ell,w_0},\psi_{\ell,k,a}\rangle_{L^2(P_\ell)},
\]
the population analogue \(G_\ell\) of \(\widehat G_\ell\) decomposes as
\[
    \begin{aligned}
    \langle G_\ell(w_0),v\rangle
    &=
    \E[a_{\ell,v}(Z)s_{\ell,w_0}(Z)] \\
    &=
    \sum_{k\ge0}\sum_a
    \alpha_{\ell,k,a}(v)\beta_{\ell,k,a}(w_0).
    \end{aligned}
\]
Equivalently, writing \(\Pi_{\ell,k}\) for the projection onto this
degree-\(k\) subspace, defined below,
\[
    \langle G_\ell(w_0),v\rangle
    =
    \sum_{k\ge0}
    \left\langle
        \Pi_{\ell,k}a_{\ell,v},
        \Pi_{\ell,k}s_{\ell,w_0}
    \right\rangle_{L^2(P_\ell)} .
\]
Thus the \(k\)-th term in the expansion is a correlation between two
degree-\(k\) polynomials, each orthogonal to all lower-degree terms.  For a
fixed data distribution and reference direction \(v\), the first polynomial is
fixed; since \(w_0\) is random, the second polynomial behaves like a randomly
chosen degree-\(k\) direction.  Heuristically, this is the same dimensional
calculation as the correlation between two generic degree-\(k\) polynomials.
If \(Z\) has \(m_\ell\)
degrees of freedom, the space of degree-\(k\) polynomials after removing
lower-order terms has dimension of order \(m_\ell^k\) for fixed \(k\).  A
random normalized vector in an \(m_\ell^k\)-dimensional space has squared
projection of order \(m_\ell^{-k}\) on any fixed normalized direction.

\begin{center}
\setlength{\fboxsep}{6pt}
\fbox{%
\begin{minipage}{0.90\linewidth}
\textbf{Neural Manifold Ansatz.}
Fix \(k\) and a layer \(\ell\), and write \(Z=z_{\ell-1}(X)\).  There exists
a scale \(m_\ell\) such that, for any fixed normalized degree-\(k\) polynomial
\(p_k(Z)\) orthogonal to lower-degree polynomials in \(L^2(P_\ell)\), any
fixed normalized degree-\(k\) univariate polynomial \(q_k\) orthogonal to lower
degrees, and \(w_0\sim {\rm Unif}(\mathbb S^{p_{\ell-1}-1})\),
\[
    \left|
        \E_{Z\sim P_\ell}\!
        \left[p_k(Z)q_k(\langle w_0,Z\rangle)\right]
    \right|^2_{\mathrm{typ}}
    \asymp_k
    m_\ell^{-k}.
\]
We call \(m_\ell\) the effective number of degrees of freedom in the layer
representation.  Equivalently, the unsquared overlap is typically of size
\(\asymp_k m_\ell^{-k/2}\).
\end{minipage}}
\end{center}

\paragraph{Equivalent tensor form.}

The above ansatz can equivalently be cast in terms of the scaling of random projections on higher-order moment tensors. For scalar labels, define the raw order-\(k\) label tensor
\[
    T_{\ell,k}^{\rm raw}
    :=
    \E[y\,Z^{\otimes k}].
\]
To remove lower-order correlations, associate a symmetric tensor \(T\) with
the polynomial \(Z\mapsto\langle T,Z^{\otimes k}\rangle\), project this
polynomial onto the degree-\(k\) component orthogonal to all lower degrees,
and write \(T_{\ell,k}^{\rm res}\) for the corresponding residualized tensor
obtained from \(T_{\ell,k}^{\rm raw}\).  In these coordinates, only the fixed
label tensor is residualized; the random initialization enters through the
raw rank-one tensor \(w_0^{\otimes k}\).  Thus the ansatz can also be read as
\[
    \left|
        \left\langle
        \frac{T_{\ell,k}^{\rm res}}{\|T_{\ell,k}^{\rm res}\|},
        \frac{w_0^{\otimes k}}{\|w_0^{\otimes k}\|}
        \right\rangle
    \right|_{\rm typ}^2
    \asymp_k
    m_\ell^{-k},
\]
For Gaussian or Hermite coordinates, \(w_0^{\otimes k}\) is the usual
degree-\(k\) random rank-one tensor.

\subsection{Why this scaling is natural}
The reason for \(m_\ell^{-k}\) is dimension counting.  If the current
representation has \(m_\ell\) independent coordinates, then the number of
degree-\(k\) monomials after lower degrees have been removed is
\[
    \binom{m_\ell+k-1}{k}
    \asymp_k
    m_\ell^k .
\]
Thus the degree-\(k\) component lives in a space with about \(m_\ell^k\)
orthogonal directions.  A random unit vector in an \(N\)-dimensional space has
squared projection \(1/N\) on a fixed unit vector; taking
\(N\asymp_k m_\ell^k\) gives the ansatz.  The assumption is therefore not that
higher-order terms are absent on individual samples.  It is that, after
averaging over data, a generic random degree-\(k\) component spreads its mass
over the full degree-\(k\) space rather than aligning with the particular
label-weighted component.

For Gaussian \(Z\sim N(0,I_m)\), this count is exact at the level needed
here.  The orthogonal degree-\(k\) basis is the Hermite basis and has
dimension \(\binom{m+k-1}{k}\).  If \(q_k\) were uniform on this unit sphere,
then \(\E|\langle p_k,q_k\rangle|^2=1/\binom{m+k-1}{k}\asymp_k m^{-k}\).
The random ridge feature \(\rho(\langle w_0,Z\rangle)\) has degree-\(k\)
Hermite component represented by \(w_0^{\otimes k}\).  Rotational invariance
forces the covariance of this random component to be a scalar multiple of the
identity on the degree-\(k\) Hermite space, and the trace normalization gives
the same \(m^{-k}\) squared-overlap scaling.

The same picture applies to an \(m\)-dimensional product distribution.  If
\(U=(U_1,\ldots,U_m)\) has product law and \(\varphi_{j,r}\) are the
one-dimensional orthonormal polynomials, then the degree-\(k\) products
\(\prod_j\varphi_{j,\alpha_j}(U_j)\), \(|\alpha|=k\), form
\(\binom{m+k-1}{k}\asymp_k m^k\) orthogonal directions.  For data lying near
an \(m\)-dimensional manifold, the same substitution \(d\mapsto m\) is
expected when the degree-\(k\) restrictions to the manifold remain
nondegenerate, so that the order-\(k\) polynomial features span a space of
dimension comparable to \(m^k\).  This is the effective-dimension content of
the ansatz.

\subsection{Consequence for the layerwise Neural LoFi approximation.}
The population analogues of the first- and second-order Neural LoFi operators
are
\[
    u_\ell:=\E[yZ],
    \qquad
    C_\ell:=\E[yZZ^\top]
\]
In the layerwise GD approximation discussed above, the full neuron
update is replaced by
\[
    G_\ell(w_0)=\E[y\sigma'(\langle w_0,Z\rangle)Z].
\]
For a reference direction \(v\), set
\[
    a_{\ell,v}(Z):=\E[y\langle Z,v\rangle\mid Z],
    \qquad
    s_{\ell,w_0}(Z):=\sigma'(\langle w_0,Z\rangle).
\]
Then
\[
    \langle G_\ell(w_0),v\rangle
    =
    \langle a_{\ell,v},s_{\ell,w_0}\rangle_{L^2(P_\ell)}.
\]
Now expand \(s_{\ell,w_0}\) in an orthonormal polynomial basis for the law
\(P_\ell\), grouped by degree after lower-degree correlations have been
removed.  Writing \(\Pi_{\ell,r}\) for the degree-\(r\) projection, the first
two components have the form
\[
    \Pi_{\ell,0}s_{\ell,w_0}=c_0,
    \qquad
    \Pi_{\ell,1}s_{\ell,w_0}=c_1\langle w_0,Z\rangle,
\]
where \(c_0,c_1\) are scalar expansion coefficients.  For \(r\ge2\), write
\[
    \Pi_{\ell,r}s_{\ell,w_0}
    =
    c_r q_{\ell,r,w_0},
    \qquad
    \|q_{\ell,r,w_0}\|_{L^2(P_\ell)}=1.
\]
These coefficients come from the orthonormal-polynomial expansion of
\(\sigma'(\langle w_0,Z\rangle)\), not from a Taylor expansion of \(\sigma'\).
By orthogonality across residual degrees,
\[
    \langle G_\ell(w_0),v\rangle
    =
    c_0\langle u_\ell,v\rangle
    +
    c_1\langle C_\ell w_0,v\rangle
    +
    \sum_{r\ge2}c_r R_{\ell,r}(w_0,v),
\]
where
\[
    R_{\ell,r}(w_0,v)
    :=
    \left\langle
        \Pi_{\ell,r}a_{\ell,v},
        q_{\ell,r,w_0}
    \right\rangle_{L^2(P_\ell)}.
\]
The averaging ansatz says that for each fixed \(v\), the residual
degree-\(r\) projection created by a random initialization direction satisfies
\begin{equation}
    \boxed{
    |R_{\ell,r}(w_0,v)|_{\mathrm{typ}}
    \lesssim_r
    m_\ell^{-r}.
    }
    \label{eq:higher-order-decay-ansatz}
\end{equation}
Consequently,
\[
    \langle G_\ell(w_0),v\rangle
    =
    c_0\langle u_\ell,v\rangle
    +
    c_1\langle C_\ell w_0,v\rangle
    +
    O_{\mathrm{typ}}(m_\ell^{-2}),
\]
with higher corrections of order \(m_\ell^{-r}\) at degree \(r\).  This is the
averaging mechanism behind the Neural LoFi approximation: the first two
terms are exactly the population first- and second-order LoFi operators, while
the higher-order ridge contributions are suppressed after averaging over data.

\subsection{CIFAR test of the ansatz}
\Cref{fig:lofi-layer-higher-order-decay} tests the ansatz on real data by measuring
the same random-overlap effect without explicitly constructing the full
degree-\(k\) polynomial space.  Each random initialization direction gives a
one-dimensional ridge slice of the current representation.  For random
directions \(w_{\ell,i}(0)\), set
\[
    a_{\mu,i}
    =
    \langle w_{\ell,i}(0),Z_\mu\rangle .
\]
For each \(i\), orthonormalize
\(1,a_{\mu,i},a_{\mu,i}^2,\ldots,a_{\mu,i}^{k_{\max}}\) by QR on the
reference samples, and let \(\phi_{\ell,k,i}\) be the resulting degree-\(k\)
feature.  Thus \(\phi_{\ell,k,i}\) is the empirical version of the
degree-\(k\) component of a random ridge feature after the lower-degree
components along the same slice have been removed.  It satisfies
\[
    \widehat{\E}_N[
        \phi_{\ell,j,i}(X)\phi_{\ell,k,i}(X)]
    =
    \delta_{jk}.
\]
We then plot
\[
    \widehat B_{\ell,k}
    =
    \mathrm{median}_i
    \left|
        \widehat{\E}_N[
            y\,\phi_{\ell,k,i}(X)]
    \right|^2
\]
after subtracting the fitted \(1/n\) finite-sample floor.  The quantity inside
the median is the squared empirical correlation between the label and a
normalized residual degree-\(k\) random ridge component.  Taking the median
over \(i\) estimates the typical random-initialization overlap rather than a
rare aligned direction.  Thus \(\widehat B_{\ell,k}\) is the empirical analogue
of the squared random-overlap effect underlying
\eqref{eq:higher-order-decay-ansatz}; the observed decay with \(k\) is
consistent with a degree-\(k\) polynomial space whose dimension grows like
\(m_\ell^k\).

\begin{figure}[H]
    \centering
    \includegraphics[width=0.65\linewidth]{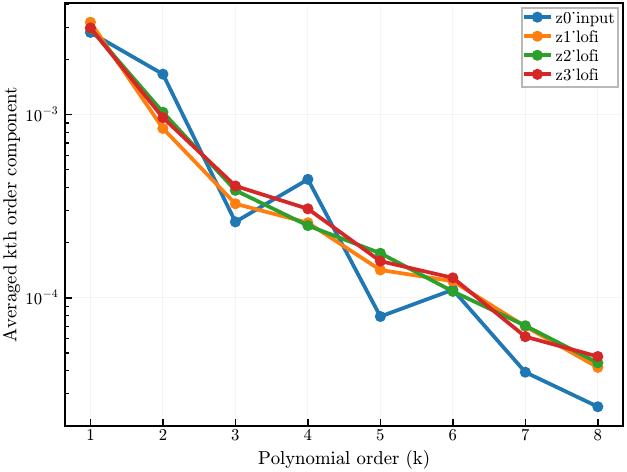}
    \caption{textbf{Layer-wise higher-order decay on CIFAR-10.}
    For a three-layer fully connected Neural LoFi model trained on binary
    CIFAR-10 animals-versus-vehicles, we plot the averaged squared
    degree-\(k\) residual random-projection component as a function of
    polynomial order \(k\) for the input representation and the three
    successive LoFi layer representations.}
    \label{fig:lofi-layer-higher-order-decay}
\end{figure}

\endgroup

\newpage

\section{Two-layer networks}\label{app:two-layer}

\subsection{Training a two-layer network on multi-index models}\label{app:two-layer-mim}

In this section, we instantiate the general framework on a canonical setting that has driven much of the recent theory of feature learning: \emph{training a two-layer network on a Gaussian multi-index model} (e.g. ~\cite{goldt2019dynamics,benarous2024stochastic,abbe2021staircase,abbe2023sgd,dandi2024two,bietti2022learning,bietti2023learning, arnaboldi2023high,damian2024computational,damiangenerative,dandi2024benefits}.)

The goal of this subsection is to make explicit how the second-order correlation criterion that drives Neural LoFi in the main text reduces, in this setting, to the classical \emph{information exponent} of ~\cite{benarous2024stochastic}, and how the regime $\mathrm{IE}\le 2$ is precisely the regime in which Neural LoFi alone (i.e., a single low-degree filtering step) suffices for weak recovery of the hidden subspace. This is connected to a recent line of work on the Hessian of the loss in, e.g. \cite{bonnaire2025role,zhang2025neural,defilippis2026optimal,montanari2026phase}. We then describe how, when $\mathrm{IE}=2$, the resulting dynamics of the two-layer network under gradient descent take the form of a \emph{saddle-to-saddle} cascade through the top directions of the population correlation matrix.

\paragraph{Setting}
We consider inputs $\bx\in\mathbb{R}^d$ with $\bx\sim\mathcal{N}(0,I_d)$ and a teacher of \emph{multi-index} form
\begin{equation}\label{eq:mim-teacher}
y \;=\; f^\star(\bx) \;=\; g^\star\!\bigl(\langle \bu^\star_1,\bx\rangle,\dots,\langle \bu^\star_r,\bx\rangle\bigr) + \xi,
\end{equation}
where $U^\star=[\bu^\star_1,\dots,\bu^\star_r]\in\mathbb{R}^{d\times r}$ has orthonormal columns spanning a hidden subspace $V^\star\subset\mathbb{R}^{d}$ of dimension $r=\mathcal{O}(1)$, $g^\star:\mathbb{R}^r\to\mathbb{R}$ is a fixed link function with $\mathbb{E}[g^\star(Z)^2]<\infty$ for $Z\sim\mathcal{N}(0,I_r)$, and $\xi$ is independent sub-Gaussian noise. The student is the two-layer network
\begin{equation}\label{eq:two-layer-student}
\hat f(\bx) \;=\; \frac{1}{\sqrt{p}}\sum_{i=1}^{p} a_i\,\sigma\!\bigl(\langle\bw_i,\bx\rangle\bigr),\qquad \bw_i\in\mathbb{R}^d,\ a_i\in\mathbb{R},
\end{equation}
trained by (online or layerwise) gradient descent on the squared loss with small initialization $\bw_i(0)\sim\mathcal{N}(0,d^{-1}I_d)$ and second-layer weights $a_i$ either fixed at random signs or trained on a fresh batch. This is the standard setup of e.g. ~\cite{goldt2019dynamics,arnaboldi2023high,benarous2024stochastic,bietti2023learning,abbe2023sgd,dandi2024benefits}.

\paragraph{From the LoFi correlation criterion to the information exponent}
Specializing the first-layer ($\ell=1$) Neural LoFi criterion of equation~\eqref{eq:lofi-accessibility} to this setting, the input representation is $\bz_0(\bx)=\bx$, and the population version of the second-order correlation maximized by Neural LoFi at layer $1$ becomes
\begin{equation}\label{eq:lofi-mim}
\rho(\varphi)
\;=\;
\Bigl|\,\mathbb{E}\!\left[y\,\varphi(\bx)^2\right]\Bigr|,
\qquad
\varphi\in\mathcal{H}_0,\ \|\varphi\|_{\mathcal{H}_0}=1,
\end{equation}
where $\mathcal{H}_0=L^2(\gamma_d)$ is the Gaussian RKHS associated with the identity representation reducing to the constraint $\norm{\bw}=1$.
. Expanding $f^\star$ in the Hermite basis, $f^\star(\bx)=\sum_{k\ge 0}\langle f^\star,H_k\rangle\,H_k(U^{\star\top}\bx)$, and any test feature $\varphi$ in the same basis, the criterion~\eqref{eq:lofi-mim} retains only the components of $\varphi^2$ that overlap with the lowest non-zero Hermite component of $f^\star$. Following~\cite{benarous2024stochastic}, define the \emph{information exponent} of $f^\star$ as
\begin{equation}\label{eq:ie-def}
\mathrm{IE}(f^\star)
\;\coloneqq\;
\min\bigl\{\,k\ge 1 \,:\, \mathbb{E}[f^\star(\bx)\,H_k(\langle \bu,\bx\rangle)]\not\equiv 0 \text{ for some } \bu\in V^\star\,\bigr\}.
\end{equation}
Our second-order correlation $\rho(\varphi)$ probes $\varphi^2$, so it is sensitive to the first two Hermite components. In particular,
\begin{itemize}
    \item if $\mathrm{IE}(f^\star)=1$ (linear teacher direction), the maximizer of~\eqref{eq:lofi-mim} corresponds to $\varphi$ linear in $\bx$ and aligned with the leading direction of $\mathbb{E}[y\bx]$;
        \item if $\mathrm{IE}(f^\star)=2$, the maximizer is a quadratic form in $\bx$ and the criterion reduces to the top-eigenvector problem for the population correlation matrix $C^\star\coloneqq \mathbb{E}[y\,\bx\bx^\top]\!-\!\mathbb{E}[y]\,I_d$;
            \item if $\mathrm{IE}(f^\star)\ge 3$, the second-order correlation is degenerate at the population level on the hidden subspace, and a single LoFi step recovers no signal.
            \end{itemize}
            This is precisely the well-known threshold separating ``easy'' from ``hard'' multi-index problems for online SGD~\cite{benarous2024stochastic,bietti2022learning,abbe2021staircase,damian2024computational}.

            \paragraph{Neural LoFi suffices for weak recovery when $\mathrm{IE}\le 2$.}
            Let $\hat C^{(1)}=\frac{1}{n}\sum_{\mu=1}^n y_\mu\,\bx_\mu\bx_\mu^\top - \bar y\,I_d$ be the empirical correlation used by Neural LoFi at layer $1$, and let $\widehat V_1$ be the top-$k_1$ eigenspace of $\hat C^{(1)}$. The following statement, which is a direct consequence of matrix concentration results applied to $\hat{C}^
            {(1)}$, makes the link with weak recovery quantitative.

            \begin{proposition}[Neural LoFi recovers $V^\star$ when $\mathrm{IE}\le 2$]\label{prop:lofi-ie-2}
            Assume the multi-index model~\eqref{eq:mim-teacher} with $\mathrm{IE}(f^\star)\le 2$ and bounded link function. There exist constants $c,C,\delta>0$ depending only on $g^\star$ such that, with sample size $n\ge C\,d$, the top-$r$ eigenspace $\widehat V_1$ of $\hat C^{(1)}$ satisfies
            \begin{equation}
            \Bigl\|\,P_{\widehat V_1} - P_{V^\star}\,\Bigr\|_{\mathrm{op}} \;\le\; 1-\delta,
            \qquad
            \text{i.e.\ weak recovery of }V^\star,
            \end{equation}
            with probability $1-d^{-c}$.
            \end{proposition}

            In other words, when $\mathrm{IE}\le 2$, a single Neural LoFi layer is statistically sufficient: the same low-degree filtering step that, in the main text, defines the Neural LoFi feature already captures the hidden subspace at the optimal $n=\Theta(d)$ sample complexity, in agreement with the upper bounds of~\cite{benarous2024stochastic, abbe2023sgd}. When $\mathrm{IE}\ge 3$, Proposition~\ref{prop:lofi-ie-2} fails by construction, and recovery of $V^\star$ requires either non-Gaussian preprocessing of $y$ (e.g.\ polynomial reweightings, as in~\cite{damian2024computational,damiangenerative}) or genuinely deeper compositionality, which is the regime studied in Appendix~\ref{app:GD-neural} and the main text.

            \paragraph{Saddle-to-saddle dynamics in the Neural LoFi regime with $\mathrm{IE}=2$.}
            We now describe how, under the same multi-index model with $\mathrm{IE}(f^\star)=2$, the actual gradient-descent dynamics of the two-layer network~\eqref{eq:two-layer-student} traverse the top directions of $C^\star$ in a saddle-to-saddle cascade, recovering the same ordered features as Neural LoFi.

            Order the eigenvalues of $C^\star$ as $\lambda_1>\lambda_2>\cdots>\lambda_r>0$ on the hidden subspace and $\lambda_{r+1}=\cdots=\lambda_d=0$ off it, with associated eigenvectors $\bv^\star_1,\dots,\bv^\star_r$. With small initialization $\|\bw_i(0)\|\asymp d^{-1/2}$ and step-size $\eta$, the leading-order continuous-time dynamics of the (rescaled) neuron weights $\widetilde\bw_i(t)$ satisfy, after averaging over the second-layer signs, a power-iteration--type ODE driven by $C^\star$:
            \begin{equation}\label{eq:two-layer-pi}
            \dot{\widetilde\bw}_i(t)
            \;=\;
            \bigl(C^\star - \widetilde\bw_i^\top C^\star \widetilde\bw_i\,I_d\bigr)\,\widetilde\bw_i(t)\,+\,o(1),
            \end{equation}
            Equation~\eqref{eq:two-layer-pi} is the gradient flow of $-\widetilde\bw^\top C^\star \widetilde\bw$ on the sphere, whose only attractors are the top eigenvectors $\pm\bv^\star_1$ and whose other critical points are strict saddles indexed by the remaining $\bv^\star_j$.

            When several eigenvalues $\lambda_j$ are well separated, this gives rise to the \emph{saddle-to-saddle} picture: starting from a small isotropic initialization, the trajectory spends a time $T_j\asymp \lambda_j^{-1}\log(d)$ in a neighborhood of the saddle associated with $\bv^\star_j$ before escaping along the next leading direction \cite{arous2026learning}. The neurons therefore learn the directions $\bv^\star_1,\bv^\star_2,\dots,\bv^\star_r$ \emph{sequentially}, in order of decreasing population correlation $\lambda_j$, with sharp transitions between successive plateaus.

            The eigenbasis traversed by GD in such saddle-to-saddle dynamics is exactly the basis selected by the Neural LoFi criterion~\eqref{eq:lofi-mim}. In both cases, the relevant operator is $C^\star=\mathbb{E}[y\,\bx\bx^\top]$ (up to centering), and the ordering by $\lambda_j$ is the ordering of low-degree correlations with the label. Neural LoFi can thus be viewed as the \emph{static abstraction} of the saddle-to-saddle GD dynamics in the $\mathrm{IE}=2$ regime, replacing the dynamics through saddles by a single eigendecomposition of $\hat C^{(1)}$.

Finally, we note that with the use of a different loss other than the squared loss or with data reuse or label transformations~\cite{damian2024computational,damiangenerative}, the criterion in Equation \ref{eq:lofi-mim} is modified to:

\begin{equation}\label{eq:lofi-mim-gen}
\rho_g(\varphi)
\;=\;
\Bigl|\,\mathbb{E}\!\left[g(y)\,\varphi(\bx)^2\right]\Bigr|,
\qquad
\varphi\in\mathcal{H}_0,\ \|\varphi\|_{\mathcal{H}_0}=1,
\end{equation}

for a transformation $g:\mathbb{R}\rightarrow\mathbb{R}$. The condition $\rho_g(\varphi)\neq 0$ then corresponds to generative exponent~\cite{damiangenerative,arous2021online} $\geq 2$ instead of the information exponent~\cite{benarous2024stochastic}.

\newpage 
\section{A Solvable Theoretical Setting}
\label{app:th_rf}

The agnostic and recursive nature of Neural LoFi calls for a theoretical setting that contains a compositional structure, while not revealing the relevant intermediate variables to the learner. In this appendix, we follow the hierarchical spectral construction of \cite{tabanelli2026deep}, building on the fundamental  earlier works \cite{wang2023learning,nichani2023provable,fu2025learning}. 
This allows us to define a controlled high-dimensional model and to use it to study how depth turns a globally hard learning problem into a sequence of simpler spectral recoveries.

A natural way to isolate the role of depth is to consider teacher-student models where the target is not merely a low-dimensional function of the input, but is built through a hierarchy of intermediate representations. Such models have appeared in several recent works on compositional learning and the computational advantage of depth, including random hierarchy models, hierarchical Gaussian targets, and polynomial teacher-student constructions (e.g.
\cite{cagnetta2024deep,garnierbrun2025transformerslearnstructureddata,dandicomputational, wang2023learning, nichani2023provable, fu2025learning,cagnetta2025learning,tabanelli2026deep,cagnetta2026deriving,ren2026provable}). 
They share the same guiding principle: a target may look high-dimensional or high-degree as a function of the input, while becoming low-degree after the right intermediate representation has been found. This is precisely the situation in which depth should help, since learning can proceed by a sequence of simpler feature-recovery problems rather than by solving the full high-degree task at once.

\subsection{Setting}\label{subsec:set_th}

We focus on the high-dimensional Gaussian teacher-student model of \cite{tabanelli2026deep}, which provides a particularly tractable instance of this principle. The input is Gaussian, $\bx\sim\mathcal N(0,I_d)$, and the target is generated by a two-level compositional hierarchy
\begin{align}
    \bx \in \mathbb R^d
    \quad \longrightarrow \quad
    h^{(1)}(\bx) \in \mathbb R^{d_1}
    \quad \longrightarrow \quad
    h^{(2)}(\bx) \in \mathbb R
    \quad \longrightarrow \quad
    y .
\end{align}
For $q\geq 1$, we denote by $H_q(\cdot)$ the normalized Hermite polynomial of order $q$, either in its scalar or tensor-valued form.\footnote{
For a vector $\bx\in\mathbb R^m$, the tensor Hermite polynomial $H_q(\bx)\in(\mathbb R^m)_{\mathrm{sym}}^{\otimes q}$ is defined through the tensorial Rodrigues formula
$
\sqrt{q!}\,H_q(\bx)
=
(-1)^q e^{\|\bx\|^2/2}
\nabla_{\bx}^{\otimes q}
\left(e^{-\|\bx\|^2/2}\right),
$
where $\nabla_{\bx}^{\otimes q}$ denotes the $q$-fold symmetric tensor of derivatives. 
For $m=1$, this reduces to the normalized scalar Hermite polynomial
$
\sqrt{q!}\,H_q(z)
=
(-1)^q e^{z^2/2}
\frac{d^q}{dz^q}
\left(e^{-z^2/2}\right).
$
}
We write $\langle\cdot,\cdot\rangle$ for the Frobenius product between tensors of the same order.\footnote{
For example, if $\bx\in\mathbb R^m$ and $A\in\mathbb R^{m\times m}$ is symmetric, then
$
\left\langle A,H_2(\bx)\right\rangle
=
\frac{1}{\sqrt{2}}
\left(
    \bx^\top A\bx-\mathrm{Tr}(A)
\right).
$
}
We write $F_q$ for the symmetric flattening map from order-$q$ symmetric tensors to $\mathbb R^{D_q}$, with $D_q=\binom{d+q-1}{q}$, chosen so that the Frobenius product is preserved. 
Below, we freely identify a symmetric tensor with its flattened representation whenever no ambiguity arises. The teacher parameters are given by symmetric tensors
\begin{align}
    A_i^{(1)}
    \in
    (\mathbb R^d)_{\mathrm{sym}}^{\otimes q},
    \qquad
    i=1,\ldots,d_1,
    \qquad
    d_1=d^\epsilon,
\end{align}
normalized so that the first-layer features have order-one variance, and by a symmetric matrix
\begin{align}
    A^{(2)}
    \in
    \mathbb R^{d_1\times d_1}.
\end{align}
The latent variables and the label are then defined as
\begin{align}
    h^{(1)}_i(\bx)
    &=
    \left\langle
        A_i^{(1)}, H_q(\bx)
    \right\rangle,
    \qquad i=1,\ldots,d_1, \\
    h^{(2)}(\bx)
    &=
    \left\langle
        A^{(2)}, H_2(h^{(1)}(\bx))
    \right\rangle, \\
    y
    &=
    g^\star(h^{(2)}(\bx)).
\end{align}
Thus the first layer selects only $d_1=d^\epsilon$ directions inside the ambient degree-$q$ Hermite space of dimension $D_q$. When $d_1\ll D_q$, the relevant information is sparse in this low-degree feature space: the target depends on $\bx$ only through a small hidden subspace of Hermite features. Learning the first layer therefore amounts to recovering this subspace, so that the rest of the hierarchy can be expressed as a low-degree problem in the variables $h^{(1)}$.

The analysis of \cite{tabanelli2026deep} shows that this sparse compositional structure can be exploited by a hierarchical spectral estimator. Rather than learning the full composed function in one step, the procedure first recovers the hidden degree-$q$ subspace defining $h^{(1)}$, and then uses this recovered representation to make the next component of the hierarchy accessible. In this sense, depth turns the learning problem into a sequence of spectral recovery tasks, each one exposing the variables needed by the next layer.

This gives a clean explanation for the advantage of depth in this model. A one-shot method that works directly on the input must resolve the full high-degree dependence of $y$ on $x$. By contrast, the hierarchical spectral procedure only needs to reveal the first representation and then reuse it to make the next layer visible. In the regime $d_1=d^\epsilon$, the first stage requires on the order of $d^{q+\epsilon}$ samples, while the second stage requires only the sample complexity of a quadratic problem in dimension $d_1$. The dominant cost is therefore the recovery of the first hidden representation, rather than the degree of the full composed polynomial.

The price to pay is that the corresponding spectral estimators are still partially co-designed with the Gaussian-Hermite structure of the teacher. Indeed, the first stage is built in an explicit degree-$k$ Hermite feature space, while the second stage uses a prescribed second-order Hermite structure in the recovered variables. A natural way to relax this feature-design aspect is to replace the explicit Hermite construction by nonlinear random features. Rather than specifying the polynomial coordinates in advance, one lets a random feature map generate a generic nonlinear representation and applies the same layer-wise spectral selection in that space. In the hierarchical teacher-student setting above, this random-feature extension of the hierarchical spectral estimator matches precisely the Neural LoFi algorithm.

\subsection{Random-feature hierarchical estimator}
\label{app:rf-estimator}

In this setting, Neural LoFi takes the following concrete form. Let $W_1=(w_{1,a})_{a=1}^{p_1}$ with $w_{1,a}\sim \mathrm{Unif}(\mathbb S^{d-1})$, and define the first random-feature representation
\begin{align}
    \phi^{(1)}_\mu
    =
    \frac{1}{\sqrt{p_1}}
    \sigma_1(W_1 x_\mu)
    \in \mathbb R^{p_1}.
\end{align}
The first spectral operator is
\begin{align}
    \widehat C_1
    =
    \frac1n
    \sum_{\mu=1}^n
    y_\mu
    \phi^{(1)}_\mu
    \phi^{(1)\top}_\mu .
\end{align}
Let $\widehat V_1\in\mathbb R^{p_1\times d_1}$ contain the eigenvectors of $\widehat C_1$ associated with the largest eigenvalues in absolute value. Here we keep $d_1$ directions for simplicity; this can be replaced by a standard rank-selection step. The recovered first-layer coordinates are
\begin{align}
    \widehat h^{(1)}_\mu
    =
    \widehat V_1^\top \phi^{(1)}_\mu
    \in \mathbb R^{d_1}.
\end{align}
We then draw $W_2=(w_{2,a})_{a=1}^{p_2}$ with $w_{2,a}\sim \mathrm{Unif}(\mathbb S^{d_1-1})$, and define
\begin{align}
    \phi^{(2)}_\mu
    =
    \frac{1}{\sqrt{p_2}}
    \sigma_2(W_2 \widehat h^{(1)}_\mu)
    \in \mathbb R^{p_2}.
\end{align}
At the second layer, since the teacher contains only one hidden direction, we do not construct a matrix-valued spectral estimator. Instead, we directly form the first-order moment vector
\begin{align}
    \widehat v_2
    =
    \frac1n
    \sum_{\mu=1}^n
    y_\mu
    \phi^{(2)}_\mu .
\end{align}
The associated second-layer coordinate is then obtained by projection:
\begin{align}
    \widehat h^{(2)}_\mu
    =
    \widehat v_2^\top \phi^{(2)}_\mu .
\end{align}
Finally, the readout is fitted on the one-dimensional representation $\{(\widehat h^{(2)}_\mu,y_\mu)\}_{\mu=1}^n$, for instance by ridge regression in a polynomial feature space,
\begin{align}
    \widehat g
    \in
    \arg\min_{g\in\mathcal G}
    \frac1n
    \sum_{\mu=1}^n
    \left(
        y_\mu - g(\widehat h^{(2)}_\mu)
    \right)^2
    +
    \lambda \|g\|_{\mathcal G}^2 .
\end{align}

The choice of keeping exactly \(d_1\) eigenvectors at the first layer is not meant to define a different procedure, but simply to make explicit the outcome of the usual rank-selection step within Neural LoFi in the present setting. More generally, one could keep an arbitrary number \(k\) of directions and optimize over \(k\), exactly as in the rest of the pipeline. In the model considered here, this optimization would select \(d_1\) directions at the first layer and a single direction at the second layer. In that sense, fixing \(d_1\) in the first layer and using the linear estimator \(\widehat v_2\) in the second layer is fully equivalent to the standard Neural LoFi selection rule.

\subsection{Main Results}\label{subsec:th_main_results}

The main prediction of this tractable model is that replacing the explicit Hermite structure of the estimator by random features preserves the emergence transition. In particular, for a degree-$q$ first layer with $d_1=d^\epsilon$ hidden variables, we expect the first representation $h^{(1)}$ to become recoverable at the sample scale
\begin{align}
    n
    \gg
    d^{q+\epsilon}.
\end{align}
In the quadratic setting used in our experiments, $q=2$, and the predicted first-stage transition is therefore $n \gg d^{2+\epsilon}$.

This prediction is the concrete instance of the emergence criterion in Eq.~\eqref{eq:emergence_feature} in the main text. Indeed, Eq.~\eqref{eq:emergence_feature} states that the empirical fluctuation level is governed by the effective dimension of the current residual feature class. Section~\ref{app:effective-dimension-hermite} shows that, before the first hidden variables have been recovered, the relevant effective dimension is the size of the degree-$q$ Hermite block, namely $D_q=O(d^q)$. The additional factor $d_1=d^\epsilon$ comes from separating and aligning with the $d_1$ planted directions in this block, giving the scale $D_qd_1=O(d^{q+\epsilon})$.

The role of the next subsection is to explain why the random-feature estimator contains the same signal-bearing Hermite spectral object as the explicit construction, now embedded in the random-feature representation. Figure~\ref{fig:rf_theoretical} illustrates the resulting transition numerically: in the quadratic case, the drop in MSE, the growth of the overlap with $h^{(1)}$, and the separation of the leading eigenvalues all occur at the predicted first-layer emergence scale.

\subsection{Mathematical Justification}\label{gif:rf_maths}

We now analyze the signal structure of the random-feature estimator introduced above. The key point is that, although the estimator is built from generic nonlinear random features, its population signal behaves as a well-defined Hermite estimator aligned with the teacher hierarchy. This will allow us to connect the first step of the Neural LoFi algorithm to the spectral transition established in the Hermite model of \cite{tabanelli2026deep}.

Let us expand the first-layer activation in Hermite as
\begin{align}
    \sigma_1(z)
    =
    \sum_{q\geq 2}
    c_q H_q(z),
    \qquad
    c_q
    =
    \mathbb E_{z\sim \mathcal N(0,1)}
    \left[
        \sigma_1(z)H_q(z)
    \right],
\end{align}
where $H_q$ denotes the normalized scalar Hermite polynomial of degree $q$.
For each row $\bw_{1,a}$ of $W_1$, we use the standard Hermite tensor identity
\begin{align}
    H_q(\langle \bw_{1,a},\bx\rangle)
    =
    \left\langle
        \bw_{1,a}^{\otimes q},
        H_q(\bx)
    \right\rangle .
\end{align}
Here and below, $\bw_{1,a}^{\otimes q}$ denotes the degree-$q$ Hermite coefficient vector, with the multi-index normalization chosen so that the above identity holds. Collecting these coefficient tensors over all rows of $W_1$, and flattening them with the Frobenius-preserving map $F[\cdot]$, we define:
\begin{align}
    \frac{1}{\sqrt{p_1}} H_q(W_1\bx)
    =
    P_q H_q(\bx),\quad \textrm{with} \quad P_q
    :=
    \frac1{\sqrt{p_1}}
    \begin{bmatrix}
        F[\bw_{1,1}^{\otimes q}]^\top \\
        \vdots \\
        F[\bw_{1,p_1}^{\otimes q}]^\top
    \end{bmatrix}
    \in \mathbb R^{p_1\times D_q}.
\end{align}
where $H_q(W_1\bx)\in\R^{p_1}$ is understood entrywise on the left-hand side and $H_q(\bx)\in \R^{D_q}$ is the flattened Hermite tensor. Hence the first random-feature representation decomposes as
\begin{align}
    \phi_\mu^{(1)}
    =
    \frac{1}{\sqrt{p_1}}\sigma_1(W_1\bx_\mu)
    =
    \sum_{q\geq 2}
    c_q
    P_q H_q(\bx_\mu).
\end{align}

We denote by
\begin{align}
    \widehat C_H^{(q)}
    :=
    \frac1n
    \sum_{\mu=1}^n
    y_\mu
    H_q(\bx_\mu)H_q(\bx_\mu)^\top
\end{align}
the degree-$q$ Hermite moment matrix appearing in the middle. If one uses the centered second-order Hermite convention of \cite{tabanelli2026deep}, this matrix should be replaced by its centered version; the difference is an empirical isotropic term, controlled by the same estimates and absorbed in the rates below. With this notation, the degree-$q$ contribution to
the first Neural LoFi operator is
\begin{align}
    \widehat C_{\rm RF}^{(q)}
    =
    c_q^2
    P_q
    \widehat C_H^{(q)}
    P_q^\top .
\end{align}
Thus the random-feature operator contains the Hermite estimator conjugated by the random
Hermite embedding $P_q$, up to the scalar factor $c_q^2$.

We now focus on the quadratic case $q=2$, which is the case controlled explicitly by the quadratic Hermite decomposition of \cite{wen2025when}. Define the normalized RF Gram
\begin{align}
    G_2
    :=
    D_2 P_2^\top P_2 .
\end{align}
The important point is that $G_2$ is not close to the identity on the full degree-$2$ Hermite space. The trace direction produces a deterministic contraction spike. More precisely, by Lemma F.6 and Corollary F.7 of \cite{wen2025when}, one has the decomposition
\begin{align}
    G_2
    =
    I_{D_2}
    +
    K_2
    +
    R_2,
    \qquad
    K_2=\theta_d ee^\top,
    \qquad
    |\theta_d|\le \widetilde O(d),
\end{align}
where $e\in\mathbb R^{D_2}$ is the normalized trace direction in the degree-$2$ Hermite block. The term $K_2$ is the explicit non-trivial contraction term coming from the trace component of the quadratic Hermite features, while $R_2$ is the remaining centered fluctuation of the random-feature Gram, with
\begin{align}
    \|R_2\|_{\op}
    \le
    \widetilde O\!\left(
        \sqrt{\frac{D_2}{p_1}}
        +
        \frac{D_2}{p_1}
    \right),
\end{align}
where the $\widetilde O(\cdot)$ hides logarithmic factors in $d$. This conservative full-column concentration bound is not expected to be optimal, but it is sufficient for the projected comparison below.

Finally, let $A^{(1)}\in\mathbb R^{d_1\times D_2}$ be the matrix of planted first-layer Hermite directions, and define its random-feature image by
\begin{align}
    A_{\RF}^{(1)}
    :=
    \sqrt{D_2}\,
    A^{(1)}
    P_2^\top
    \in\mathbb R^{d_1\times p_1}.
\end{align}
We can now state the projected RF analogue of Theorem~3.1 in
\cite{tabanelli2026deep}.
\medskip

\begin{theorem}[Projected RF Hermite recovery, quadratic case]
\label{thm:rf-hermite-bridge-k2}
Assume the setting and normalization of Theorem~3.1 in
\cite{tabanelli2026deep}, with $d_1=d^\varepsilon$ and
$\varepsilon<1/2$. With $G_2$, $K_2$, $R_2$ as defined above, the following holds with
high probability:
\begin{align}
&\sqrt{d_1}
\left\|
    \frac{D_2}{c_2^2}
    A_{\RF}^{(1)}
    \widehat C_{\rm RF}^{(2)}
    A_{\RF}^{(1)\top}
    -
    \nu_1A^{(2)}
\right\|_{\op}
\nonumber\\
&\qquad\le
\widetilde O\!\left(
    \sqrt{\frac{d^2d_1}{n}}
\right)
+
\widetilde O\!\left(
    \frac{d_1}{\sqrt d}
\right)
+
\widetilde O\!\left(
    \frac1{\sqrt{d_1}}
\right)
\nonumber\\
&\qquad\quad+
\widetilde O\!\left(
    d_1\sqrt{\frac{d_1}{n}}
    +
    \frac{d_1^2}{d^2}
    +
    \sqrt{d_1} \left(
        \sqrt{\frac{D_2}{p_1}}
        +
        \frac{D_2}{p_1}
    \right)\left(1+
        \sqrt{\frac{D_2}{p_1}}
        +
        \frac{D_2}{p_1}
    \right)
\right).
\label{eq:rf-hermite-final-rate}
\end{align}
In particular, if
\[
    n\gg d^2d_1,
    \qquad
    p_1 \gg \sqrt{d_1}d^2,
\]
then the additional RF bridge error is $o(1)$. Therefore the projected RF
estimator has the same limiting signal as the Hermite estimator, up to
the deterministic scalar factor $c_2^2/D_2$.
\end{theorem}

\begin{proof}
The proof is a direct comparison with the Hermite estimator. By the
definitions of $A_{\RF}^{(1)}$, $\widehat C_{\rm RF}^{(2)}$, and $G_2$, we
have the exact identity
\begin{align}
    \frac{D_2}{c_2^2}
    A_{\RF}^{(1)}
    \widehat C_{\rm RF}^{(2)}
    A_{\RF}^{(1)\top}
    =
    A^{(1)}G_2\widehat C_H^{(2)}G_2A^{(1)\top}.
\end{align}
Thus it is enough to compare
$A^{(1)}G_2\widehat C_H^{(2)}G_2A^{(1)\top}$ with
$A^{(1)}\widehat C_H^{(2)}A^{(1)\top}$.

We use the RF Gram decomposition stated above,
\[
    G_2=I_{D_2}+K_2+R_2,
    \quad
    K_2=\theta_d ee^\top,
    \quad
    |\theta_d|\le \widetilde O(d),
    \quad
    \|R_2\|_{\op} \le
    \widetilde O\!\left(
        \sqrt{\frac{D_2}{p_1}}
        +
        \frac{D_2}{p_1}
    \right).
\]
Here $K_2$ is the trace-contraction spike and $R_2$ is the centered RF
Gram remainder.

Let $C:=\widehat C_H^{(2)}$. We use the following standard high-probability
bounds:
\begin{align}
    \|A^{(1)}e\|_2
    &\le
    \widetilde O\!\left(\frac{\sqrt{d_1}}{d}\right),
    \label{eq:Ae-bound}
    \\
    \|C\|_{\op}
    &\le
    \widetilde O\!\left(\frac1{\sqrt{d_1}}\right),
    \label{eq:C-op-bound}
    \\
    \|e^\top C A^{(1)\top}\|_2
    &\le
    \widetilde O\!\left(
        \sqrt{\frac{d_1}{n}}+\frac1d
    \right),
    \label{eq:eCA-bound}
    \\
    |e^\top C e|
    &\le
    \widetilde O\!\left(
        \frac1{\sqrt n}
        +
        \frac{\sqrt{d_1}}{d^2}
    \right).
    \label{eq:eCe-bound}
\end{align}
The first estimate is the overlap of a random $d_1$-dimensional Gaussian subspace with the fixed trace direction. The remaining estimates are obtained by the same truncation, matrix Bernstein, and Gaussian-equivalence arguments used in Appendix A.4 of \cite{tabanelli2026deep}, applied either with one test tensor equal to the trace direction $e$, or with both test tensors equal to $e$. The population contributions are smaller because the trace overlap satisfies $\|A^{(1)}e\|_2=\widetilde O(\sqrt{d_1}/d)$.

Expanding $G_2=I+K_2+R_2$ gives
\[
    G_2CG_2-C
    =
    K_2C+CK_2+K_2CK_2
    +
    \text{terms containing }R_2.
\]
For the first trace term,
\[
    A^{(1)}K_2CA^{(1)\top}
    =
    \theta_d
    (A^{(1)}e)(e^\top C A^{(1)\top}),
\]
and therefore, using
\eqref{eq:Ae-bound}--\eqref{eq:eCA-bound},
\[
    \sqrt{d_1}
    \|A^{(1)}K_2CA^{(1)\top}\|_{\op}
    \le
    \widetilde O\!\left(
        d_1\sqrt{\frac{d_1}{n}}
        +
        \frac{d_1}{d}
    \right).
\]
The transpose term $A^{(1)}CK_2A^{(1)\top}$ is bounded identically. For
the quadratic trace term,
\[
    A^{(1)}K_2CK_2A^{(1)\top}
    =
    \theta_d^2
    (A^{(1)}e)(e^\top C e)(A^{(1)}e)^\top,
\]
so by \eqref{eq:Ae-bound} and \eqref{eq:eCe-bound},
\[
    \sqrt{d_1}
    \|A^{(1)}K_2CK_2A^{(1)\top}\|_{\op}
    \le
    \widetilde O\!\left(
        d_1\sqrt{\frac{d_1}{n}}
        +
        \frac{d_1^2}{d^2}
    \right).
\]
Finally, all terms containing $R_2$ are controlled by
$\|R_2\|_{\op}$ and
\eqref{eq:C-op-bound}. The worst mixed terms are $K_2CR_2$ and
$R_2CK_2$, and they satisfy
\[
    \sqrt{d_1}
    \|A^{(1)}K_2CR_2A^{(1)\top}\|_{\op}
    \le
    \widetilde O\!\left(
        \sqrt{d_1}\left(
        \sqrt{\frac{D_2}{p_1}}
        +
        \frac{D_2}{p_1}
    \right)
    \right),
\]
with the same bound for the transpose. The $R_2CR_2$ term contributes at
most
\[
    \widetilde O\!\left(
        \sqrt{d_1}\left(
        \sqrt{\frac{D_2}{p_1}}
        +
        \frac{D_2}{p_1}
    \right)^2
    \right).
\]
Combining these estimates gives the bridge bound
\begin{align}
\sqrt{d_1}
\left\|
    A^{(1)}
    \left(
        G_2CG_2-C
    \right)
    A^{(1)\top}
\right\|_{\op}&
\le
\widetilde O\!\left(
    d_1\sqrt{\frac{d_1}{n}}
    +
    \frac{d_1}{d}
    +
    \frac{d_1^2}{d^2}
\right) \nonumber\\
& + \widetilde O\!\left( \sqrt{d_1}\left(\sqrt{\frac{D_2}{p_1}}
        +
        \frac{D_2}{p_1}
    \right)\left(1+\sqrt{\frac{D_2}{p_1}}
        +
        \frac{D_2}{p_1}\right)
        \right).
\label{eq:rf-bridge-bound}
\end{align}
Theorem~3.1 of \cite{tabanelli2026deep} gives
\[
\sqrt{d_1}
\left\|
    A^{(1)}\widehat C_H^{(2)}A^{(1)\top}
    -
    \nu_1A^{(2)}
\right\|_{\op}
\le
\widetilde O\!\left(
    \sqrt{\frac{d^2d_1}{n}}
\right)
+
\widetilde O\!\left(
    \frac{d_1}{\sqrt d}
\right)
+
\widetilde O\!\left(
    \frac1{\sqrt{d_1}}
\right).
\]
Combining this with \eqref{eq:rf-bridge-bound} proves
\eqref{eq:rf-hermite-final-rate}.
\end{proof}

We stated the theorem for the quadratic case $q=2$ because the contraction
correction is completely explicit. For any fixed degree $q$, the same argument
applies by replacing the trace direction by the finite list of lower
contraction sectors given by the Gegenbauer decomposition; see Lemma F.8 of
\cite{wen2025when}. These contraction terms are subleading on the random
planted subspace by the same vanishing-contraction estimates used in Appendix
A.1--A.4 of \cite{tabanelli2026deep}. Consequently, the Hermite recovery
theorem transfers to the degree-$q$ RF estimator with the same sample scaling.
Together with the RF Gram concentration requirement, the sufficient scaling for
recovery is
\begin{align}
    n\gg D_q d_1,
    \qquad
    p_1\gg D_q,
\end{align}
up to logarithmic factors. Under these conditions, the degree-$q$ RF estimator
has the same projected signal limit as the Hermite estimator, up to the scalar
factor $c_q^2/D_q$.

The projected comparison above shows that the first LoFi step
has the same signal scaling as the explicit Hermite estimator. Therefore, the
bottleneck remains the Hermite recovery scale of the first layer, together with the RF width needed to realize the corresponding degree-$q$ block. The same reasoning applies recursively to later LoFi steps, with the ambient dimension replaced by the dimension of the representation entering that layer.

\subsection{Numerical experiments}
\label{app:rf-numerics}

We begin by defining the two observables reported in the random-feature experiments. 
Given a fresh test set \(\{(\bx_\mu,y_\mu)\}_{\mu=1}^{n_{\mathrm{test}}}\), we measure the prediction error through the test mean-squared error
\begin{align}
    \mathrm{MSE}_{\mathrm{test}}
    =
    \frac{1}{n_{\mathrm{test}}}
    \sum_{\mu=1}^{n_{\mathrm{test}}}
    \bigl(\hat y_\mu-y_\mu\bigr)^2,
\end{align}
and we quantify the recovery of the first hidden layer through the feature overlap
\begin{align}
    \mathrm{Overlap}(h^{(1)},\widehat h^{(1)})
    =
    \bigl\|
    H^{(1)} \hat H^{(1)\top}
    \bigr\|_F^2,
    \quad
    H^{(1)} = \bigl(h^{(1)}_\mu\bigr)_{\mu \le n_{\mathrm{test}}},
    \quad
    \hat H^{(1)} = \bigl(\hat h^{(1)}_\mu\bigr)_{\mu \le n_{\mathrm{test}}}.
\end{align}
All curves are averaged over \(10\) seeds, and error bars indicate one standard deviation.
Concerning the precise setting of the experiments, we specialize to the quadratic first-layer setting \(q=2\). Unless otherwise stated, we fix
\begin{align}
    d_1=\lfloor d^\epsilon\rfloor,
    \qquad
    \epsilon=\frac12,
    \qquad
    n=\lfloor d^\alpha\rfloor,
\end{align}
and generate labels according to
\begin{align}
    y = g^\star(h^{(2)}(\bx)),
    \qquad
    g^\star(t)=\tanh(t).
\end{align}
The purpose of these experiments is to test whether the random-feature pipeline exhibits the same first-layer emergence transition as the explicit Hermite spectral estimator, while replacing the structured Hermite features by generic nonlinear random features.

The random-feature maps used by the learner are chosen independently of the teacher. The rows of the random matrices are sampled uniformly on the sphere,
\begin{align}
    \bw_{1,a}\sim \mathrm{Unif}(\mathbb S^{d-1}),
    \qquad
    \bw_{2,a}\sim \mathrm{Unif}(\mathbb S^{d_1-1}),
\end{align}
as in the estimator of Section~\ref{app:rf-estimator}. For both layers, we use a ReLU activation with its degree-zero and degree-one Hermite components removed:
\begin{align}
    \sigma_{\perp\{0,1\}}(z)
    =
    \mathrm{ReLU}(z)
    -
    c_0 H_0(z)
    -
    c_1 H_1(z),
    \qquad
    c_r
    =
    \mathbb E_{G\sim\mathcal N(0,1)}
    \left[
        \mathrm{ReLU}(G)H_r(G)
    \right].
\end{align}
The random-feature widths are chosen in the overcomplete regime relative to the quadratic Hermite block,
\begin{align}
    p_1 \gtrsim D_2,
    \qquad
    D_2=\binom{d+1}{2},
\end{align}
in agreement with the finite-width requirement appearing in Theorem~\ref{thm:rf-hermite-bridge-k2}. In the second layer, the width $p_2$ is chosen large enough with an equivalent regime as for the first layer, $p_2\geq d_1^2$, with the same activation $\sigma_{\perp\{0,1\}}$ at both layers.

For each value of $d$ and $\alpha$, we generate a training set of size $n=\lfloor d^\alpha\rfloor$ and an independent test set. We follow the Neural LoFi learning strategy detailed in Sec.~\ref{app:rf-estimator}, adding only batch normalization to smooth the anisotropy. The readout $\widehat g$ is fitted by ridge regression on degree-$5$ polynomial features of $\widehat h^{(2)}$, with regularization parameter obtained by cross-validation. We additionally apply the same output calibration procedure throughout all experiments.

Figure~\ref{fig:rf_theoretical} shows that the random-feature estimator undergoes a clear transition near the predicted first-layer scale. For small $\alpha$, the test MSE remains close to its baseline value and the overlap with $h^{(1)}$ is essentially zero, indicating that the first hidden representation is not yet spectrally accessible. Around the predicted scale $n\gg d^{2+\epsilon}$, the MSE drops and the representation overlap grows sharply. This confirms that the improvement in prediction is tied to the recovery of the first hidden representation, rather than merely to the final one-dimensional regression step.

\begin{figure}[t]
    \centering
    \includegraphics[width=\linewidth]{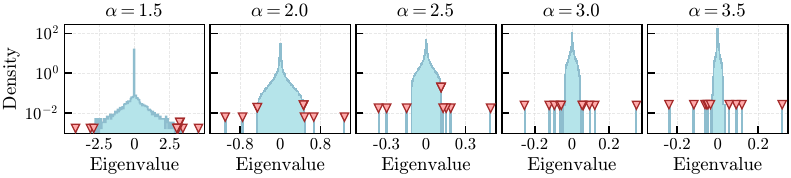}
    \caption{
    \textbf{Spectral emergence in the Neural LoFi estimator.}
    Spectrum of the first random-feature spectral operator $\widehat C_1$ for the hierarchical  solvable model of section \ref{sec:the model}, shown at increasing sample exponents $\alpha=\log(n)/\log(d)$. 
    Blue histograms display the bulk eigenvalue density, while red triangles indicate the leading $d_1$ eigenvalues in absolute value. 
    As $\alpha$ increases, the leading eigenvalues progressively separate from the bulk, marking the {\it emergence} of concepts in the first layer, as predicted by the sample-complexity scale $n\gg d^{2+\epsilon}$ (from Eq.~\eqref{eq:lofi-emergence-threshold}) in the quadratic setting $k=2$.
    }
    \label{fig:rf_spectrum}
\end{figure}

Figure~\ref{fig:rf_spectrum} gives a more direct spectral view of the same phenomenon. At small sample sizes, the leading eigenvalues of $\widehat C_1$ remain buried in the random-feature bulk. As $\alpha$ increases, the leading $d_1$ eigenvalues separate from the bulk and become stable outliers. This outlier formation is the spectral signature of the first-layer signal emerging in the random-feature representation. Together with the overlap curve, it shows that Neural LoFi recovers the hidden degree-$2$ representation at the same sample scale predicted by the Hermite theory, while avoiding the explicit construction of the Hermite feature map.

\newpage 
\section{Neural LoFi Kernel}
\label{app:kernel-nlofi}

\subsection{Preliminaries: RKHS, spectral theorem and kernel integral operators}\label{app:rkhs-prelims}

This subsection collects the standard background on Reproducing Kernel Hilbert Spaces (RKHS) and the spectral decomposition of kernel integral operators that we use in the rest of the appendix.
We refer to \cite{scholkopf2002learning,steinwart2008support} for a more detailed exposition. Throughout, $(\mathcal{X},\mu)$ denotes a measurable input space equipped with a finite Borel measure $\mu$ (e.g.\ the data distribution), and $K:\mathcal{X}\times\mathcal{X}\to\mathbb{R}$ a symmetric kernel.

\paragraph{Reproducing Kernel Hilbert Space}
A Hilbert space $\mathcal{H}$ of real-valued functions on $\mathcal{X}$ is a \emph{Reproducing Kernel Hilbert Space} (RKHS) if, for every $\bx\in\mathcal{X}$, the evaluation functional $\mathrm{ev}_\bx:f\mapsto f(\bx)$ is bounded (continuous) on $\mathcal{H}$. By the Riesz representation theorem, there exists a unique element $K_\bx\in\mathcal{H}$ such that
\begin{equation}\label{eq:reproducing}
  f(\bx)=\langle f,K_\bx\rangle_{\mathcal{H}}\qquad\text{for all }f\in\mathcal{H}.
  \end{equation}
  Defining $K(\bx,\bx'):=\langle K_\bx,K_{\bx'}\rangle_{\mathcal{H}}=K_{\bx'}(\bx)$ yields a symmetric positive semi-definite kernel, and \eqref{eq:reproducing} is called the \emph{reproducing property} because the kernel section $K_\bx=K(\bx,\cdot)$ literally reproduces the value of $f$ at $\bx$ via an inner product. Conversely, by the Moore--Aronszajn theorem~\cite{aronszajn1950theory}, every symmetric positive semi-definite kernel $K$ uniquely determines an RKHS $\mathcal{H}_K$ in which \eqref{eq:reproducing} holds; concretely, $\mathcal{H}_K$ is obtained by completing $\mathrm{span}\{K(\bx,\cdot):\bx\in\mathcal{X}\}$ under the inner product $\langle K(\bx,\cdot),K(\bx',\cdot)\rangle_{\mathcal{H}_K}=K(\bx,\bx')$.

  \paragraph{Spectral theorem for compact self-adjoint operators}
  Let $\mathcal{H}$ be a separable Hilbert space and $T:\mathcal{H}\to\mathcal{H}$ a bounded linear operator. Recall that $T$ is \emph{self-adjoint} if $\langle Tf,g\rangle=\langle f,Tg\rangle$, and \emph{compact} if it maps bounded sets to relatively compact sets (equivalently, it is a norm-limit of finite-rank operators). The spectral theorem states:
  \begin{theorem}[Hilbert--Schmidt spectral theorem]\label{thm:spectral}
  If $T$ is compact and self-adjoint on a separable Hilbert space $\mathcal{H}$, then there exist a (finite or countable) sequence of real eigenvalues $\{\lambda_i\}_{i\geq 1}$ with $|\lambda_1|\geq|\lambda_2|\geq\cdots\to 0$ and an orthonormal system $\{e_i\}_{i\geq 1}\subset\mathcal{H}$ of eigenvectors $Te_i=\lambda_i e_i$ such that
  \begin{equation}\label{eq:spectral}
    T f \;=\; \sum_{i\geq 1} \lambda_i\,\langle f,e_i\rangle\,e_i\qquad\text{for all }f\in\mathcal{H}.
    \end{equation}
    \end{theorem}
  Trace-class and Hilbert--Schmidt operators are special cases of compact operators on which the spectrum is, respectively, absolutely summable ($\sum_i|\lambda_i|<\infty$) and square-summable ($\sum_i\lambda_i^2<\infty$).

    \paragraph{Kernel integral operator}
    Given a kernel $K$ and the measure $\mu$, the associated \emph{integral operator} $T_K:L^2(\mathcal{X},\mu)\to L^2(\mathcal{X},\mu)$ is defined by
    \begin{equation}\label{eq:integral-op}
      (T_K f)(\bx)\;:=\;\int_{\mathcal{X}} K(\bx,\bx')\,f(\bx')\,d\mu(\bx').
      \end{equation}
      Under the standard assumptions that $K$ is symmetric, continuous (or merely measurable) and square-integrable in the sense that $\int_{\mathcal{X}\times\mathcal{X}} K(\bx,\bx')^2\,d\mu(\bx)d\mu(\bx')<\infty$, the operator $T_K$ is self-adjoint and Hilbert--Schmidt (in particular compact) on $L^2(\mathcal{X},\mu)$.

      Applying Theorem~\ref{thm:spectral} to $T_K$ yields eigenpairs $\{(\lambda_i,e_i)\}_{i\geq 1}$ with $\lambda_i\geq 0$, $\lambda_i\downarrow 0$, and $\{e_i\}$ orthonormal in $L^2(\mathcal{X},\mu)$. Mercer's theorem~\cite{mercer1909functions} then states that, under mild continuity assumptions on $\mathcal{X}$ and $K$ (e.g.\ $\mathcal{X}$ compact and $K$ continuous), the kernel itself admits the absolutely and uniformly convergent expansion
      \begin{equation}\label{eq:mercer}
        K(\bx,\bx')\;=\;\sum_{i\geq 1}\lambda_i\,e_i(\bx)\,e_i(\bx').
        \end{equation}

        \paragraph{Relation between the integral operator and the RKHS}
        The Mercer decomposition~\eqref{eq:mercer} provides an explicit isometric description of the RKHS $\mathcal{H}_K$ in terms of the spectral data of $T_K$. Restricting to indices with $\lambda_i>0$,
        \begin{equation}\label{eq:rkhs-mercer}
          \mathcal{H}_K\;=\;\Bigl\{\,f=\sum_{i:\lambda_i>0}a_i\,e_i\;:\;\|f\|_{\mathcal{H}_K}^2\;:=\;\sum_{i:\lambda_i>0}\frac{a_i^2}{\lambda_i}<\infty\,\Bigr\},
          \end{equation}
          with inner product $\langle f,g\rangle_{\mathcal{H}_K}=\sum_i a_i b_i/\lambda_i$. Equivalently, $\mathcal{H}_K$ is the image of $L^2(\mathcal{X},\mu)$ under the square root $T_K^{1/2}$, with norm $\|T_K^{1/2}g\|_{\mathcal{H}_K}=\|g\|_{L^2(\mathcal{X},\mu)}$ (modulo $\ker T_K$); equivalently, $T_K^{1/2}:L^2(\mathcal{X},\mu)\to\mathcal{H}_K$ is a partial isometry. Two consequences will be used repeatedly below:
          \begin{itemize}
            \item The eigenfunctions $\{\sqrt{\lambda_i}\,e_i\}_{i:\lambda_i>0}$ form an orthonormal basis of $\mathcal{H}_K$, while $\{e_i\}$ form an orthonormal basis of the closure of $\mathrm{range}(T_K)$ in $L^2(\mathcal{X},\mu)$.
              \item Smoothness in $\mathcal{H}_K$ corresponds to spectral concentration: $f\in\mathcal{H}_K$ iff its $L^2$ coefficients $a_i=\langle f,e_i\rangle_{L^2}$ decay fast enough that $\sum_i a_i^2/\lambda_i<\infty$. In particular, functions in the top-$k$ eigenspace of $T_K$ are exactly the ``low-degree'' or ``low-frequency'' functions used throughout the main text.
              \end{itemize}
              This dictionary between the kernel $K$, the integral operator $T_K$ and the RKHS $\mathcal{H}_K$ is what allows us, in the rest of this appendix, to translate between function-space statements (norms, projections, low-degree truncation in $\mathcal{H}_K$) and operator-spectral statements (eigenvalues and eigenfunctions of $T_K$).
              
\subsection{Representer Property}

We begin by establishing a fundamental property: the optimal features lie in the finite-dimensional span of training features, which enables efficient computation via the kernel trick.

\begin{lemma}[Representer property of LoFi features]\label{lem:representer-lofi}
    Let $\{\hat{\bm v}_j^{(\ell)}\}_{j=1}^{k_\ell}$ denote any minimizers of the empirical objective in \eqref{eq:lofi-rkhs-variational}. Then the corresponding weight vectors $\{\hat{\bm v}_j^{(\ell)}\}_{j=1}^{k_\ell}$ lie in the span of the previous-layer features evaluated at the training inputs, i.e.\
    \begin{equation}\label{eq:lofi-span}
        \hat{\bm v}_j^{(\ell)} \in \mathrm{span}\bigl\{\bz_{\ell-1}(\bm x_1),\dots,\bz_{\ell-1}(\bm x_n)\bigr\}, \qquad j=1,\dots,k_\ell.
    \end{equation}
\end{lemma}
\begin{proof}[Proof of \Cref{lem:representer-lofi}]
Fix any index $j\in\{1,\dots,k_\ell\}$. The empirical objective in~\eqref{eq:lofi-rkhs-variational} depends on $\hat{\bm v}_j^{(\ell)}$ only through the inner products $\langle \hat{\bm v}_j^{(\ell)}, \bz_{\ell-1}(\bm x_\mu)\rangle$, $\mu=1,\dots,n$, and is optimized subject to the norm constraint $\|\hat{\bm v}_j^{(\ell)}\|_2 = 1$ (together with the deflation/orthogonality constraints to previously selected directions, which are also expressible solely through such inner products). Decompose $\hat{\bm v}_j^{(\ell)} = \bm v_j^\parallel + \bm v_j^\perp$, where $\bm v_j^\parallel$ is the orthogonal projection onto the finite-dimensional subspace $\mathcal{S} := \mathrm{span}\{\bz_{\ell-1}(\bm x_\mu)\}_{\mu=1}^n \subset \mathbb{R}^{p_{\ell-1}}$ and $\bm v_j^\perp \in \mathcal{S}^\perp$. By orthogonality, $\langle \bm v_j^\perp, \bz_{\ell-1}(\bm x_\mu)\rangle = 0$ for every training input, so $\langle \hat{\bm v}_j^{(\ell)}, \bz_{\ell-1}(\bm x_\mu)\rangle = \langle \bm v_j^\parallel, \bz_{\ell-1}(\bm x_\mu)\rangle$, i.e.\ the data-dependent objective depends only on $\bm v_j^\parallel$. Pythagoras gives $1 = \|\hat{\bm v}_j^{(\ell)}\|_2^2 = \|\bm v_j^\parallel\|_2^2 + \|\bm v_j^\perp\|_2^2$, so any nonzero $\bm v_j^\perp$ forces $\|\bm v_j^\parallel\|_2 < 1$. The rescaled vector $\tilde{\bm v}_j := \bm v_j^\parallel/\|\bm v_j^\parallel\|_2 \in \mathcal{S}$ is then feasible (it has unit norm and inherits the orthogonality constraints, since these involve only inner products with vectors in $\mathcal{S}$) and yields a strictly larger objective by a factor $\|\bm v_j^\parallel\|_2^{-2}>1$, contradicting optimality of $\hat{\bm v}_j^{(\ell)}$. Hence every minimizer satisfies $\bm v_j^\perp = 0$, i.e.\ $\hat{\bm v}_j^{(\ell)} \in \mathcal{S}$, which is~\eqref{eq:lofi-span}.
\end{proof}

\subsection{RKHS-Euclidean Equivalence}\label{app:kernel-equiv}

We now establish a fundamental result that underlies all subsequent proofs in this section. By Lemma~\ref{lem:representer-lofi}, we know features lie in the span of training features. This allows us to identify RKHS constraints with Euclidean constraints on weight vectors.

\begin{lemma}[RKHS-Euclidean Norm Equivalence]\label{lem:rkhs-euclidean}
Let $K_{\ell-1}(\bm x,\bm x')=\langle\bz_{\ell-1}(\bm x),\bz_{\ell-1}(\bm x')\rangle$ and let $\mathcal H_{\ell-1}$ be its induced RKHS. Then the RKHS norm of a linear feature coincides with the Euclidean norm of its weight vector: for any $\bm u\in\mathbb R^{p_{\ell-1}}$, the linear feature
\[
    \varphi_{\bm u}({\bm x}) := \langle \bm u, \bz_{\ell-1}(\bm x)\rangle.
\]
satisfies
\[
    \|\varphi_{\bm u}\|_{\mathcal H_{\ell-1}} = \|\bm u\|_2.
\]
Consequently, the RKHS unit ball $\{\varphi:\|\varphi\|_{\mathcal H_{\ell-1}}\leq 1\}$ and the Euclidean unit sphere $\{\bm u:\|\bm u\|_2=1\}$ are in bijection through $\bm u\mapsto\varphi_{\bm u}$, establishing a primal--dual equivalence between RKHS constraints and Euclidean constraints.
\end{lemma}

\subsection{Proof of Theorem~\ref{thm:variational-rkhs}}

We prove each part of Theorem~\ref{thm:variational-rkhs} in turn, using Lemma~\ref{lem:rkhs-euclidean} to convert between the RKHS and Euclidean formulations.

\paragraph{Proof of part (i).}

By Lemma~\ref{lem:rkhs-euclidean}, the objective equals
\[
    \left|\widehat{\mathbb E}_n\bigl[y\,\psi(\bm x)\bigr]\right|
    =
    \left|\frac{1}{n}\sum_{\mu=1}^n y_\mu\langle \bm u, \bz_{\ell-1}(\bm x_\mu)\rangle\right|
    =
    \left|\left\langle \bm u,\; \hat{\bm u}^\ell\right\rangle\right|,
\]
where $\hat{\bm u}^\ell = \tfrac{1}{n}\sum_{\mu=1}^n y_\mu \bm z_{\ell-1}(\bm x_\mu)$ is the empirical first-order moment computed in Algorithm~\ref{alg:neural-lofi}. By the Cauchy--Schwarz inequality this is maximized at $\bm u = \hat{\bm u}^\ell/\|\hat{\bm u}^\ell\|_2$, giving $\psi^\ell(\bm x)=\langle \hat{\bm u}^\ell, \bm z_{\ell-1}(\bm x)\rangle$ as the unique maximizer (up to sign). \qed

\paragraph{Proof of part (ii).}

By Lemma~\ref{lem:rkhs-euclidean}, the second-order objective equals
\[
    \left|\widehat{\mathbb E}_n\bigl[y\,\varphi(\bm x)^2\bigr]\right|
    =
    \left|\frac{1}{n}\sum_{\mu=1}^n y_\mu\langle \bm u,\bz_{\ell-1}(\bm x_\mu)\rangle^2\right|
    =
    \left|\bm u^\top \widehat{\bm C}^{(\ell)} \bm u\right|,
\]
where $\widehat{\bm C}^{(\ell)} = \tfrac{1}{n}\sum_{\mu=1}^n y_\mu\,\bz_{\ell-1}(\bm x_\mu)\bz_{\ell-1}(\bm x_\mu)^\top$ is the empirical second-order moment operator. By the Courant--Fischer minimax theorem, the unit-norm vector maximizing $|\bm u^\top \widehat{\bm C}^{(\ell)}\bm u|$ is the leading eigenvector $\hat{\bm v}^{(\ell)}_1$ of $\widehat{\bm C}^{(\ell)}$ (in absolute eigenvalue). Successive orthogonal maximizers are the subsequent eigenvectors $\hat{\bm v}^{(\ell)}_2,\dots,\hat{\bm v}^{(\ell)}_{k_\ell}$. Via $\bm u\mapsto\varphi_{\bm u}$, these correspond exactly to $\varphi_1,\dots,\varphi_{k_\ell}$ satisfying the stated variational recursion. \qed

\paragraph{Dual (kernel-space) form}

By Lemma~\ref{lem:representer-lofi}, any optimal feature $\varphi_j$ lies in the span of the $n$ training feature vectors:
\begin{equation}
\varphi_j(\bm x) = \sum_{\mu=1}^n \alpha_{j,\mu} K_{\ell-1}(\bm x, \bm x_\mu).
\end{equation}
Define $\bm\alpha_j = (\alpha_{j,1}, \dots, \alpha_{j,n})^\top$ as the coefficient vector for feature $j$. By Lemma~\ref{lem:rkhs-euclidean}, the RKHS norm of this feature is
\begin{equation}
\|\varphi_j\|^2_{\mathcal H_{\ell-1}} = \bm\alpha_j^\top G_{\ell-1} \bm\alpha_j,
\end{equation}
where $G_{\ell-1} = Z_{\ell-1}Z_{\ell-1}^\top \in \mathbb R^{n \times n}$ is the Gram matrix with entries $G_{\ell-1}(\mu, \nu) = K_{\ell-1}(\bm x_\mu, \bm x_\nu)$.

The empirical objective from Theorem~\ref{thm:variational-rkhs}~(ii) becomes
\begin{equation}
\max_{\bm\alpha:\bm\alpha^\top G_{\ell-1}\bm\alpha=1}
\left|
    \frac{1}{n}\bm\alpha^\top
    G_{\ell-1}Y G_{\ell-1}\bm\alpha
\right|,
\end{equation}
where $Y=\operatorname{diag}(y_1,\ldots,y_n)$ and we used that the vector of
feature evaluations on the training data is $G_{\ell-1}\bm\alpha$.

Hence, we obtain the following proposition.

\begin{proposition}[Generalized Eigenvector Problem]\label{prop:generalized-eigenvector}
The optimal dual coefficients $\{\hat{\bm\alpha}_j\}_{j=1}^{k_\ell}$ that sequentially maximize the above constrained objective can be chosen on the range of $G_{\ell-1}$ to satisfy
\begin{equation}
\frac{1}{n}G_{\ell-1}Y G_{\ell-1}\hat{\bm\alpha}_j
=\lambda_j G_{\ell-1}\hat{\bm\alpha}_j,
\qquad j=1,\dots,k_\ell,
\label{eq:gen-eigenproblem}
\end{equation}
with $\hat{\bm\alpha}_i^\top G_{\ell-1}\hat{\bm\alpha}_j=\delta_{ij}$,
where the generalized eigenvalues $\lambda_j$ are ordered by magnitude. These
coefficients define the second-order features without requiring explicit access
to the feature vectors.
\end{proposition}

\paragraph{Proof of part (iii).}
At every layer, assume that the limiting eigenvalue (or
eigenvalue cluster) being selected is separated from the rest of the spectrum. For a single
feature this means a positive eigengap; for a cluster the statements below hold for the
spectral projector, with an arbitrary orthonormal basis chosen inside the limiting
eigenspace.

\emph{Induction statement.} Let $\rho_{\bm x}$ be the data distribution. After layer $r$, let
\[
    \bm g_r^{(p)}(\bm x)
    =
    \bigl(\psi_r^{(p)}(\bm x),\varphi_{r,1}^{(p)}(\bm x),\dots,
    \varphi_{r,k_r}^{(p)}(\bm x)\bigr)
\]
denote the projected features selected by Neural LoFi at finite width, and let
$K_r^{(p)}$ be the kernel produced by the following random lift. The induction claim
$\mathsf I_r$ is that there exist deterministic functions $\bm g_r^\infty$ and a
deterministic kernel $K_r^\infty$ such that
\begin{equation}
    \|\bm g_r^{(p)}-\bm g_r^\infty\|_{L^2(\rho_{\bm x})}\xrightarrow[p\to\infty]{P}0
    \label{eq:feature-induction}
\end{equation}
and, for the training inputs,
\begin{equation}
    \|\bm G_r^{(p)}-\bm G_r^\infty\|_{\mathrm{op}}\xrightarrow[p\to\infty]{P}0,\qquad
    \sum_{\mu=1}^n
    \|K_r^{(p)}(\cdot,\bm x_\mu)-K_r^\infty(\cdot,\bm x_\mu)\|_{L^2(\rho_{\bm x})}^2
    \xrightarrow[p\to\infty]{P}0,
    \label{eq:kernel-section-induction}
\end{equation}
where $(\bm G_r^{(p)})_{\mu\nu}=K_r^{(p)}(\bm x_\mu,\bm x_\nu)$.
The base case is deterministic: $\bm g_0(\bm x)=\bm x$ and $K_0(\bm x,\bm x')=
\langle\bm x,\bm x'\rangle$.

We first record the kernel-propagation step used in the induction. Suppose that, for some
layer $r$, the projected features already satisfy
$\bm g_r^{(p)}\to\bm g_r^\infty$ in $L^2(\rho_{\bm x})$. For a deterministic feature map
$\bm g$, write
\[
    \mathcal K_r[\bm g](\bm x,\bm x')
    =
    \E_{\bm a\sim\pi_r}
    \bigl[
        \sigma(\bm a^\top \bm g(\bm x))
        \sigma(\bm a^\top \bm g(\bm x'))
    \bigr].
\]
The limiting next-layer kernel is
$K_r^\infty=\mathcal K_r[\bm g_r^\infty]$. Since $\bm g_r^{(p)}\to \bm g_r^\infty$ in
$L^2$ and $\sigma$ is pseudo-Lipschitz with the required moment bounds, dominated
convergence gives convergence of the deterministic kernels
$\mathcal K_r[\bm g_r^{(p)}]\to \mathcal K_r[\bm g_r^\infty]$ on the fixed Gram matrix
and in the $L^2$ section norm in~\eqref{eq:kernel-section-induction}. Conditional on
$\bm g_r^{(p)}$, the finite-width kernel is an average of iid random features:
\[
    K_r^{(p)}(\bm x,\bm x')
    =
    \frac1{p_r}\sum_{i=1}^{p_r}
    \sigma(\bm a_i^\top \bm g_r^{(p)}(\bm x))
    \sigma(\bm a_i^\top \bm g_r^{(p)}(\bm x')).
\]
For each fixed anchor $\bm x_\mu$,
\begin{equation}
    \E_{\bm a}
    \bigl\|
        K_r^{(p)}(\cdot,\bm x_\mu)
        -\mathcal K_r[\bm g_r^{(p)}](\cdot,\bm x_\mu)
    \bigr\|_{L^2(\rho_{\bm x})}^2
    \leq \frac{C_\mu}{p_r},
    \label{eq:rf-section-lln}
\end{equation}
and the same iid law of large numbers applies to the finitely many Gram entries. Thus
\eqref{eq:kernel-section-induction} follows. This is the formal sense in which, once the
features entering a layer have a deterministic limit, the next layer sees a fixed
deterministic distribution.

It remains to show that this deterministic kernel convergence propagates through the
spectral filtering step. Assume $\mathsf I_{\ell-1}$ and abbreviate
$K_p=K_{\ell-1}^{(p)}$, $K_\infty=K_{\ell-1}^\infty$, and
$\bm G_p,\bm G_\infty$ for their Gram matrices. By~\eqref{eq:kernel-section-induction},
\begin{equation}
    \|\bm G_p-\bm G_\infty\|_{\mathrm{op}}\xrightarrow[p\to\infty]{P}0,\qquad
    \sum_{\mu=1}^n
    \|K_p(\cdot,\bm x_\mu)-K_\infty(\cdot,\bm x_\mu)\|_{L^2(\rho_{\bm x})}^2
    \xrightarrow[p\to\infty]{P}0.
    \label{eq:gram-and-section-conv}
\end{equation}
The explicit $O_P(n/\sqrt p)$ Gram-matrix rate in the non-adaptive, fixed-$\bm g$
case is the usual random-feature concentration bound; for the section convergence,
\eqref{eq:rf-section-lln} gives the sharper Hilbert-space law of large numbers needed for
out-of-sample features.

By Lemma~\ref{lem:representer-lofi}, every selected feature has the form
\[
    \varphi(\bm x)=\sum_{\mu=1}^n \alpha_\mu K_p(\bm x,\bm x_\mu).
\]
Its RKHS norm and empirical second-order objective are
\[
    \|\varphi\|_{\cH_p}^2=\balpha^\top \bm G_p\balpha,\qquad
    \widehat{\E}_n[y\varphi(\bm x)^2]
    =
    \frac1n\balpha^\top \bm G_p Y \bm G_p\balpha,
    \qquad
    Y=\operatorname{diag}(y_1,\dots,y_n).
\]
Therefore, on the range of $\bm G_p$, the generalized eigenproblem
in~\eqref{eq:gen-eigenproblem} is equivalently the
symmetric eigenproblem
\begin{equation}
    \bm B_p \bm\beta=\lambda \bm\beta,\qquad
    \bm B_p
    =
    \frac1n \bm G_p^{1/2}Y\bm G_p^{1/2},\qquad
    \bm\beta=\bm G_p^{1/2}\balpha.
    \label{eq:symmetric-dual-eigenproblem}
\end{equation}
The square-root map is continuous on positive semidefinite matrices, hence
$\|\bm B_p-\bm B_\infty\|_{\mathrm{op}}\to 0$ in probability. If the $k$th limiting eigenvalue is
isolated with gap $\Delta_k>0$, the Davis--Kahan $\sin\Theta$ theorem
\cite{davis1970rotation} implies convergence of the corresponding projectors at
rate $O_P(\|\bm B_p-\bm B_\infty\|_{\mathrm{op}}/\Delta_k)$. For a simple eigenvalue, after fixing
the sign,
\begin{equation}
    \bm\beta_{p,k}\xrightarrow[p\to\infty]{P}\bm\beta_{\infty,k}.
    \label{eq:beta-conv}
\end{equation}

We next convert the convergence of $\bm\beta_{p,k}$ into convergence of the actual feature functions.
Let
\[
    \bm k_p(\bm x)=\bigl(K_p(\bm x,\bm x_1),\dots,K_p(\bm x,\bm x_n)\bigr)^\top,
    \qquad
    \balpha_{p,k}=\bm G_p^{\dagger/2}\bm\beta_{p,k},
\]
where $\dagger$ denotes the Moore--Penrose inverse on the stable range of the limiting
Gram matrix. If $\bm G_\infty$ is nonsingular, this is the ordinary inverse square root. Continuity of the pseudo-inverse operator and~\eqref{eq:beta-conv} give
$\balpha_{p,k}\to\balpha_{\infty,k}$ in probability. Hence, with
\[
    \widehat\varphi_{p,k}(\bm x)=\bm k_p(\bm x)^\top\balpha_{p,k},\qquad
    \phi^\infty_k(\bm x)=\bm k_\infty(\bm x)^\top\balpha_{\infty,k},
\]
we have
\begin{equation}
\begin{aligned}
    \|\widehat\varphi_{p,k}-\phi^\infty_k\|_{L^2(\rho_{\bm x})}
    &\leq
    \|\bm k_p-\bm k_\infty\|_{L^2(\rho_{\bm x};\R^n)}\,
    \|\balpha_{p,k}\|_2 \\
    &\quad+
    \|\bm k_\infty\|_{L^2(\rho_{\bm x};\R^n)}\,
    \|\balpha_{p,k}-\balpha_{\infty,k}\|_2
    \xrightarrow[p\to\infty]{P}0.
\end{aligned}
    \label{eq:feature-function-conv}
\end{equation}
The linear feature satisfies the same conclusion directly, since
\[
    \psi_p^\ell(\bm x)
    =
    \frac1n\sum_{\mu=1}^n y_\mu K_p(\bm x,\bm x_\mu)
    \to
    \frac1n\sum_{\mu=1}^n y_\mu K_\infty(\bm x,\bm x_\mu)
    =
    \psi_\infty^\ell(\bm x)
\]
in $L^2(\rho_{\bm x})$. Thus the whole projected vector
$\bm g_\ell^{(p)}=(\psi_p^\ell,\widehat\varphi_{p,1},\dots,\widehat\varphi_{p,k_\ell})$
converges to the deterministic vector $\bm g_\ell^\infty$ in $L^2$. Applying the first
part of the induction to this deterministic limit gives convergence of the next random
feature kernel $K_\ell^{(p)}$ to
\[
    K_\ell^\infty(\bm x,\bm x')
    =
    \E_{\bm a\sim\pi_\ell}
    \left[
        \sigma(\bm a^\top\bm g_\ell^\infty(\bm x))
        \sigma(\bm a^\top\bm g_\ell^\infty(\bm x'))
    \right],
\]
which proves $\mathsf I_\ell$. By induction, all finite collections of learned features and
the induced kernels converge layer by layer to deterministic infinite-width limits.

 \qed

\newpage 

\section{Effective dimension and sample complexity of emergence}
\label{app:effective-dimension}

As discussed in the main text,  training dynamics is found in practice to often display long plateaus followed by abrupt transitions, with new directions in representation space \emph{emerging} sequentially \cite{wei2022emergent,raventos2023pretraining,arora2023theory,schaeffer2023emergent}. We return in this section to the emergence of learned features at each layer, and in particular, we now make the noise level \(\tau^k_\ell(n)\) in \eqref{eq:lofi-emergence-threshold} explicit.  Throughout the section our results apply conditioned on a fixed kernel $K_{\ell-1}(\bm x,\bm x')$ defined by the features $\bz_{\ell-1}$ of the previous layers.

Recall the residual candidate class
\[
  \mathcal{S}^\ell_{k}
  =
  \left\{
  \varphi\in\mathcal H_{\ell-1}:
  \norm{\varphi}_{\mathcal H_{\ell-1}}=1,\,
  \varphi \perp \varphi_1,\cdots, \varphi_{k-1}
  \right\}.
\]
Let \(T_{\ell-1}:L^2(P_x)\to L^2(P_x)\) be the integral operator of the current kernel,
\[
    (T_{\ell-1}f)(\bm x)
    =
    \int K_{\ell-1}(\bm x,\bm x')f(\bm x')\,dP_x(\bm x'),
\]
and let \(\Pi^\perp_k\) denote the projection orthogonal to the selected features \(\varphi_1,\dots,\varphi_{k-1}\). The residual covariance operator is
\[
    T_{\ell,k}
    :=
    \Pi^\perp_k T_{\ell-1}\Pi^\perp_k .
\]
For a resolution parameter \(r>0\), define the {\it residual effective dimension}:
\begin{definition}[Residual effective dimension]
    \label{def:lofi-effective-dimension}
Let  \(\{\lambda^{\ell,k}_j\}_{j\geq 1}\) denote the eigenvalues of \(T_{\ell,k}\). The \emph{effective residual dimension} at scale \(r\) is defined as
    \begin{equation}
    D^\ell_k(r)
    \coloneqq
    \operatorname{Tr}\!\left[
        T_{\ell,k}(T_{\ell,k}+r I)^{-1}
    \right]
    =
    \sum_{j\geq 1}
    \frac{\lambda^{\ell,k}_j}{\lambda^{\ell,k}_j+r},
    \label{eq:lofi-effective-dimension}
\end{equation}
\end{definition}
The above notion of effective dimension is standard in kernel regression
\cite{caponnetto2007optimal,rudi2017generalization}, and similar quantities
appear in local Rademacher analyses of kernel classes
\cite{mendelson2003performance,bartlett2005local}. \(D^\ell_k(r)\) can be
interpreted as counting the number of residual directions whose variance is
above the resolution \(r\).

\begin{assumption}[Residual eigenfunction hypercontractivity]
\label{ass:residual-eigenfunction-hypercontractivity}
Let \(\mathcal E_{\ell,k}\) denote the
\(L^2(P_x)\)-closed span of the residual eigenfunctions
\(\{\phi_j^{\ell,k}\}_{j\ge1}\). There exists \(H_\ell<\infty\) such that every
\(g\in\mathcal E_{\ell,k}\) satisfies
\[
    \norm{g}_{L^4(P_x)}
    \le
    H_\ell \norm{g}_{L^2(P_x)}.
\]
Equivalently, the same inequality holds for every \(L^2\)-convergent expansion
\(g=\sum_j c_j\phi_j^{\ell,k}\) with \(\sum_jc_j^2<\infty\).

\end{assumption}

Such an assumption holds, in particular, for the polynomial eigenbases that
appear in random-feature kernels such as Hermite polynomials and spherical harmonics (\cite{mei2022generalization}).

\begin{theorem}[Emergence sample complexity]
\label{prop:emergence}
Suppose that Assumption~\ref{ass:residual-eigenfunction-hypercontractivity}
holds and the kernel \(K_{\ell-1}\) and labels \(y\) are uniformly bounded. Let
\[
    r_\ell^{k,\star}
    \coloneqq
    \argmax_{r:\, r \leq \lambda_1^{\ell,k}}
    r\sqrt{D^\ell_k(r)} .
\]
Here \(\lambda_1^{\ell,k}\) is the top residual eigenvalue, hence the largest
possible variance scale of a unit-RKHS candidate in \(\mathcal S^\ell_k\).
Then, for any \(\delta>0\), there is a constant
\(\widetilde C_{\ell,H}>0\) such that with probability at least \(1-\delta\),
\begin{equation}
    \begin{aligned}
     \tau^k_\ell(n)
     &\coloneqq
     \sup_{\varphi \in  \mathcal{S}^\ell_{k}}
    \left|
    \widehat c_{\ell,n}(\varphi)-c_\ell(\varphi)
    \right|
    \le
    \widetilde C_{\ell,H}\,
    r_\ell^{k,\star}
    \sqrt{\frac{D^\ell_k(r_\ell^{k,\star})}{n}} .
    \end{aligned}
\end{equation} This gives, up to polylogarithmic factors, the sample
scale
\[
    n_k
    \gtrsim
    \frac{
        (r_\ell^{(k)})^2 D^\ell_k(r_\ell^{(k)})
    }{
        (\rho_\ell^{(k)})^2
    },
\]
for recovery of $\varphi_k$ conditional on the recovery of previous features $\varphi_1,\cdots,\varphi_{k-1}$.
\end{theorem}

\begin{remark}\label{rem:order}
The order of maximizers of the noise scale $r\sqrt{D^\ell_k(r)}$ may not perfectly match the order of maximizers of the population correlation $\abs{\mathbb{E}[y \varphi(\bx)^2]}$. Hence, the next emergent feature after recovering $\varphi_1,\cdots,\varphi_{k-1}$ may not be $\varphi_{k}$. For simplicity, we neglect the correction in sample complexity due to such order switches.
\end{remark}
\begin{proof}

 For \(r>0\), let
\[
    \mathcal S^\ell_k(r)
    :=
    \{\varphi\in\mathcal S^\ell_k:\E[\varphi(\bm X)^2]\le r\},
\]

We proceed by decomposing the fluctuations into contributions from $ S^\ell_k(r)$ for different variance scales $r$. Define
\begin{equation}
    \tau^k_\ell(r,n)
    \coloneqq
    \sup_{\varphi \in \mathcal S^\ell_k(r)}
    \abs{\hat{\mathbb{E}}[y\varphi^2]-\mathbb{E}[y\varphi^2]} .
\end{equation}
We will show that:
\begin{equation}
    \tau^k_\ell(r,n)
    \lesssim
    \widetilde C_{\ell,H}\,
    r\sqrt{\frac{D^\ell_k(r)}{n}}
    \label{eq:lofi-noise-floor}
\end{equation}

The above bound follows through an extension of Theorem~2.1 in
\cite{mendelson2003performance} which bounds the rademacher complexity of $S^\ell_k(r)$. Concretely, Theorem 2.1 in \cite{mendelson2003performance} provides that for a unit ball of an RKHS of a bounded kernel with eigenvalues \(\{\lambda_j\}_{j\ge 1}\):
\[
    \mathbb E R_n
    \left\{
        f\in\mathcal H:\norm{f}_{\mathcal H}\le 1,\,
        \mathbb E f(\bm X)^2\le r
    \right\}
    \lesssim
    \left(
        \frac{1}{n}\sum_{j\ge 1}\min\{r,\lambda_j\}
    \right)^{1/2}.
\]
Since for every \(r>0\) and every
eigenvalue \(\lambda_j\),
\[
    \frac12 \min\{r,\lambda_j\}
    \le
    \frac{r\lambda_j}{r+\lambda_j}
    \le
    \min\{r,\lambda_j\},
\]
the same bound holds for $D^\ell_k(r)$ defined in Definition \ref{def:lofi-effective-dimension}.

We now adapt the proof of Theorem~2.1 of
\cite{mendelson2003performance} to the quadratic process. By the boundedness of the labels $y$, $\tau^k_\ell(r,n)$ can be bounded up to constants by the rademacher complexity of $\varphi^2$ over $S^\ell_k(r)$, defined as:

\begin{equation}
 \E_\epsilon[\sup_{\varphi\in\mathcal S^\ell_k(r)}
    \left|
        \sum_{i=1}^n\epsilon_i\varphi(\bm X_i)^2
        \right|],
\end{equation}
where $\epsilon \in \{\pm 1\}$ denote standard rademacher random variables.

Let
\(\{\phi_j^{\ell,k}\}_{j\ge1}\) denote the \(L^2(P_x)\)-orthonormal eigenfunctions
of \(T_{\ell,k}\) 
Any $\varphi$ with $\norm{\varphi}_{\mathcal{H}}=1$ can be expressed in the basis w.r.t \(\{\phi_j^{\ell,k}\}_{j\ge1}\)  as
\[
    \varphi_\beta(\bm x)
    =
    \sum_{j\ge1}
    \beta_j\sqrt{\lambda_j^{\ell,k}}\,
    \phi_j^{\ell,k}(\bm x),
    \qquad
    \sum_j\beta_j^2\le 1.
\]
The constraint $\E[\varphi(\bm X)^2]\le r$ then translates to
\(\sum_j\lambda_j^{\ell,k}\beta_j^2\le r\). Hence, the set of coefficients for $\varphi \in S^\ell_k(r)$ is given by:
\[
    I_r
    :=
    \left\{
        \beta:
        \sum_j\beta_j^2\le 1,\quad
        \sum_j\lambda_j^{\ell,k}\beta_j^2\le r
    \right\}.
\]
The proof of Theorem~2.1 in 
\cite{mendelson2003performance} relates the set $I_r$ to an ellipsoid $B_r$ through the following inclusions:
\[
    B_r\subset I_r\subset \sqrt2\,B_r,
    \qquad
    B_r
    :=
    \left\{\beta:\sum_j\mu_j(r)\beta_j^2\le1\right\},
    \qquad
    \mu_j(r)=\left(\min\{1,r/\lambda_j^{\ell,k}\}\right)^{-1}.
\]
Thus the supremum over \(I_r\) is bounded, up to an absolute constant, by the
supremum over \(B_r\). Setting \(\theta_j=\sqrt{\mu_j(r)}\,\beta_j\), 
\(B_r\) is re-parameterized as the Euclidean unit ball \(\norm{\theta}_2\le1\). Next, define
\[
    a_j(r) \coloneqq \frac{\lambda_j^{\ell,k}}{\mu_j(r)}
    =
    \min\{\lambda_j^{\ell,k},r\}
\]

The decomposition $\varphi_\beta(\bm x)
    =
    \sum_{j\ge1}
    \beta_j\sqrt{\lambda_j^{\ell,k}}\,
    \phi_j^{\ell,k}(\bm x)$ can then be re-expressed as:
    \begin{align*}
         &\sum_{j\ge1}
    \beta_j\sqrt{\lambda_j^{\ell,k}}\,
    \phi_j^{\ell,k}(\bm x)\\&= \sum_j
    \sqrt{\mu_j(r)}\,\beta_j \frac{1}{\sqrt{\mu_j(r)}}\phi_j^{\ell,k}(\bm x)\\&= \sum_j
            \theta_j\sqrt{a_j(r)}\phi_j^{\ell,k}(\bm x)
    \end{align*}

Hence,
\begin{equation}\label{eq:radbound}
    \sup_{\beta\in I_r}
    \left|
        \sum_{i=1}^n
        \epsilon_i\varphi_\beta(\bm X_i)^2
    \right|
    \lesssim
    \sup_{\norm{\theta}_2\le1}
    \left|
        \sum_{i=1}^n
        \epsilon_i
        \left(
            \sum_j
            \theta_j\sqrt{a_j(r)}\phi_j^{\ell,k}(\bm X_i)
        \right)^2
    \right|.
\end{equation}
Define:
\[
    v_{\ell,k,r}(\bm x)
    :=
    \left(\sqrt{a_j(r)}\,\phi_j^{\ell,k}(\bm x)\right)_{j\ge1},
\]
then the RHS in Equation \ref{eq:radbound} can be expressed as:
\begin{equation}\label{eq:matconc}
    \left\|
        \sum_{i=1}^n
        \epsilon_i
        v_{\ell,k,r}(\bm X_i)\otimes
        v_{\ell,k,r}(\bm X_i)
    \right\|_{\mathrm{op}}.
\end{equation}

Let
\[
    \psi_{\ell,k}(r)^2
    :=
    \sum_j a_j(r)
    \asymp rD^\ell_k(r).
\]

We now bound Equation \ref{eq:matconc} through a matrix concentration argument which requires bounding the operator norm of the following fourth-moment matrix:
\[
    M_{\ell,k}(r)
    :=
        \E\!\left[
            \norm{v_{\ell,k,r}(\bm X)}_2^2
            v_{\ell,k,r}(\bm X)
            \otimes
            v_{\ell,k,r}(\bm X)
        \right].
\]
We bound $\norm{M_{\ell,k}(r)}$ through Assumption \ref{ass:residual-eigenfunction-hypercontractivity}. For every \(\norm{u}_2=1\),
write \(g_u(\bm x)=\langle u,v_{\ell,k,r}(\bm x)\rangle\). Then
\[
    \E[g_u(\bm X)^2]
    =
    \sum_j u_j^2a_j(r)
    \le r,
\]
since $a_j(r) \leq r$. Subsequently by Cauchy--Schwarz and hypercontractivity (Assumption \ref{ass:residual-eigenfunction-hypercontractivity}), we obtain:
\[
    \begin{aligned}
    \left\langle u,M_{\ell,k}(r)u\right\rangle
    &=
    \sum_j a_j(r)\,
    \E\!\left[
        \big(\phi_j^{\ell,k}(\bm X)\big)^2g_u(\bm X)^2
    \right]                                                   \\
    &\le
    \sum_j a_j(r)\,
    \norm{\phi_j^{\ell,k}}_{L^4(P_x)}^2
    \norm{g_u}_{L^4(P_x)}^2                                   \\
    &\le
    H_\ell^4
    \psi_{\ell,k}(r)^2
    \E[g_u(\bm X)^2]
    \le
    H_\ell^4 r\,\psi_{\ell,k}(r)^2 .
    \end{aligned}
\]
Therefore
\[
   \norm{M_{\ell,k}(r)}_2
    \lesssim
    H_\ell^4 r\,\psi_{\ell,k}(r)^2
    \asymp
    H_\ell^4 r^2D^\ell_k(r).
\]
Matrix Bernstein then gives, up to constants and polylogarithmic
factors:
\[
    \E_\epsilon
    \sup_{\varphi\in\mathcal S^\ell_k(r)}
    \left|
        \sum_{i=1}^n\epsilon_i\varphi(\bm X_i)^2
    \right|
    \lesssim
    H_\ell^2\sqrt{nr}\,\psi_{\ell,k}(r)
    \asymp
    H_\ell^2 n r\sqrt{\frac{D^\ell_k(r)}{n}}.
\]
By symmetrization and the boundedness of labels, this
yields \eqref{eq:lofi-noise-floor}. Taking the maximum over shells gives
the stated bound with \(r_\ell^{k,\star}\).
\end{proof}

\subsection{Estimating \texorpdfstring{$\tau^k_\ell(n)$}{tau-l-k(n)} in Feature space}
We discuss here how $\tau^k_\ell(n)$ can be estimated directly in the feature space (rather than in function space, as we did so far). 

Define the true label, weighted covariance:
\begin{equation}
    \bm C^{(\ell)}
    \coloneqq
    \mathbb{E}[
    y\,
    \bz_{\ell-1}(\bm x)
    {\bz}_{\ell-1}(\bm x)^\top] ,
\end{equation}
and the unweighted covariance:
\begin{equation}
    \bm \Sigma^{(\ell)}
    \coloneqq
    \mathbb{E}[
    \bz_{\ell-1}(\bm x)
    {\bz}_{\ell-1}(\bm x)^\top].
\end{equation}

To estimate the cutoff $n^k_\ell$ for the number of samples required for the recovery of $\phi_k$, we proceed as follows:
\begin{enumerate}[leftmargin=2em]
    \item Let $\bv_1,\dots \bv_{k-1}$ denote the top $k-1$ eigenvectors of $\bm C^{(\ell)}$. Compute the residual covariance:
    \begin{equation}
        \bm\Sigma^{(\ell)}_k \coloneqq (I-\bv_{k-1}\bv_{k-1}^\top) \bm\Sigma^{(\ell)}_{k-1}(I-\bv_{k-1}\bv_{k-1}^\top).
    \end{equation}

    \item Let $\{\lambda^{\ell,k}_j\}_{j=1}^\infty$ denote the eigenvalues of $\bm\Sigma^{(\ell)}_k$ . Then, for any $r \in \mathbb{R}$, define:
    \begin{equation}
       D^k_{\ell}(r) \coloneqq \sum_{j\geq 1} \frac{\lambda^{\ell,k}_j}{\lambda^{\ell,k}_j+r}.
    \end{equation}
    
    \item Set $r^k_\ell{=}\argmax_{r \leq \lambda_1^{\ell,k}} r \sqrt{D^k_{\ell}(r)}$.
    We predict the emergence of $\bv_{k}$ with samples $n^k_{\ell}$ satisfying:
    \begin{equation}\label{eq:real-data-emergence-estimate}
        n^k_{\ell} \sim  \left(\frac{r^k_\ell}{\rho^{(k)}_\ell}\right)^2 D^k_{\ell},
    \end{equation}
    where $\rho^{(k)}_\ell$ is the $k_{th}$ eigenvalue of $\bm C^{(\ell)}$.
\end{enumerate}

\subsection{Effective dimension in the Hermite toy model}
\label{app:effective-dimension-hermite}

The abstract quantity \(D^\ell_k(r)\) becomes concrete in the toy model Appendix~\ref{app:th_rf}. In this paragraph, \(k\) denotes the
Hermite degree used to define the first hidden representation \(h^{(1)}\), not
the index of the sequentially selected feature. Write the flattened degree-\(k\)
Hermite tensor as
\[
    H_k(\bm x)
    =
    \{H_\alpha(\bm x):|\alpha|=k\}
    \in \mathbb R^{D_k},
    \qquad
    D_k=\binom{d+k-1}{k}.
\]
For Gaussian inputs, these coordinates are orthonormal in \(L^2(P_x)\). Hence,
for any two coefficient vectors \(a,b\in\mathbb R^{D_k}\),
\[
    \mathbb E\!\left[
        \langle a,H_k(\bm X)\rangle
        \langle b,H_k(\bm X)\rangle
    \right]
    =
    \langle a,b\rangle_{\mathbb R^{D_k}} .
\]

Thus a degree-\(k\) candidate feature
\(\varphi_a(\bm x)=\langle a,H_k(\bm x)\rangle\) is literally a vector in a
\(D_k\)-dimensional Euclidean feature space as far as \(L^2(P_x)\) norms and
projections are concerned.

Equivalently, the degree-\(k\) Hermite kernel
\[
    K_{H,k}(\bm x,\bm x')
    =
    \langle H_k(\bm x),H_k(\bm x')\rangle
    =
    \sum_{|\alpha|=k} H_\alpha(\bm x)H_\alpha(\bm x')
\]
has integral operator
\[
    T_{H,k}f
    =
    \sum_{|\alpha|=k}
    \langle f,H_\alpha\rangle_{L^2(P_x)} H_\alpha .
\]
This is the orthogonal projector onto the degree-\(k\) Hermite chaos. Its only
nonzero eigenvalue is \(1\), with multiplicity \(D_k\). If the same block is
reached through a random-feature activation, as in the construction of
Subsection~\ref{gif:rf_maths}, the eigenvalue is multiplied by the squared
Hermite coefficient and by normalization constants, but the multiplicity is
still at most \(D_k\).

Therefore,
effective dimension relevant for the recovery of \(h^{(1)}\) is bounded by this
Hermite block rank. For a block eigenvalue \(a_k>0\),
\[
    D_{\mathrm{toy},1}^{(k)}(r)
    =
    \sum_{j=1}^{D_k}\frac{a_k}{a_k+r}
    =
    \frac{a_k}{a_k+r}D_k
    \le D_k
    =
    \binom{d+k-1}{k}
    =
    O(d^k),
\]
and for resolutions \(\lambda\lesssim a_k\) this is
\(\asymp D_k\). Hence, during first-level
feature recovery, the statistical fluctuation is governed by the ambient
degree-\(k\) Hermite space.

The planted first-level variables in Appendix~\ref{app:th_rf},
\[
    h_i^{(1)}(\bm x)
    =
    \langle A_i^{(1)},H_k(\bm x)\rangle,
    \qquad i=1,\ldots,d_1,
\]
form a \(d_1=d^\epsilon\) dimensional signal subspace inside this
\(D_k=O(d^k)\) dimensional Hermite block. The residual effective dimension
controls the size of the empirical noise over the full candidate block, while
the planted subspace controls the rank and strength of the population signal.
Concretely, by the definition of $y$, we have that:
\begin{equation}
     \mathbb{E}[y(h_i^{(1)}(\bm x))^2] \sim \frac{1}{\sqrt{d_1}}
\end{equation}

Substituting the above scaling into the sample complexity prediction prescribed by Theorem \ref{prop:emergence} gives the
\(d^{k+\epsilon}\) first-stage sample scale stated in
Appendix~\ref{app:th_rf}.

\subsection{Effective dimension in the multi-index models}

We now discuss how the sample complexity prediction in Proposition \ref{prop:emergence} recovers the rates in the setting of multi-index models with power-law dependence on the label. 

We follow here \cite{defilippis2025scaling,defilippis2026optimal} and consider hierarchical multi-index target \(y = f^*(\langle\bm{w}_1^*,\bm{x}\rangle,\dots,\langle\bm{w}_r^*,\bm{x}\rangle) + \xi\) with isotropic Gaussian inputs \(\bm{x}\sim\mathcal{N}(0,\bm{I}_d)\) and orthonormal teacher directions \(\{\bm{w}_i^*\}_{i=1}^r\). Because the input covariance is the identity \(D^\ell \asymp d\), at the slice-wise scale used in Proposition~\ref{prop:emergence}. The features at layer \(\ell\) along each candidate direction have unit variance under the Gaussian input, so the variance scale is order one, \(\tau = r_\ell^{(k)} = O(1)\). Finally, the power-law dependence of the label on the teacher coefficients translates into a power-law decay of the population correlations across the planted directions: the \(i\)-th spike satisfies \(\rho_\ell^{(i)} = O(i^{-\gamma})\), where \(\gamma>0\) is the exponent governing the label spectrum. Plugging \(D_k^\ell \asymp d\), \(\tau = O(1)\) and \(\rho^{(i)} = O(i^{-\gamma})\) into the recovery condition \(n \gg (r_\ell^{(k)})^2 D_k^\ell / (\rho_\ell^{(k)})^2\) of Proposition~\ref{prop:emergence} yields the sample-complexity threshold \(n_i \gtrsim d\, i^{2\gamma}\) for resolving the \(i\)-th spike, which matches the optimal scaling laws derived in~\cite{defilippis2025scaling,defilippis2026optimal}. 

\newpage
\section{\texorpdfstring{{Comparison with AGOP/RFM}}{Comparison with AGOP/RFM}}\label{app:agop}

This appendix compares Neural LoFi with AGOP/RFM-type feature updates. In
Recursive Feature Machines, one fits a predictor and uses its gradient with
respect to the input, or to the current representation, to identify directions
on which the fitted function is sensitive. Deep RFM
\cite{beaglehole2024average} combines this gradient-based metric update with a
nonlinear feature map and repeats the procedure across layers.  AGOP and
Recursive Feature Machines therefore provide an adaptive alternative to fixed
kernels, and have been shown to recover meaningful features, including edge
detectors, in neural and kernel models
\cite{radhakrishnan2024mechanism,beaglehole2023mechanism,radhakrishnan2025linear}.
The distinction we emphasize is not that RFM is non-adaptive, but that its
adaptation is mediated by a predictor that has already been fit.

More precisely, RFM first forms a kernel-ridge-regression predictor \(\hat f\)
and then updates the feature metric using its empirical average gradient outer
product
\[
    \widehat G_{\mathrm{AGOP}}(\hat f)
    :=\frac1n\sum_{\mu=1}^n
       \nabla\hat f(\bx_\mu)\nabla\hat f(\bx_\mu)^\top.
\]
The resulting metric can reveal a planted feature only through the dependence
already present in \(\hat f\). Consequently, the sample size needed by an
AGOP/RFM update is tied to the prediction threshold of the preceding KRR fit.
This threshold is model dependent and can differ from the recovery threshold
of Neural LoFi, which estimates the label-weighted feature moment directly
rather than first fitting the target as a function. We first analyze a
quadratic Gaussian multi-index target, for which the two thresholds differ,
and then discuss the corresponding prediction threshold for Deep RFM on the
hierarchical model of Appendix~\ref{app:th_rf}.

\subsection{Quadratic multi-index model}

The quadratic model is the simplest case in which the target depends on a
fixed low-dimensional subspace but its coefficient belongs to a polynomial
space whose dimension grows as \(d^2\).  It lets us compare the full AGOP of a
sub-threshold KRR predictor with the low-rank spectral truncation used by
Neural LoFi.

Let
\[
    \bx\sim\mathcal N(0,I_d),
    \qquad
    y=\langle A_\star,H_2(\bx)\rangle,
    \qquad
    A_\star=V^\star B_\star V^{\star\top},
\]
where \(V^\star\in\mathbb R^{d\times r}\) has orthonormal columns,
\(B_\star=B_\star^\top\) is full rank, \(r=O(1)\), and
\(H_2(\bx)=(\bx\bx^\top-I_d)/\sqrt2\). The target AGOP is
\[
    G(y):=\E[\nabla y(\bx)\nabla y(\bx)^\top]
    =2A_\star^2,
\]
so its range is exactly \(\operatorname{span}(V^\star)\).

Since the target is purely quadratic, we isolate its relevant component by
considering the degree-two kernel
\[
    K_2(\bx,\bx')
    =\frac{\kappa_2}{D_2}
       \langle H_2(\bx),H_2(\bx')\rangle_F,
    \qquad
    D_2=\frac{d(d+1)}2,
\]
with fixed \(\kappa_2>0\). This restriction removes higher-order contributions
to the AGOP and makes the relation between KRR and its quadratic coefficient
explicit. Every KRR predictor for this kernel has the form
\[
    \hat f_n(\bx)=\langle\widehat Q_n,H_2(\bx)\rangle,
    \qquad \widehat Q_n=\widehat Q_n^\top.
\]
It follows immediately that
\begin{equation}
\label{eq:quadratic-krr-agop}
    G(\hat f_n)
    :=\E_\bx[\nabla\hat f_n(\bx)\nabla\hat f_n(\bx)^\top]
    =2\widehat Q_n^2.
\end{equation}
Let \(P_\star=V^\star V^{\star\top}\). Define the normalized trace of the
AGOP projected onto the planted subspace by
\begin{equation}
\label{eq:quadratic-agop-overlap}
    \mathcal E_\star(\hat f_n)
    :=\frac{\operatorname{Tr}(P_\star G(\hat f_n))}
            {\operatorname{Tr}(G(\hat f_n))}
    =\frac{\|\widehat Q_nP_\star\|_F^2}
           {\|\widehat Q_n\|_F^2}.
\end{equation}

\begin{proposition}[Quadratic AGOP bottleneck]
\label{prop:quadratic-agop-bottleneck}
Assume \(\|A_\star\|_F=\Theta(1)\), fix \(\delta>0\), and consider KRR with
kernel \(K_2\), penalty \(\lambda_n\geq0\), and
\(n\to\infty\) satisfying \(n=O(d^{2-\delta})\). Then
\[
    \mathcal E_\star(\hat f_n)
    =o_\mathbb P(1).
\]
Thus, before the \(d^2\) prediction threshold,
\(\operatorname{Tr}(P_\star G(\hat f_n))/\operatorname{Tr}(G(\hat f_n))\)
vanishes in probability.
\end{proposition}

\begin{proof}
We begin with the exact relation between the KRR predictor and its degree-two
Hermite coefficient. Identify the degree-two Hermite space with the
\(D_2\)-dimensional space of symmetric matrices equipped with the Frobenius
inner product, and choose an orthonormal basis \(\{E_a\}_{a=1}^{D_2}\) of this
space. Define
\[
    h(\bx)_a:=\langle E_a,H_2(\bx)\rangle,
    \qquad
    \Phi_{\mu a}:=h(\bx_\mu)_a,
\]
so that \(\Phi\in\mathbb R^{n\times D_2}\) and
\(K_n=(\kappa_2/D_2)\Phi\Phi^\top\) is the training kernel matrix. For the
objective
\[
    \frac1n\sum_{\mu=1}^n(y_\mu-f(\bx_\mu))^2
      +\lambda_n\|f\|_{\mathcal H_{K_2}}^2,
\]
write \(y=(y_1,\ldots,y_n)^\top\). The dual coefficients and the fitted
quadratic coefficient are exactly
\begin{equation}
\label{eq:quadratic-krr-dual-coefficient}
    \alpha=(K_n+n\lambda_n I_n)^{-1}y,
    \qquad
    \widehat Q_n
      =\frac{\kappa_2}{D_2}
        \sum_{\mu=1}^n\alpha_\mu H_2(\bx_\mu).
\end{equation}

Gaussian Hermite orthogonality gives
\(\E[h(\bx)h(\bx)^\top]=I_{D_2}\). The degree-two kernel-matrix
concentration used in the polynomial-regime KRR results
\cite{ghorbani2020neural,mei2022generalization,hu2024asymptotics} states that,
for \(n=O(d^{2-\delta})\),
\begin{equation}
\label{eq:quadratic-hermite-gram}
    \Delta_n:=\frac1{D_2}\Phi\Phi^\top-I_n,
    \qquad
    \|\Delta_n\|_{\mathrm{op}}=o_\mathbb P(1).
\end{equation}
Substituting \(K_n=\kappa_2(I_n+\Delta_n)\) into
\eqref{eq:quadratic-krr-dual-coefficient} and expanding the inverse gives
\[
    \alpha
      =\frac1{\kappa_2+n\lambda_n}
        (I_n+E_n)y,
    \qquad
    \|E_n\|_{\mathrm{op}}=o_\mathbb P(1).
\]
Now define the empirical degree-two label coefficient
\[
    \overline Q_n:=\frac1n\sum_{\mu=1}^n
       y_\mu H_2(\bx_\mu).
\]
Equation~\eqref{eq:quadratic-krr-dual-coefficient} therefore yields
\begin{equation}
\label{eq:krr-quadratic-degree-two-expansion}
    \widehat Q_n=a_{n,d}\overline Q_n+R_n,
    \qquad
    a_{n,d}:=\frac{\kappa_2 n}
      {D_2(\kappa_2+n\lambda_n)},
\end{equation}
where, using \(\|\Phi\|_{\mathrm{op}}^2=D_2(1+o_\mathbb P(1))\) and
\(\|y\|_2=O_\mathbb P(\sqrt n)\),
\begin{equation}
\label{eq:krr-quadratic-remainder}
    \|R_n\|_F
      =o_\mathbb P\!\left(
          a_{n,d}\sqrt{\frac{D_2}{n}}
        \right).
\end{equation}

It remains to compare \(\|\overline Q_nP_\star\|_F\) with
\(\|\overline Q_n\|_F\). Hermite orthogonality gives
\(\E\overline Q_n=A_\star\). Decompose
\(\bx=V^\star z+V_\perp^\star u\), where
\(z\sim\mathcal N(0,I_r)\), \(u\sim\mathcal N(0,I_{d-r})\), and \(z,u\) are
independent. The label depends only on the fixed-dimensional vector \(z\).
Since \(H_2(\bx)\) has \(D_2=\Theta(d^2)\) orthonormal coordinates, while
\(H_2(\bx)P_\star\) contains only \(O(d)\) coordinates, independence across
samples and Gaussian-chaos concentration give
\begin{equation}
\label{eq:quadratic-degree-two-noise-scales}
    \|\overline Q_n\|_F^2
       =\Theta_\mathbb P\!\left(\frac{D_2}{n}\right),
    \qquad
    \|\overline Q_nP_\star\|_F^2
       =O_\mathbb P\!\left(1+\frac d n\right).
\end{equation}
For completeness, these orders follow by expanding the centered second
moments:
\begin{align*}
    \E\|\overline Q_n-A_\star\|_F^2
      &=\frac1n\left(
          \E[y^2\|H_2(\bx)\|_F^2]-\|A_\star\|_F^2
        \right)=\Theta\!\left(\frac{D_2}{n}\right),\\
    \E\|(\overline Q_n-A_\star)P_\star\|_F^2
      &=\frac1n\left(
          \E[y^2\|H_2(\bx)P_\star\|_F^2]-\|A_\star\|_F^2
        \right)=O\!\left(\frac d n\right).
\end{align*}
The first expectation contains \(\Theta(d^2)\) orthonormal quadratic
coordinates, whereas the projection onto the fixed-rank space \(P_\star\)
retains only \(O(d)\) of them.

Combining \eqref{eq:quadratic-degree-two-noise-scales} gives
\[
    \frac{\|\overline Q_nP_\star\|_F^2}
         {\|\overline Q_n\|_F^2}
    =O_\mathbb P\!\left(\frac n{D_2}+\frac d{D_2}\right)
    =o_\mathbb P(1).
\]
The remainder bound \eqref{eq:krr-quadratic-remainder} is negligible relative
to \(a_{n,d}\|\overline Q_n\|_F\). Hence the same ratio holds with
\(\widehat Q_n\) in place of \(\overline Q_n\), and
\eqref{eq:quadratic-agop-overlap} proves the proposition.
\end{proof}

Proposition~\ref{prop:quadratic-agop-bottleneck} concerns the population
average in \eqref{eq:quadratic-krr-agop}: after \(\hat f_n\) is fitted, its
gradient outer product is averaged over a fresh draw
\(\bx\sim\mathcal N(0,I_d)\).  Thus the proposition removes the finite-sample
error in estimating the AGOP and shows that even this ideal average satisfies
\(\operatorname{Tr}(P_\star G(\hat f_n))/\operatorname{Tr}(G(\hat f_n))
=o_\mathbb P(1)\).

RFM instead computes an empirical average of these gradient outer products. If
it is evaluated on an independent sample \(\{\widetilde\bx_\nu\}_{\nu=1}^m\),
then
\[
    \widehat G_m^{\rm ind}(\hat f_n)
      :=\frac1m\sum_{\nu=1}^m
        \nabla\hat f_n(\widetilde\bx_\nu)
        \nabla\hat f_n(\widetilde\bx_\nu)^\top
\]
concentrates around \eqref{eq:quadratic-krr-agop} when \(m\gg d\), and hence
has the same asymptotic normalized trace projection. If the training inputs are
reused to compute the AGOP, \(\hat f_n\) and the empirical average are
dependent; controlling that case requires an additional concentration
argument and is not part of Proposition~\ref{prop:quadratic-agop-bottleneck}.

\subsubsection{Neural LoFi recovery with \texorpdfstring{\(O(d\log d)\)}{O(d log d)} samples}

We now show that, in contrast, the subspace \(\widehat V\) estimated by Neural
LoFi satisfies
\[
    \|\sin\Theta(\widehat V,V^\star)\|_{\mathrm{op}}
      \xrightarrow{\mathbb P}0
    \qquad\text{whenever}\qquad
    \frac{n}{d\log d}\longrightarrow\infty.
\]
For this target, its second-order estimator is
\[
    \widehat C_{\mathrm{LoFi}}
      :=\frac1n\sum_{\mu=1}^n y_\mu\bx_\mu\bx_\mu^\top.
\]
Since \(\E[y]=0\), \(\bx\bx^\top=I_d+\sqrt2\,H_2(\bx)\), and the
degree-two Hermite basis is orthonormal,
\begin{equation}
\label{eq:quadratic-lofi-population}
    C_{\mathrm{LoFi}}
      :=\E[y\bx\bx^\top]
      =\sqrt2\,\E[\langle A_\star,H_2(\bx)\rangle H_2(\bx)]
      =\sqrt2\,A_\star.
\end{equation}
Equivalently,
\(\widehat C_{\mathrm{LoFi}}=\sqrt2\,\overline Q_n+\bar y I_d\), so Neural
LoFi uses the same empirical quadratic coefficient \(\overline Q_n\) that
appears in the KRR expansion, but retains only its leading \(r\) eigenvectors
ordered by absolute eigenvalue and discards the remaining directions.

The relevant noise scale is therefore the operator norm rather than the
Frobenius norm.  A Gaussian-chaos matrix concentration bound, obtained for
example by truncating \(|y|\) and \(\|\bx\|_2\) and applying matrix Bernstein,
gives
\begin{equation}
\label{eq:quadratic-lofi-concentration}
    \|\widehat C_{\mathrm{LoFi}}-\sqrt2 A_\star\|_{\mathrm{op}}
    =O_\mathbb P\!\left(
       \sqrt{\frac{d\log d}{n}}+\frac{d\log d}{n}
      \right).
\end{equation}
Assume that the nonzero eigenvalues of \(B_\star\) are bounded away from zero,
and let \(\widehat V\) contain the \(r\) eigenvectors of
\(\widehat C_{\mathrm{LoFi}}\) with largest absolute eigenvalues.  The
Davis--Kahan theorem and \eqref{eq:quadratic-lofi-concentration} imply
\[
    \|\sin\Theta(\widehat V,V^\star)\|_{\mathrm{op}}
    \lesssim
    \frac{\|\widehat C_{\mathrm{LoFi}}-\sqrt2 A_\star\|_{\mathrm{op}}}
         {\min_j|\lambda_j(B_\star)|}.
\]
Consequently, for every fixed \(\eta>0\), there is a constant \(C_\eta\) such
that \(n\ge C_\eta d\log d\) implies
\[
    \|\sin\Theta(\widehat V,V^\star)\|_{\mathrm{op}}\le\eta
\]
with high probability.  If \(n/(d\log d)\to\infty\), this subspace error
converges to zero.  The gain comes from the rank-\(r\) spectral truncation: Neural
LoFi isolates the operator-norm spike even though the untruncated quadratic
coefficient still has Frobenius noise of order \(\sqrt{D_2/n}\).

\subsection{Implication for the hierarchical model}

The calculation above suggests the analogous obstruction for Deep RFM on the
hierarchical target of Appendix~\ref{app:th_rf}.  In the quadratic model, the
KRR predictor must estimate a \(D_2\)-dimensional quadratic coefficient
before
\(\operatorname{Tr}(P_\star G(\hat f_n))/\operatorname{Tr}(G(\hat f_n))\)
is nonvanishing. In the hierarchical model, the first target component carrying the
first-layer features \(h^{(1)}\) lies in the degree-\(2q\) Hermite subspace,
whose dimension is \(D_{2q}=\Theta(d^{2q})\).  We therefore expect the KRR fit used at
the first Deep RFM round, and hence its AGOP update, to remain uninformative
about \(h^{(1)}\) below this scale.  The following proposition states the KRR
part of this argument.

\begin{proposition}[Sub-threshold KRR prediction in the hierarchical Hermite model]
\label{prop:agop-hermite-bottleneck}
Consider the model of Appendix~\ref{app:th_rf} with first-layer degree \(q\)
and \(d_1=d^\epsilon\), \(\epsilon<q\). Let \(\hat f_n\) be a KRR or
random-feature predictor in the polynomial asymptotic regime of
\cite{ghorbani2020neural,mei2022generalization,hu2024asymptotics}. Fix
\(\delta>0\). If \(n=O(d^{2q-\delta})\), then with high probability
\[
    \mathbb E[\hat f_n(\bx)^2]=o_d(1),
    \qquad
    \mathbb E[y\hat f_n(\bx)]=o_d(1).
\]
Thus below the \(d^{2q}\) scale the predictor is \(L^2\)-vanishing and remains
at baseline test error.
\end{proposition}

\begin{proof}
For the target of Appendix~\ref{app:th_rf}, Lemma~4.2 of
\cite{tabanelli2026deep} and Lemma~2 of \cite{wang2023learning} give
\[
    \|\Pi_{\le 2q-1}\,y\|_{L^2(P_x)}=o_d(1).
\]
The results of \cite{ghorbani2020neural,mei2022generalization,hu2024asymptotics}
imply that, at scale \(n=O(d^{2q-\delta})\), the KRR/random-feature predictor
is asymptotically supported only on the Hermite components of \(y\) of total
degree below \(2q\). Since these components vanish, the fitted predictor
vanishes in \(L^2(P_x)\), which gives both claims.
\end{proof}

The quadratic proposition shows explicitly how a sub-threshold KRR fit produces
an uninformative AGOP. Proposition~\ref{prop:agop-hermite-bottleneck} gives the
corresponding \(d^{2q}\) prediction threshold for the hierarchical target. By
the same mechanism, we expect Deep RFM not to recover the first-layer latent
features \(h^{(1)}\) before this scale. A complete recursive theorem would also
require concentration of the dependent AGOP updates across rounds.

\begin{remark}[Neural LoFi recovery]
As shown in Theorem~\ref{thm:rf-hermite-bridge-k2}, for quadratic first-layer
features, if
\[
    n\gg d^2d_1,
    \qquad
    p_1\gg \sqrt{d_1}\,d^2,
\]
then
\[
    \frac{D_2}{c_2^2}
    A_{\RF}^{(1)}\widehat C_{\rm RF}^{(2)}A_{\RF}^{(1)\top}
    =\nu_1A^{(2)}+o_{\mathbb P}(d_1^{-1/2})
\]
in operator norm.  Hence the projected random-feature operator has the same
limiting eigenspaces as the Hermite estimator and recovers the \(d_1\) planted
directions.  Neural LoFi obtains these directions by diagonalizing the
label-weighted feature moment \(\widehat C_{\rm RF}^{(2)}\), whereas Deep RFM
forms its next feature metric from the gradients of a fitted KRR predictor.

More generally, for first-layer features of degree \(q\), the relevant Hermite
subspace has dimension \(D_q=\Theta(d^q)\).  The effective-dimension argument
of Appendix~\ref{app:effective-dimension-hermite} gives the analogous recovery
scale
\[
    n\gg D_qd_1=\Theta(d^{q+\epsilon}).
\]
Since \(\epsilon<q\), this is asymptotically smaller than the
\(d^{2q}\) KRR scale in Proposition~\ref{prop:agop-hermite-bottleneck}.
\end{remark}

\newpage

\section{Additional Numerical Explorations} \label{app:numerical-extra}

\subsection{Neural LoFi Kernel Limit} \label{app:kernel-limit}

\begin{figure}[H]
    \centering
    \hfill
    \includegraphics[width=0.32\linewidth]{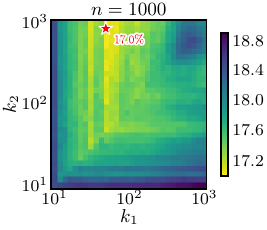}
    \hfill
    \includegraphics[width=0.32\linewidth]{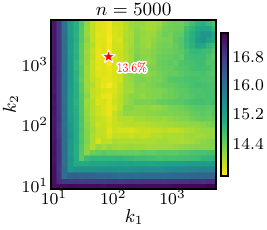}
    \hfill
    \includegraphics[width=0.32\linewidth]{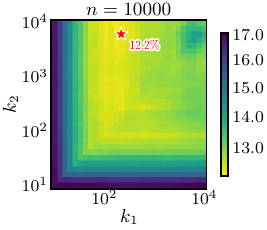}
    \caption{
        Test Error (\%) as a function of the kept features for Kernel LoFi for different training dataset sizes $n$.
        Stars indicate the best $(k_1,k_2)$ in each grid. The optimal retained dimension is finite and different from the usual kernel case (upper right corner), mirroring the behavior of the finite-width Neural LoFi heatmaps in~\Cref{fig:heatmaps}
    }
    \label{fig:kernel-heatmaps}
\end{figure}

The kernel formulation of Neural LoFi developed in
Appendix~\ref{app:kernel-nlofi} replaces the random projections $\bm R_\ell$ of
Algorithm~\ref{alg:neural-lofi} by their kernel expectations and computes the
generalized eigenvalue problem of Proposition~\ref{prop:generalized-eigenvector}
on the $n{\times}n$ Gram matrix. The resulting estimator is deterministic in the projection
randomness and depends on $(\bm x_{1:n}, y_{1:n})$ only through the kernel
$K_{\ell-1}$ and the labels. 

Figure~\ref{fig:kernel-heatmaps} reports the test error of kernel LoFi on
binary CIFAR-10 over the per-layer retained ranks $(k_1,k_2)$, at two
training-set sizes. 
First, the error
landscape is U-shaped in the same way as its finite-width counterpart in
\Cref{fig:heatmaps}: keeping too few features discards predictive signal,
keeping too many includes eigendirections that are not separated from the
noise bulk at this sample size, and the optimum sits at intermediate
$(k_1,k_2)$. 
Second, the location of this optimum shifts toward larger
ranks as $n$ grows. 
Both effects are those predicted by the
relevance--complexity criterion of Theorem~\ref{thm:variational-rkhs} and
the emergence threshold of \cref{eq:lofi-emergence-threshold}: more
samples lower the noise floor $\tau^k_\ell(n)$, allowing additional
low-degree directions to cross it.

Read together with the convergence experiment of
Figure~\ref{fig:nlofi-convergence}, in which the finite-width LoFi error
approaches the kernel-LoFi error from above as the projection width $p$
grows, this establishes that the U-shape and the sample-dependent rank
optimum are properties of the supervised spectral mechanism itself rather
than artifacts of the random projections used in practice. 
This is also
what distinguishes kernel LoFi from standard fixed-kernel methods such as
the NNGP or NTK: the operator diagonalized at each layer is built from the
labels, so the geometry in which the next layer searches for signal changes
with the task. 
Kernel LoFi is, in this sense, the simplest deterministic
object that captures the layerwise task-adaptivity of feature learning
described in the main text.

The Neural LoFi kernel also gives a feature-space perspective on why overparameterization and pruning are not contradictory. Classical pruning methods show that many weights or connections can be removed after training with little loss in performance
\cite{lecun1989optimal,han2015learning}, while the lottery-ticket hypothesis suggests that large dense networks may contain smaller trainable subnetworks
\cite{frankle2019lottery}. In Neural LoFi, width and rank play complementary roles: large width creates a rich feature space in which task-correlated directions can be discovered, while spectral pruning retains only the directions that finite data can reliably support. Thus large networks are useful for discovery, but low-rank feature selection controls the effective dimension used for prediction.

\subsection{\texorpdfstring{{Measuring the feature alignment of LoFi with GD}}{Measuring the feature alignment of LoFi with GD}} \label{app:gd-feature-alignment}

The comparison in Fig.~\ref{fig:gradient-alignment} is made at the level of
subspaces, not individual hidden coordinates. This is the natural invariant:
hidden coordinates in both GD and Neural LoFi can be rotated, permuted, rescaled,
or sign-flipped without changing the represented feature span.

Let \(\mathcal D_{\mathrm{test}}=\{\bx_\mu\}_{\mu=1}^m\) be the held-out set on
which the comparison is evaluated. For Neural LoFi, we use the filtered
coordinates obtained immediately after the spectral step and before the random
feature lift. Thus, for layer \(\ell\),
\[
    \left[Z_{\ell}^{\mathrm{LoFi}}\right]_{\mu,:}
    =
    g_{\ell}^{\mathrm{LoFi}}(\bx_\mu)
    =
    \bz_{\ell-1}^{\mathrm{LoFi}}(\bx_\mu)
    V_{\ell,k_\ell}^{\mathrm{LoFi}},
    \qquad
    Z_{\ell}^{\mathrm{LoFi}}\in\mathbb R^{m\times k_\ell},
\]
where \(V_{\ell,k_\ell}^{\mathrm{LoFi}}\) contains the \(k_\ell\) retained
LoFi spectral directions. The lifted representation
\[
    \bz_\ell^{\mathrm{LoFi}}
    =
    p_\ell^{-1/2}\sigma(R_\ell g_\ell^{\mathrm{LoFi}})
\]
is not used in the overlap.

For the GD network at step \(t\), write the layer-\(\ell\) weight matrix as
\[
    W_{\ell}^{\mathrm{GD}}(t)
    =
    U_{\ell}(t)\Sigma_{\ell}(t)V_{\ell}(t)^\top ,
\]
and let \(V_{\ell,r}^{\mathrm{GD}}(t)\) be the matrix of the top \(r\) right
singular vectors. We run the GD network normally up to layer \(\ell-1\), and
then project the current representation onto these singular directions:
\[
    \left[Z_{\ell,r}^{\mathrm{GD\text{-}SVD}}(t)\right]_{\mu,:}
    =
    \bz_{\ell-1}^{\mathrm{GD}}(\bx_\mu,t)
    V_{\ell,r}^{\mathrm{GD}}(t),
    \qquad
    Z_{\ell,r}^{\mathrm{GD\text{-}SVD}}(t)\in\mathbb R^{m\times r}.
\]
We apply neither the layer-\(\ell\) nonlinearity nor any later GD layer or
classifier. Hence both sides of the comparison are spectral coordinates exposed
by the target layer before the next nonlinear expansion.

Let \(H=I_m-\frac1m\mathbf 1\mathbf 1^\top\) be the centering matrix, and let
\(\operatorname{orth}(A)\) denote an orthonormal basis for the column span of
\(A\), obtained from the nonzero left singular vectors. Define
\[
    Q_{\ell}^{\mathrm{LoFi}}
    =
    \operatorname{orth}\!\left(HZ_{\ell}^{\mathrm{LoFi}}\right),
    \qquad
    Q_{\ell,r}^{\mathrm{GD}}(t)
    =
    \operatorname{orth}\!\left(HZ_{\ell,r}^{\mathrm{GD\text{-}SVD}}(t)\right).
\]
The directed recovery score plotted in Fig.~\ref{fig:gradient-alignment} is
\begin{equation}
    \mathrm{Recovery}_{\ell,r}(t)
    =
    \frac{
    \left\|
    \left(Q_{\ell}^{\mathrm{LoFi}}\right)^\top
    Q_{\ell,r}^{\mathrm{GD}}(t)
    \right\|_F^2
    }{
    \operatorname{rank}\!\left(HZ_{\ell,r}^{\mathrm{GD\text{-}SVD}}(t)\right)
    }.
    \label{eq:gd-lofi-subspace-recovery}
\end{equation}
Equivalently, this is the average squared length of an orthonormal GD-SVD
feature direction after projection onto the LoFi filtered-feature span. The
score lies in \([0,1]\), equals \(1\) exactly when the centered GD-SVD span is
contained in the centered LoFi span, and is invariant to any invertible
reparameterization of the columns within either feature matrix. The companion
``precision'' score, useful when the LoFi span is substantially wider, uses the
same numerator but divides by
\(\operatorname{rank}(HZ_{\ell}^{\mathrm{LoFi}})\).

In Fig.~\ref{fig:gradient-alignment}, we set \(r=8\) and report this score for
each hidden layer of a four-layer fully connected network on the binary
CIFAR-10 animal-vs.-vehicle task. The plotted curve is the normalized gain
\[
    \frac{
        \mathrm{Recovery}_{\ell,r}(t)-\mathrm{Recovery}_{\ell,r}(0)
    }{
        1-\mathrm{Recovery}_{\ell,r}(0)
    }.
\]
In the layerwise GD experiment, only one hidden layer is trained at a time; the
final readout-training phase is omitted from the displayed time interval.

\subsection{Generalization and Spectral Performance: GD vs. LoFi} 
\label{app:gd-spectral-performance-comp}
In this experiment we aim to train networks for image binary classification (Animals vs. Vehicles in CIFAR-10) using both Neural LoFi and Gradient Descent (GD), and compare their generalization performance across a range of training set sizes. 
We investigate the generalization properties of Neural LoFi relative to GD by evaluating their test performance across varying training set sizes $n \in [10^2, 5 \times 10^4]$. To ensure a controlled comparison, we fix the architectures for both FC and CNN models to have a comparable number of parameters and identical random projection dimensions.

For GD, we employ a compute-constrained training protocol where the total number of gradient steps is held constant across dataset sizes. We adopt an adaptive scaling law for the learning rate, $\eta \propto \sqrt{B/n}$, where $B$ is the batch size, and utilize a cosine annealing schedule with a linear warmup. This ensures stable convergence dynamics without the need for exhaustive per-configuration hyperparameter tuning. In contrast, LoFi remains a one-pass algorithm requiring only the selection of the eigenvector bottleneck $k_\ell$.

Figure~\ref{fig:gd-lofi-comparison} illustrates the test error as a function of the training set size. We observe that LoFi achieves generalization performance comparable to, and in low-sample regimes occasionally superior to, GD. By reporting GD performance at varying step intervals, we demonstrate that LoFi, despite being a one-pass spectral method, captures a feature set equivalent to that learned by GD during its early-to-mid training stages. 

While GD eventually outperforms LoFi as the training budget increases, this gap is expected; GD iteratively refines features across the full model capacity, whereas LoFi performs a sequential, layer-wise spectral projection. Crucially, these results are obtained without auxiliary regularization (e.g., dropout, early stopping) or optimized filtering levels for LoFi. Thus, the results in Figure~\ref{fig:gd-lofi-comparison} suggest that the spectral alignment of LoFi is not merely a theoretical property but a robust driver of generalization in deep architectures.


\begin{figure}[t]
    \centering
    \includegraphics[width=\linewidth]{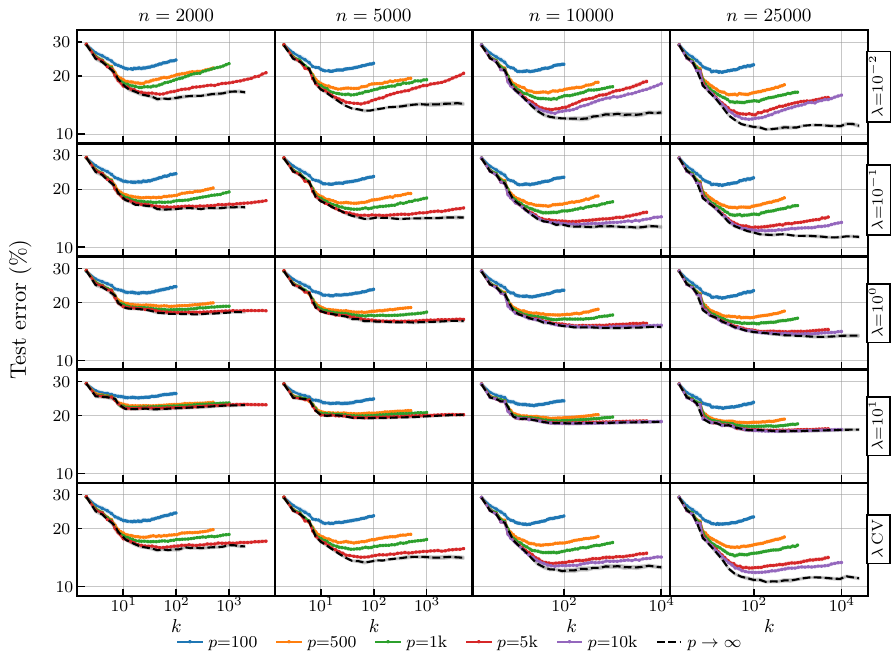}
    \caption{
        Test error (\%) as a function of the number of retained features $k$ at
        the hidden layer of a one-hidden-layer Neural LoFi with ReLU activation
        on CIFAR-10. Each panel fixes the ridge regularization $\lambda$ (rows,
        from $10^{-2}$ to $10^{1}$; the bottom row tunes $\lambda$ by
        cross-validation) and the training-set size $n$ (columns), and overlays
        curves for a range of random-projection widths $p$, including the
        infinite-width limit $p\!\to\!\infty$ (dashed). 
        Only the
        intermediate representation is filtered; the input and output layers are not reduced. 
    }
    \label{fig:filtering-one-hidden-layer}
\end{figure}

\begin{figure}
    \centering
    \includegraphics[width=\linewidth]{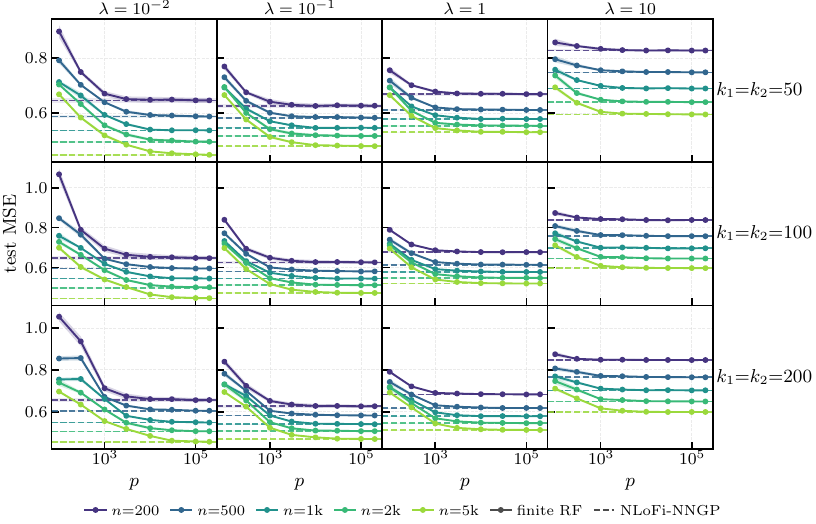}
    \caption{%
    Convergence of finite-width NLoFi to the NLoFi-NNGP limit as the hidden width $p$ grows.  Each panel shows the test MSE versus $p$ for four training-set sizes $n$. 
    Solid curves are finite-RF NLoFi; dashed horizontal lines mark the corresponding NLoFi-NNGP limits at the same $n$.  Rows vary the number of retained eigenfeatures per layer ($k_1 = k_2 \in \{50, 100, 200\}$); columns vary the ridge regularisation $\lambda$.
    Across all settings the finite-width curves converge from above to the kernel limit.%
    }  
    \label{fig:nlofi-convergence}
\end{figure}

\subsection{Filtering of Features}

\Cref{fig:filtering-one-hidden-layer} reports the test error as a function of
the number of retained features $k$ for a single hidden layer, where only the
intermediate representation is subject to Neural LoFi's feature selection while
the input and output layers are left untouched. This is the one-dimensional
slice of the two-dimensional heatmaps of \Cref{fig:heatmaps}, isolating the
role of $k$.

Within any single panel every curve is U-shaped: retaining too few features
underfits the target (high bias), while retaining too many carries over the
noise of the unfiltered random features (high variance). The minimum defines an
optimal rank $k^\star$, so Neural LoFi's filtering realizes a
bias-variance trade-off, here controlled by an explicit number of retained
directions rather than by the ridge penalty $\lambda$ alone. The trend is
robust across the fixed-$\lambda$ rows and the cross-validated one.

Reading along a row (increasing the number of random projections $p$) lowers the
error at the optimum and moves $k^\star$ to larger values: a wider random-feature
bank spans a richer function class, so more of its directions are worth keeping
and progressively less has to be cut away. The $p\!\to\!\infty$ limit (dashed) is the lower envelope of the finite-$p$ family. Thus $p$ should
be taken as large as possible, yet filtering remains beneficial at every $p$.

Reading down a column (increasing the sample size $n$) has the same qualitative
effect through a different mechanism: with more data the label-aware covariance
estimates better features, so again $k^\star$ increases and the error decreases
as there is less overfitting to prune.

Together these trends suggest a design point: build the widest feature bank
you can afford and then recover generalization by projecting onto the
task-relevant low-dimensional subspace. 
This is the same rationale behind
network pruning
\cite{frankle2019lottery,li2017pruning} and behind rainbow networks \cite{guth2024rainbow}, in which
each layer is a random feature map composed with a learned low-rank rotation that
aligns features to the task. 
Neural LoFi makes that alignment explicit and
label-aware, and \Cref{fig:filtering-one-hidden-layer} quantifies how much of the
width can be pruned as a function of $p$ and $n$.

\subsection{Predicting the sample complexity for feature recovery}
\label{app:feature-recovery}


The variational analysis of Section~\ref{sec:emergence-snr} predicts
that the $k$-th eigenvector of the layer-$\ell$ signed covariance
becomes recoverable once the population correlation $\rho_\ell^{(k)}$
exceeds the empirical noise floor $\tau_\ell^k(n)$, with the latter
controlled by the residual effective dimension of the kernel induced by
the previous layer.  This is an asymptotic statement, and the
quantitative prediction
Eq.~\eqref{eq:real-data-emergence-estimate} rests on assumptions that
CIFAR-10 features are not designed to satisfy.  The experiment in
Fig.~\ref{fig:sample-complexity-pred} (and its convolutional analog
Fig.~\ref{fig:sample-complexity-pred-conv}) is intended as a stress test
of the prediction in that regime.

We use the full binary CIFAR-10 training set ($N=60{,}000$, animals
vs.\ vehicles) as a population proxy: running Neural LoFi on this set
yields stable reference eigenvectors $\{\hat{\bm v}_i^{(N)}\}_{i\ge 1}$
at the layer of interest.  We then refit the same pipeline on random
subsets of size $n\le N$, holding the random projections $\bm W_1,\bm
W_2$ fixed across subsets, and record the index-aligned overlap
$|\langle \hat{\bm v}_i^{(n)},\hat{\bm v}_i^{(N)}\rangle|^2$, averaged
over $100$ independent dataset permutations.  In the convolutional
figure we plot the eight overlaps after sorting them descending within
each draw, which is robust to the $|\lambda|$-ordering swaps that arise
when two leading eigenvalues of opposite sign have similar magnitude.  The predicted thresholds $n_\ell^k$
are obtained from the proxy at $n=N$ by the recursive procedure of
App.~\ref{app:effective-dimension} and use no information from the
smaller-$n$ runs.

The predicted thresholds (dashed verticals) track the order and
approximate scale of the empirical transitions across the eight leading
eigenvectors, which themselves are sharp on a logarithmic scale and
saturate in the order predicted by the recursive effective-dimension
recipe — consistent with the BBP/EJ picture of
Section~\ref{sec:emergence-snr}.  The bound is qualitative and not
tight, so the prediction should be read as the correct scaling of
$n^k_\ell$ with the residual effective dimension and the spectral gap
rather than as an absolute threshold; even at this level, however, a
single eigendecomposition on a reference set is enough to indicate the
order of magnitude at which each eigendirection becomes extractable.

\begin{figure}[t]
    \centering
    \includegraphics[width=\linewidth]{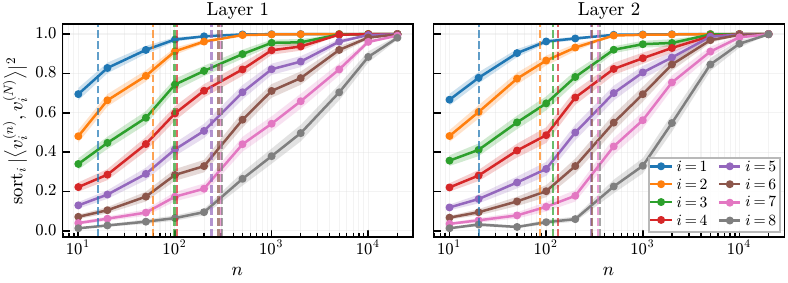}
    \caption{
    \textbf{Predicting when individual features emerge on CIFAR-10 with convolutional networks.} Convolutional analog of Fig.~\ref{fig:sample-complexity-pred} in the main text.
    For a CNN Neural LoFi model on the CIFAR-10 animal-vs.-vehicle task (random feature matrices $\bm W_1, \bm W_2$ held fixed across subsamples), we track the squared overlap $|\langle v_i^{(n)},v_i^{(N)}\rangle|^2$ between eigenvectors estimated from $n$ samples and the large-sample reference eigenvectors at layer $1$ (Left) and $2$ (Right).  At each sample size we order the eight tracked overlaps from largest to smallest within each draw, so $i=1$ is the best-aligned axis at that $n$ and so on; curves show mean $\pm$ SEM over $100$ random subsamples.  Dashed verticals are the predicted emergence thresholds $n_\ell^k$ from \eqref{eq:lofi-emergence-threshold},\eqref{eq:emergence_feature},
    sorted ascending and color-matched to the curves (see
    Eq.~\eqref{eq:real-data-emergence-estimate} in the Appendix).   The sharp rise of each curve near its predicted threshold shows that the effective-dimension criterion predicts the order and approximate scale at which successive task-relevant directions become learnable.
    }
    \label{fig:sample-complexity-pred-conv}
\end{figure}

\subsection{Spectrum of the signed covariance}
The core of LoFi is the spectral decomposition of the signed covariance operator, thus it is crucial to understand the structure of its spectrum. 

We analyze the eigenvalue spectra learned by ridge spectral networks on a binary CIFAR-10 classification task (animal vs. vehicle). The architecture comprises two convolutional layers with $p=512$ channels, each followed by $2\times 2$ max pooling and $L_2$ normalization of the features—concluding with a fully-connected layer of $p=512$ units. All layers are trained using signed-covariance eigendecomposition with eigenvalue-based feature selection (ridge spectral training). We evaluate the model across a range of training set sizes ($n \in \{200, 500, 2000, 50000\}$) while maintaining a fixed network architecture. The grid plots in \Cref{fig:eigenvalues-histogram} illustrate the full eigenvalue spectrum for each layer, utilizing a symlog $x$-axis to resolve both fine-grained and large-magnitude spectral components. The top five eigenvalues per layer are highlighted with red triangles, revealing how the spectral structure evolves with data scale and illustrating the relative importance of dominant versus distributed features across the network depth.
\begin{figure}
    \centering
    \includegraphics[width=0.94\linewidth]{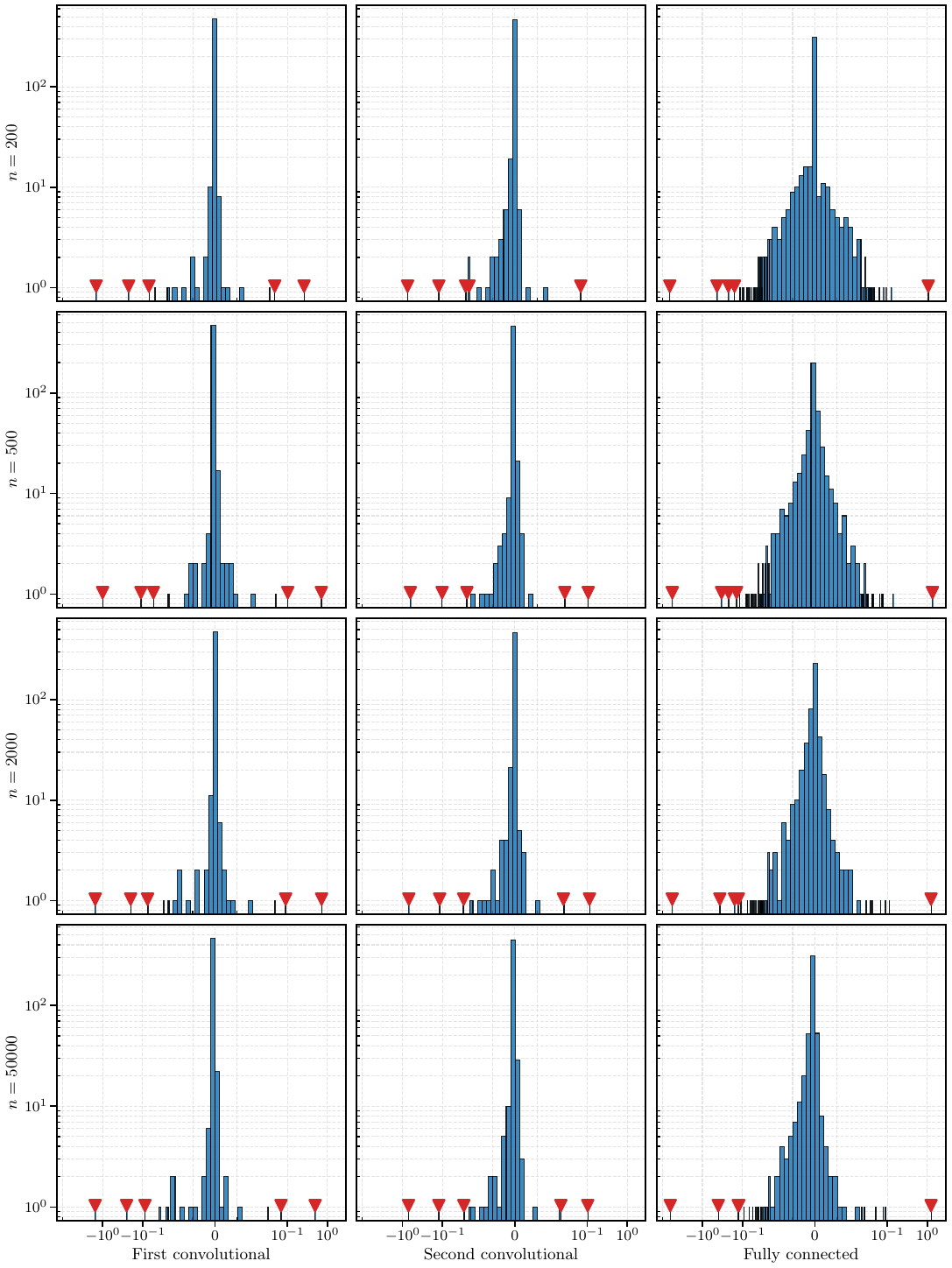}
    \caption{\textbf{Spectral Distribution:} Histograms of eigenvalues across network layers (columns) and dataset sizes (rows). Red markers indicate the five most dominant eigenvalues. The symlog scale reveals the emergence of spectral structure and the separation of lead features from the bulk distribution as $n$ grows. The random feature dimension in this experiment is $p=512$ for all layers.}
    \label{fig:eigenvalues-histogram}
\end{figure}

\subsection{Feature Visualization}

\begin{figure}
    \centering
    \includegraphics[width=\linewidth]{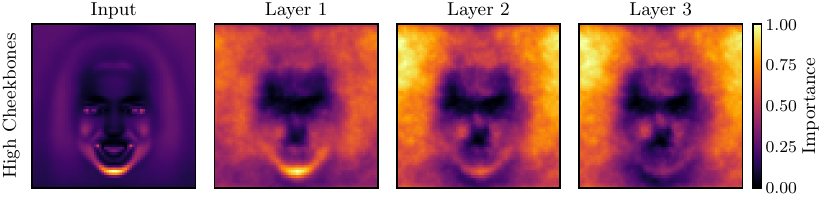}\\
    \includegraphics[width=\linewidth]{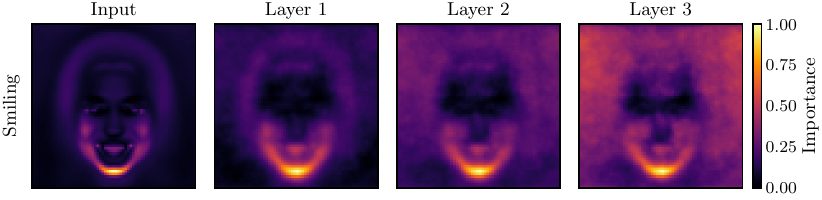}
    \caption{
    Layer-wise feature importance $\bar{I}^{(\ell)}$ (\cref{eq:importance-layer}) on CelebA \cite{liu2015faceattributes} for two binary 
    attributes. \textit{(Top, High Cheekbones)} Importance concentrates 
    progressively on the cheekbone region across layers, while the chin area, 
    salient at the input, is progressively suppressed in deeper layers. 
    \textit{(Bottom, Smiling)} Importance remains focused on the mouth and jaw 
    region throughout all layers, reflecting that the discriminative signal for 
    this attribute is preserved and not discarded by Neural LoFi's feature 
    selection. This contrast illustrates that the retained features are 
    task-dependent and geometrically meaningful.
    }
    \label{fig:celeba-feature-importance}
\end{figure}

\paragraph{Input importance} To probe which input pixels drive the activation along $v_k^{(\ell)}$, we
back-propagate through the random features and accumulate the absolute
Jacobian--vector product over the training set,
\begin{equation}
  I_k^{(\ell)}(d) \;=\; \frac{1}{N}\sum_{i=1}^{N}
  \left| \frac{\partial}{\partial x_{i,d}} \!\left( v_k^{(\ell)\!\top} \phi_\ell(x_i) \right) \right|
  \;=\; \frac{1}{N}\sum_{i=1}^{N} \left| J_\ell(x_i)^{\!\top} v_k^{(\ell)} \right|_d,
  \label{eq:importance_per_k}
\end{equation}
where $J_\ell(x) = \partial \phi_\ell(x)/\partial x \in \mathbb{R}^{P\times D}$.
The raw map $I_k^{(\ell)}$ is convolved with an isotropic Fourier low-pass
filter $\smash{m(f) = (1 + \|f\|/f_0)^{-\alpha}}$ to suppress
high-frequency noise inherited from the random projections. For the input layer ($\ell = 0$) the Jacobian collapses to the identity and~\eqref{eq:importance_per_k} reduces to
$I_k^{(0)}(d) = \bigl(v_k^{(0)}\bigr)_d^{2}$.

Different eigenvectors carry very different amounts of label signal. Rather than
treating them uniformly, we aggregate the $K_\ell$ retained per-$k$ importance
maps with weights given by the absolute signed-covariance eigenvalues,
\begin{equation}\label{eq:importance-layer}
  \bar{I}^{(\ell)}(d) \;=\;\frac{\sum_{k=0}^{K_\ell-1} |\lambda_k^{(\ell)}|\, I_k^{(\ell)}(d)}{\sum_{k=0}^{K_\ell-1} |\lambda_k^{(\ell)}|}.
\end{equation}
This emphasises directions that most strongly co-vary with the target while
preserving the spatial information carried by the lower-ranked, but still
label-aligned, eigenvectors.

In~\Cref{fig:celeba-feature-importance} each panel shows $\bar{I}^{(\ell)}$ reshaped to the $64\!\times\!64$ image grid and averaged over the three colour channels. We min--max rescale each panel independently to $[0,1]$. 
The shared colorbar reports this per-panel normalised intensity. The vertical label on the leftmost panel names the binary CelebA attribute used as $y$.

\paragraph{Filter and activations for CNNs} In Section~\ref{sec:real-data} we have already discussed the filters and activations learned by LoFi for CNNs. Here we provide additional visualizations of these filters and activations, for different layers and training set sizes.
\begin{figure}
    \centering
    \begin{minipage}{0.3\linewidth}
    \includegraphics[width=0.8\linewidth]{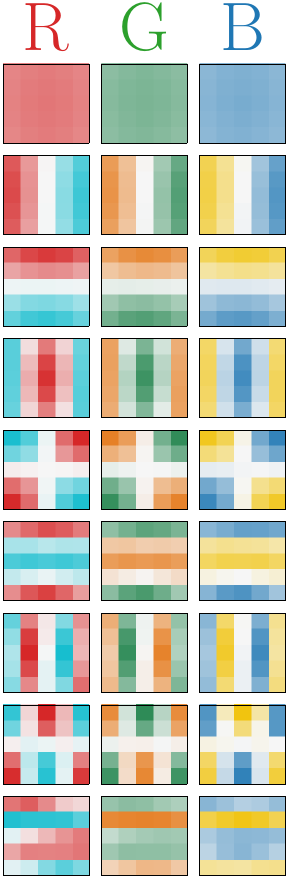}
    \end{minipage}
    \hfill
    \begin{minipage}{0.6\linewidth}
        \includegraphics[width=0.8\linewidth]{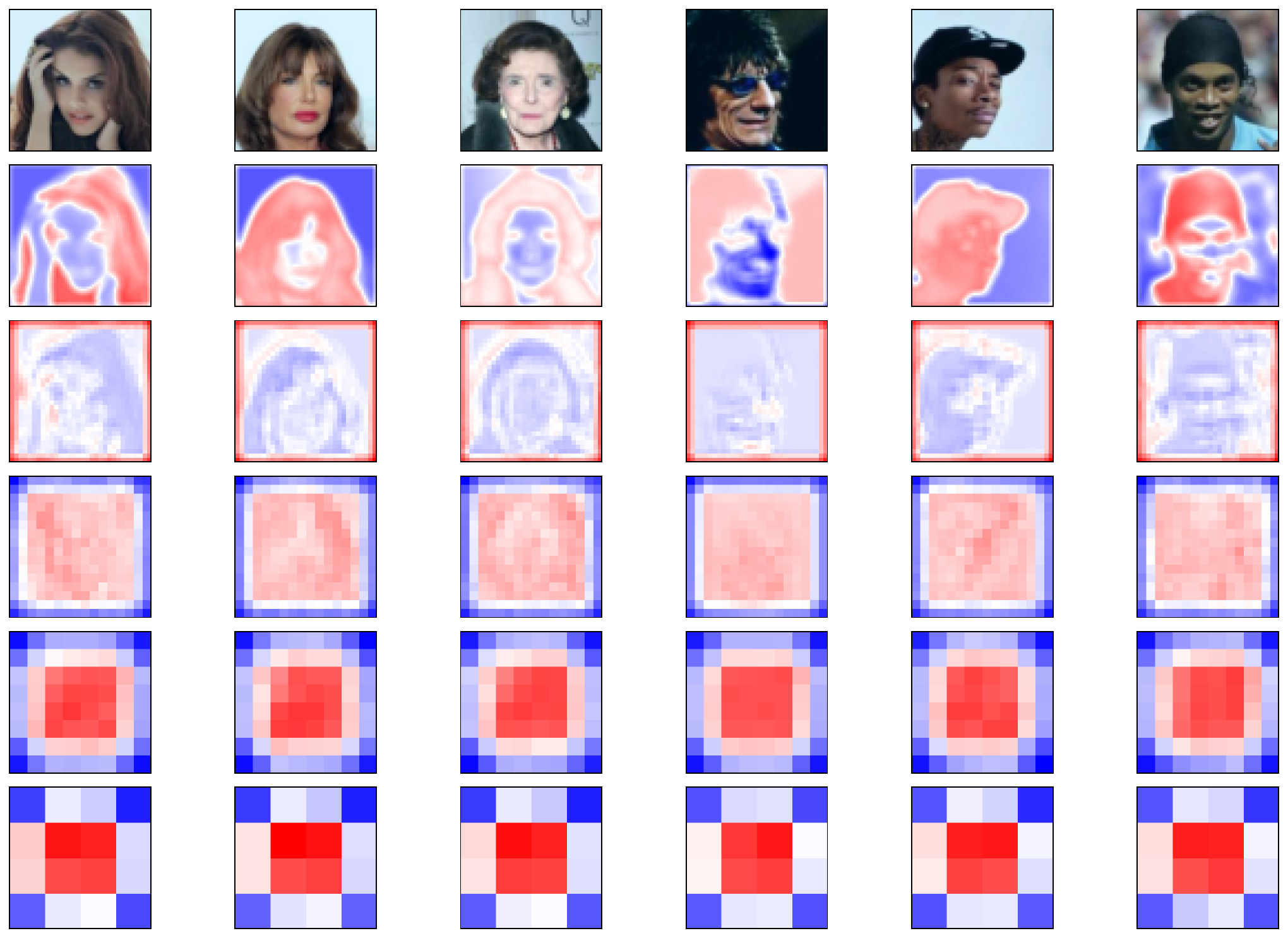}
        \begin{center}activation of top eigenvector \end{center}
        \vspace{.3cm}

        \includegraphics[width=0.85\linewidth]{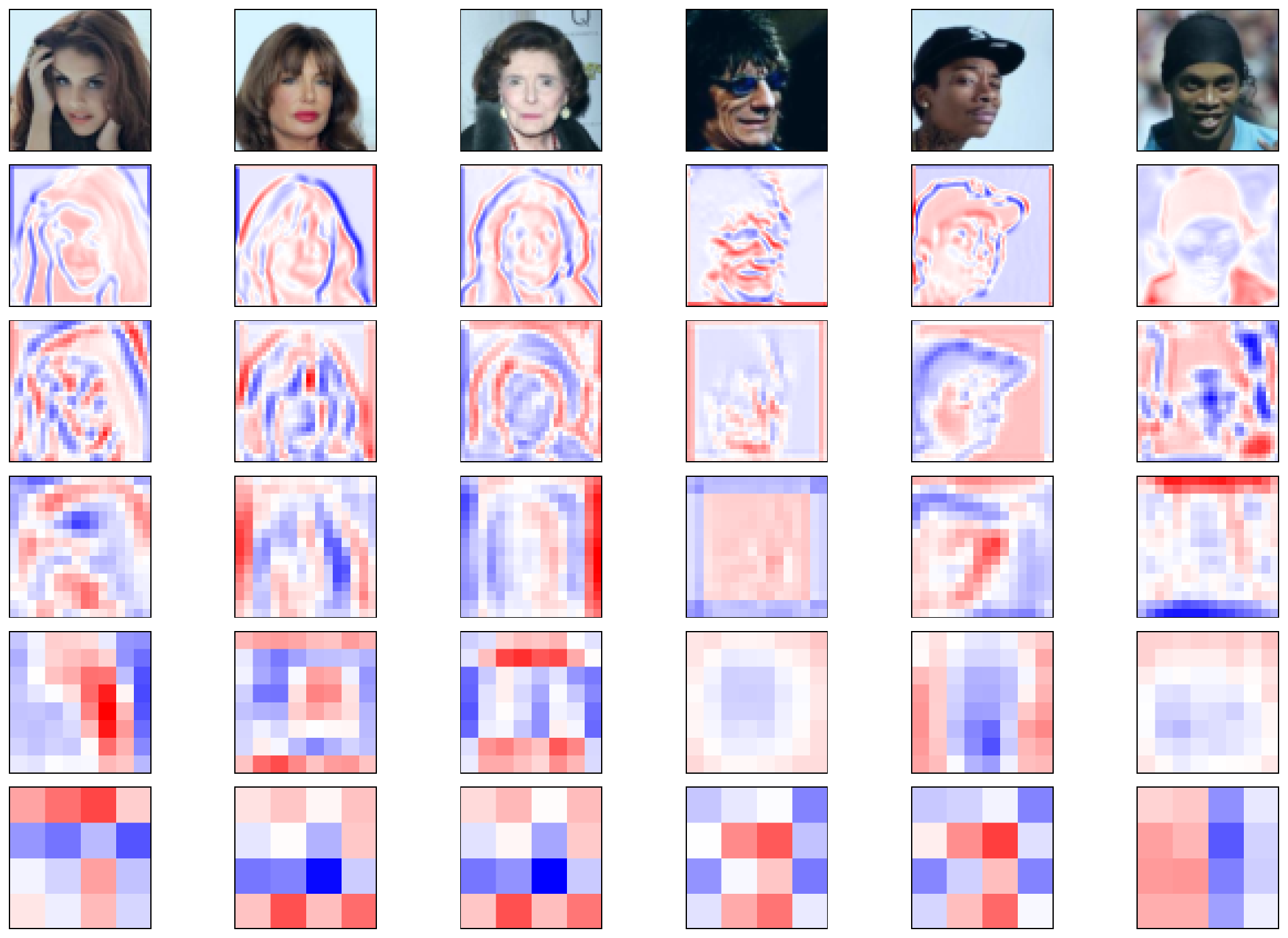}
        \begin{center}activation of 7th eigenvector \end{center}
    \end{minipage}
    \caption{\textbf{Filters and activations for CNNs.} We train a 6 convolutional + 1 fully connected layer neural network on CelebA \cite{liu2015faceattributes} for binary classification of the "Gender" attribute, using Neural LoFi with signed-covariance eigendecomposition and eigenvalue-based feature selection. \textit{(Left)} The $5\times5$ first-layer filters learned by Neural LoFi, visualized as RGB images. \textit{(Right)} The activations of the top-ranked (1st) and mid-ranked (7th) eigenvectors at each convolutional layer (lower rows are deeper), for 6 images of the test set.}
    \label{fig:additional-cnn-feature-visualizations}
\end{figure}
In Figure~\ref{fig:additional-cnn-feature-visualizations} we show the first-layer filters learned by LoFi, as well as the activations of the top-ranked and mid-ranked eigenvectors at each convolutional layer. Despite being larger than the filters learned on CIFAR-10 (Figure~\ref{fig:cnn-filter-activations}), the filters learned on CelebA are still interpretable as edge detectors, with a variety of orientations and frequencies. It is interesting to compare the features learned by different eigenvectors: the top eigenvector tends to be a brightness filter and converges to the same representation when going deeper. The mid-ranked eigenvector, instead, learns more complex features even at the first layer, obtaining a self-evident representation of the sample after the last convolutional layer.

\newpage 
\section{Code Implementation}\label{app:code-implem}

All experiments are carried out with a self-contained Python package built on PyTorch \cite{pytorch}, Numpy \cite{numpy}, Scipy \cite{scipy} and scikit-learn \cite{scikit-learn}. The source code is available in the attached zip file.

For each layer $\ell = 1, \ldots, L$, the trainer
(i) estimates the RMS norm
$c_\ell = \bigl(\tfrac{1}{n}\sum_i \|\mathbf{h}_{\ell-1}^{(i)}\|^2\bigr)^{1/2}$
to keep pre-activations $O(1)$;
(ii) draws a frozen random map
$\mathbf{W}_\ell \in \mathbb{R}^{d_{\ell-1}\times p_\ell}$,
$W_{\ell,ij}\sim\mathcal{N}(0,1)$,
and forms
$\mathbf{Z}_\ell = p_\ell^{-1/2}\,
\sigma_\ell\!\bigl(\mathbf{H}_{\ell-1}\mathbf{W}_\ell / c_\ell\bigr)$;
(iii) computes the top-$k_\ell$ eigenvectors $\mathbf{V}_\ell$ of the signed
covariance
$\mathbf{C}_\ell = \tfrac{1}{n}\mathbf{Z}_\ell^\top
\operatorname{diag}(\mathbf{y})\,\mathbf{Z}_\ell$
ranked by $|\lambda|$; and
(iv) sets $\mathbf{H}_\ell = \mathbf{Z}_\ell\mathbf{V}_\ell$.

The eigendecomposition of $\mathbf{C}_\ell$ is the dominant cost, and three
paths are dispatched automatically based on $(n,p_\ell,k_\ell)$: for
$p_\ell\le 25{,}000$ on GPU, $\mathbf{C}_\ell$ is formed explicitly and
diagonalised with \texttt{torch.linalg.eigh}; for moderate $p_\ell$ on CPU,
the full LAPACK eigensolver is used when $k_\ell\ge p_\ell/4$ and Lanczos \\
(\texttt{scipy.sparse.linalg.eigsh}) otherwise; for $p_\ell > 50{,}000$ and
$k_\ell\le p_\ell/2$, a \texttt{LinearOperator} is passed to \texttt{eigsh}
so that $\mathbf{C}_\ell$ is never materialised, requiring only $O(np_\ell)$
memory and two passes per Lanczos iteration. A deterministic starting vector
$\mathbf{v}_0=\mathbf{1}/\!\sqrt{p_\ell}$ ensures reproducibility, and when
$\mathbf{Z}_\ell$ does not fit in GPU memory, features are streamed in batches
and $\mathbf{C}_\ell$ is accumulated via outer products before calling
\texttt{eigsh}.

The final representation $\mathbf{H}_L$ is passed to
\texttt{sklearn.linear\_model.RidgeCV}; classification accuracy is obtained by thresholding at zero (targets $\pm 1$). 

\subsection{Figure Reproducibility}\label{app:hyper-params}

We list below the hyperparameters used for each figure (panel by panel).
Unless stated otherwise, random-feature weights are Gaussian, biases are zero,
the final readout is ridge regression, and curves are averaged over multiple
seeds.

\begin{description}[leftmargin=1.5em]

  \item[\Cref{fig:heatmaps}, left.]
    Binary CIFAR-10 (animals vs.\ vehicles), $L=3$ ReLU LoFi layers.
    Test error vs.\ $n$ for ridge, 3-layer random features, and LoFi with
    $p\in\{5000, 10000\}$; LoFi ranks $(k_1,k_2)$ optimally found after a logarithmic grid search in $[2, p-1]^2$. Ridge $\lambda$ optimally tuned with \texttt{RidgeCV} from \cite{scikit-learn} in $[10^{-6}, 10^{6}]$ with $500$ points; 10 seeds.

  \item[\Cref{fig:heatmaps}, right.]
    Same setup. Test error grid over $(k_1,k_2)$ at fixed $p=5000$, $k_3=p$, for $n\in\{5000,10000,20000,50000\}$. Stars mark the best $(k_1,k_2)$ in each grid. 10 seeds.

  \item[\Cref{fig:gradient-alignment}.]
    Four-layer FC ReLU net on binary CIFAR-10; width $p=1000$, LoFi ranks
    $k=20$, $n=50000$.  The plotted score is
    Eq.~\eqref{eq:gd-lofi-subspace-recovery} using the filtered LoFi features
    and the top $r=8$ GD-SVD features at each layer, converted to the normalized
    gain \((O_t-O_{\mathrm{init}})/(1-O_{\mathrm{init}})\).  GD is trained
    layerwise for 200 steps per hidden layer, and the final readout-training
    phase is omitted from the displayed 800-step window.

  \item[\Cref{fig:cnn-filter-activations}, left.]
    Binary CIFAR-10 animal-vehicle. Convolutional LoFi: 4 conv layers ($3{\times}3$,
    $2{\times}2$ pool, $L_2$ norm) + 1 FC layer; channels $p=4096$ , ranks
    32,64,128,256, ReLU; $n=50000$. Top first-layer filters from the three RGB
    channels.

  \item[\Cref{fig:cnn-filter-activations}, right.]
    Same architecture. Activation maps of the 6th LoFi feature on test images
    at successive depths.

\item[\Cref{fig:rf_theoretical}, left, center.]
    Hierarchical teacher of App.~\ref{app:th_rf} with $k=2$, $\epsilon=\tfrac12$ and $g^\star=\tanh$; $d\in\{80,100,120,140\}$. RF widths $(p_1,p_2)\in\{(20000,512),(30000,768),(40000,1024),(50000,1280)\}$, spherical weights, and activations as in App.~\ref{app:rf-estimator}. The top eigenvectors are computed with a power iteration using at most $15$ iterations and oversampling parameter $10$. Readout: polynomial kernel of maximal degree $5$, ridge regularization $10^{-6}$, kernel regularization $10^{-4}$. Left: test MSE vs.\ $\alpha=\log(n)/\log(d)$. Center: overlap between $\widehat h^{(1)}$ and $h^{(1)}$ vs.\ $\alpha$. $10$ seeds.

\item[\Cref{fig:rf_theoretical}, right.]
    Same hierarchical teacher with $q=2$, $d=100$, $\epsilon=\tfrac12$, $\alpha=3$, and $g^\star=\tanh$, using a first-layer random-feature width $p_1=10000$ and the activations of App.~\ref{app:rf-estimator}. The displayed spectrum is computed with a randomized eigensolver.
    One seed.
    
\item[\Cref{fig:rf_spectrum}]
  Same hierarchical teacher as in \Cref{fig:rf_theoretical}, right, with $q=2$, $d=100$, $\epsilon=\tfrac12$, and $g^\star=\tanh$, using a first-layer random-feature width $p_1=10000$ and the activations of App.~\ref{app:rf-estimator}. The figure shows the spectrum of the first-layer spectral operator for several values of $\alpha \in \{1.5, 2.0, 2.5, 3.0, 3.5\}$, each computed with a randomized eigensolver. One seed per value of $\alpha$.

  \item[\Cref{fig:kernel-heatmaps}]
    Kernel LoFi (App.~\ref{app:kernel-nlofi}) on binary CIFAR-10 (animals vs.\ vehicles), $L=3$ layers, ReLU activation. Test error grid over $(k_1,k_2)\in[10,n-1]^2$ for $n\in\{1000,5000\}$; stars mark the best $(k_1,k_2)$. 10 seeds. The final optimization for the ridge parameter of Kernel Ridge Regression is done for $\lambda \in [10^{-5}, 1]$ with 20 log-spaced points.

  \item[\Cref{fig:gd-lofi-comparison}.]
    Binary CIFAR-10 animal vehicle, $n\in[10^2,5\times 10^4]$. Matched architectures with
    $p=4096$ , The LoFi ranks are kept fixed between architectures $32,64,128,256$. SGD: batch $B=512$,
    $\eta\propto\sqrt{B/n}$ (peak 0.01), cosine + warmup, total steps fixed
    across $n$. 10 seeds. 

  \item[\Cref{fig:filtering-one-hidden-layer}.]
    %
    Binary CIFAR-10, single LoFi hidden layer (ReLU), only the hidden representation reduced to rank $k$; input and output layers unreduced. Specifically the network is
    \begin{align}
    \phi(x)   &= \sigma\!\Big(\tfrac{1}{\sqrt{d}}\,W^{(1)} x\Big) \in \mathbb{R}^{p},
                 && W^{(1)}\in\mathbb{R}^{p\times d}, \label{eq:lofi-h1}\\
    g(x)      &= \sigma\!\Big(\tfrac{1}{\sqrt{k}}\,W^{(2)} V\,\phi(x)\Big) \in \mathbb{R}^{p},
                 && W^{(2)}\in\mathbb{R}^{p\times k},\ \ V\in\mathbb{R}^{k\times p}, \label{eq:lofi-h2}\\
    f(x)      &= w^{\top} g(x) + b,
                 && w\in\mathbb{R}^{p},\ b\in\mathbb{R}, \label{eq:lofi-readout}
\end{align}
where $\sigma$ is the ReLU. The two projection matrices $W^{(1)},W^{(2)}$ are
drawn i.i.d.\ Gaussian and \emph{frozen}; they are never trained. 
    Panels are a grid: rows fix the ridge penalty $\lambda\in\{10^{-2},10^{-1},10^{0},10^{1}\}$ (bottom row: $\lambda$ tuned per point with \texttt{RidgeCV} from \cite{scikit-learn} in $[10^{-6},10^{6}]$ over $100$ points), columns fix the training-set size $n\in\{2000,5000,10000,25000\}$. Each panel overlays $p\in\{100,500,1000,5000,10000\}$ and the infinite-width limit $p\to\infty$ (two-layer ReLU NNGP kernel, dashed; same per-row $\lambda$, double-centered readout). Retained rank $k$ spans $2,\dots,p{-}1$ (finite, $\sim 50$ points) and $2,\dots,n{-}1$ (kernel, $30$ points). Shading: $\pm 1$ SEM over seeds ($10$ finite, $5$ kernel).

  \item[\Cref{fig:nlofi-convergence}.]
    Binary CIFAR-10 (animals vs.\ vehicles), 3-layer ReLU LoFi with last layer before Ridge not reduced. Test MSE vs.\ width $p$ for $n\in\{200,500,1000,2000,5000\}$. All shaded areas are average over 10 seeds.

  \item[\Cref{fig:sample-complexity-pred}]
    Binary CIFAR-10 (animals vs.\ vehicles); a 2-layer ReLU LoFi with a
    population proxy obtained from the full dataset ($N=60{,}000$).
    For each sample size $n$, the LoFi pipeline is fit on a fresh
    permutation of $n$ training points (with $\bm W_1,\bm W_2$ held fixed
    across permutations); we plot the index-aligned overlap
    $|\langle\hat{\bm v}_i^{(n)},\hat{\bm v}_i^{(N)}\rangle|^2$ for the top
    $i=1,\dots,6$ eigenvectors of the layer-$\ell$ signed covariance. Mean
    $\pm$ SEM over $100$ dataset permutations. Dashed verticals: predicted
    thresholds $n_\ell^k$ from~\cref{eq:real-data-emergence-estimate}.

  \item[{\Cref{fig:lofi-layer-higher-order-decay}.}]
   {Binary CIFAR-10 (animals vs.\ vehicles); a 3-layer fully connected ReLU
    LoFi model.  For the input representation and the three successive LoFi
    layer representations, we use the layer features directly, sample random
    directions, build an empirical orthonormal polynomial basis in each
    one-dimensional projection, and plot the median squared label projection
    of the degree-$k$ component after subtracting the fitted $1/n$ finite-sample
    noise floor.}

  \item[\Cref{fig:sample-complexity-pred-conv}]
      Binary CIFAR-10 (animals vs.\ vehicles); a 2-layer convolutional ReLU LoFi (random convolutions with $P=5{,}000$ channels, $5{\times}5$ kernels, padding $2$, stride $1$; layer-1 eigenreduction $K_1=100$) with a population proxy obtained from $N=30{,}000$ samples. For each sample size $n$, the LoFi pipeline is fit on a fresh permutation of $n$ training points (with $\bm W_1,\bm W_2$ held fixed across permutations); we plot the per-layer diagonal overlaps $|\langle\hat{\bm v}_i^{(n)},\hat{\bm v}_i^{(N)}\rangle|^2$ for $i=1,\dots,8$, sorted descending within each draw so that $i=1$ is best-aligned axis at that $n$. Mean $\pm$ SEM over $100$ dataset permutations. Dashed verticals: predicted thresholds $n_\ell^k$ from~\cref{eq:real-data-emergence-estimate}.

  \item[\Cref{fig:eigenvalues-histogram}.]
    Binary CIFAR-10. Two conv layers (kernel $3\times3$, $p=512$, $2{\times}2$
    pool, $L_2$ norm) + FC layer ($p=512$); $n\in\{200,500,2000,50000\}$,
    ranks $k_\ell=32,64,128$. Symlog $x$-axis; top-5 eigenvalues marked.

  \item[\Cref{fig:celeba-feature-importance}, top.]
    CelebA at $64{\times}64$, attribute \emph{High Cheekbones} from CelebA.
    $L=3$ FC LoFi layers (ReLU), $P=50000$, $(k_1, k_2, k_3) = (100,50,20)$,
    $n=200{,}000$, seed $0$. Importance via Eq.~\eqref{eq:importance-layer}; Fourier low-pass parameters $f_0=0.15$, $\alpha=3.0$.

  \item[\Cref{fig:celeba-feature-importance}, bottom.]
    Same setup, attribute \emph{Smiling} from CelebA.

  \item[\Cref{fig:additional-cnn-feature-visualizations}, left.]
    CelebA \emph{Gender}. Six conv LoFi layers
    ($5{\times}5$, $2{\times}2$ pool, $L_2$ norm) + FC layer; channels
    $p=512$, ranks 16, 32,64,128,256,512 , ReLU; $n=500000$. First-layer
    $5{\times}5$ filters as RGB images.

  \item[\Cref{fig:additional-cnn-feature-visualizations}, right.]
    Same architecture. Activations of the 1st-ranked (top) and 7th-ranked
    (bottom) eigenvectors at each conv layer (shallow$\to$deep), on six test
    images.

\end{description}


\end{document}